\documentclass[accepted]{uai2023} % after acceptance, for a revised
\usepackage[american]{babel}
\usepackage{dsfont, amsfonts}

\usepackage{natbib} % has a nice set of citation styles and commands
\usepackage{mathtools} % amsmath with fixes and additions
\usepackage{booktabs} % commands to create good-looking tables
\usepackage{tikz} % nice language for creating drawings and diagrams
\usepackage{enumitem}
\usepackage{mathabx}

\usepackage{amssymb}
\usepackage{amsthm}
\usepackage[capitalize,noabbrev]{cleveref}

\usepackage{subfiles}
\usepackage{tikz, subcaption}
\usetikzlibrary{positioning}
\usepackage{algorithm}
\usepackage{algpseudocode}
\usepackage{todonotes}
\usepackage{bbm}
\usepackage{centernot}

\newtheorem{theorem}{Theorem}%[section]
    
\newtheorem{corollary}{Corollary}%[theorem]
\newtheorem{lemma}{Lemma}

\newenvironment{customlem}[1]{\renewcommand\thelemma{#1}\lemma}{\endlemma}

\newtheorem{definition}{Definition}
\newtheorem{remark}{Remark}

\newcommand{\V}[0]{\mathbf{V}}

\newcommand{\W}[0]{\mathbf{W}}
\newcommand{\U}[0]{\mathbf{U}}
\newcommand{\E}[0]{\mathbf{E}}
\newcommand{\X}[0]{\mathbf{X}}
\newcommand{\R}[0]{\mathbf{R}}
\newcommand{\Y}[0]{\mathbf{Y}}
\newcommand{\Z}[0]{\mathbf{Z}}

\newcommand{\A}[0]{\mathbf{A}}
\newcommand{\B}[0]{\mathbf{B}}
\newcommand{\T}[0]{\mathbf{T}}
\newcommand{\D}[0]{\mathbf{D}}
\newcommand{\Pa}[2]{\textit{Pa}_{#2}(#1)}
\newcommand{\Ch}[2]{\textit{Ch}_{#2}(#1)}
\newcommand{\Anc}[2]{\textit{Anc}_{#2}(#1)}
\newcommand{\De}[2]{\textit{De}_{#2}(#1)}

\newcommand{\x}[0]{\mathbf{x}}
\newcommand{\y}[0]{\mathbf{y}}
\newcommand{\z}[0]{\mathbf{z}}
\newcommand{\w}[0]{\mathbf{w}}
\newcommand{\dd}[0]{\mathbf{d}}

\newcommand{\G}[0]{\mathcal{G}}
\newcommand{\M}[0]{\mathcal{M}}
\newcommand{\F}[0]{\mathcal{F}}
\newcommand{\dom}[2]{\mathfrak{X}_{#2}(#1)}

\newcommand{\independent}{\perp\!\!\!\perp}
\newcommand{\notindependent}{\centernot{\independent}}

\title{On Identifiability of Conditional Causal Effects}

\author[1]{\href{mailto:<yaroslav.kivva@epfl.ch>?Subject=Your UAI 2023 paper}{Yaroslav Kivva}{}}
\author[1, 3]{Jalal Etesami}
\author[1,2]{Negar Kiyavash}
\affil[1]{%
    School of Computer and Communication Sciences\\
    EPFL\\
    Lausanne, Switzerland
}
\affil[2]{%
     College of Management of Technology\\
    EPFL\\
    Lausanne, Switzerland
}
\affil[3]{%
   TUM School of Computation, Information and Technology\\
   Technical University of Munich
}
\begin{document}
\maketitle
\begin{abstract}
We address the problem of identifiability of an arbitrary \textit{conditional causal effect} given both the causal graph and a set of any observational and/or interventional distributions of the form $Q[S]:=P(S|do(V\setminus S))$, where $V$ denotes the set of all observed variables and $S\subseteq V$. 
    We call this problem conditional generalized identifiability (\mbox{\textbf{c-gID}} in short) and prove the completeness of Pearl's $do$-calculus for the \mbox{c-gID} problem by providing sound and complete algorithm for the c-gID problem.  
    This work revisited the c-gID problem in \cite{lee2020general, correa2021nested} by adding explicitly the positivity assumption which is crucial for identifiability. It extends the results of \citep{lee2019general, kivva2022revisiting} on general identifiability (gID) which studied the problem for  \mbox{\textit{unconditional}} causal effects and \cite{shpitser2012identification} on identifiability of conditional causal effects given \textit{merely} the observational distribution $P(\mathbf{V})$ as our algorithm generalizes the algorithms proposed in \citep{kivva2022revisiting} and \citep{shpitser2012identification}.
\end{abstract}

\section{Introduction}

This paper addresses the problem of identification of a conditional post-interventional distribution from the combination of observational and/or interventional distributions. Formally, the relationships between the variables of interest are established by a directed acyclic graph (DAG) \cite{pearl1995causal}. Each node in the causal graph represents some random variable that may simulate real-life measurements, and each directed edge encodes a possible causal relationship between the variables. 
In general, a subset of the nodes in DAG are observed and others may be hidden. 
The hidden nodes could result in spurious correlations between observed variables and complicate the question of identifiability. On the other hand, when all the variables in the system are observable and the distribution over them is known then any conditional causal effect is identifiable.

%Studying these causal relationships is very important for applied science such as health and the social sciences, artificial intelligence, machine learning, etc.  When someone performs an experiment or forces somebody to follow some treatment then this action is represented by a $do(\cdot)$ operation in the causal inference literature. Practically an action $do(\x)$ changes the distribution over the observed variables $P(\cdot)$ into a new one $P_{\x}(\cdot)$ which is called post-interventional distribution. The ability to compute a causal effect that was not observed before may help in constructing explanations and making conclusions about the phenomenon under investigation. \cite{shpitser2012identification} provides an example of the important relationship between the identifiability of a conditional causal distribution and sequential decision problem that arise in many domains \citep{pearl1995probabilistic}.

The question of identification of the causal effect has been one of the central focus of research in causal inference literature. The classical setting of the problem asks whether the causal effect $P_{\x}(\y)$\footnote{This notation indicates causal effect on $\textbf{y}$ after intervention $do(\X=\x)$, That is, $P(\textbf{y}| do(\X=\x))$ shortened to $P_{\x}(\textbf{y})$.} is identifiable in a given graph $\G$ from observational distribution $P(\V)$ ($\V$ is a set of all observed nodes in the graph $\G$). The problem was solved in \cite{shpitser2006identification, huang2006identifiability} and later \cite{shpitser2012identification} extended the result by answering the question when a conditional causal effect $P_{\x}(\y|\z)$ is identifiable in a given graph $\G$.
The work of \cite{bareinboim2012causal, lee2019general, kivva2022revisiting} solved a generalization of the classical identifiability problem, namely identifiability of unconditional causal effect $P_\x(\y)$ from a specific mix of observational and interventional distributions.  It is noteworthy that all aforementioned works proved that the rules of do-calculus are sound and complete for the identification of the causal effect in their settings. 
Furthermore, the work of \cite{tikka2019causal, mokhtarian2022causal, bareinboim2014transportability, bareinboim2015recovering} considers the problem of identifiability in a presence of additional information to observational/interventional distributions and the causal graph $\G$. More specifically, \cite{mokhtarian2022causal} considers the identifiability problem in the presence of additional knowledge in the form of context-specific independence for some variables.
\cite{tikka2019causal} assumes that they have access to multiple incomplete data sources and \cite{bareinboim2015recovering} studies the identifiability problem under a selection bias.
%and \cite{bareinboim2014transportability} solves the problem of identification of the query in a target domain using the observations from other source domains

%{\color{red}
%Briefly discuss \cite{lee2020general,correa2021nested} that consider similar problem but their completeness proof is not valid because of the positivity assumption.}

{

In this paper, we extend both the general identifiability (gID) result of \cite{kivva2022revisiting} and the conditional identifiability result of \cite{shpitser2012identification}. 
%the ideas of \cite{kivva2022revisiting} on general identifiability problem (gID) and \cite{shpitser2012identification} on identifiability of conditional causal effect from merely observational distribution $P(\V)$.  
More specifically, our work answers the question of identifiability of an \textit{arbitrary} conditional causal effect $P_{\x}(\y|\z)$ under the same set of assumptions as in gID problem. 
We call this problem \textit{conditional general identifiability}, for short \textbf{c-gID}.
This problem has been studied in \cite{lee2020general, correa2021nested}.
 %The c-gID problem can be considered as a corollary of the works \cite{lee2020general, correa2021nested}. 
 The authors of \cite{lee2020general} generalizes the problem of c-gID by assuming that observable data is available from multiple domains and \cite{correa2021nested} considers the c-gID problem as an identifiability problem of counterfactual quantities. 
 However, both of the aforementioned works are based on causal models that violate the positivity assumption (See Appendix B) which is crucial for identification as it is discussed in \cite{kivva2022revisiting}. 
 %More precisely, their proofs of completeness rely on causal models that ignores a positivity assumption. 
Since they did not discuss whether their proposed models can be modified such that the positivity assumption holds and it is not straightforward whether such modifications exist, herein we present an alternative proof for the c-gID problem including its soundness and completeness.
The causal models developed here for proving the completeness of our algorithm are novel and satisfy the positivity assumption. 
%Furthermore our proof of the completeness part of the c-gID problem is novel and we believe that they  important itself for the further research in generalization of identifiability results.
}
%\color{red}In this paper, we extend the ideas of \cite{kivva2022revisiting} on general identifiability problem (gID) and \cite{shpitser2012identification} on identifiability of conditional causal effect from merely observational distribution $P(\V)$.  More specifically, our work answers the question of the identification of an \textit{arbitrary} causal effect $P_{\x}(\y|\z)$ under assumptions similar to the ones in gID problem. We call this problem \textit{conditional general identifiability}, for short \textbf{c-gID}.}
%Additionally, we provide the sound and complete algorithm that generalizes both algorithms of \cite{kivva2022revisiting} and \cite{shpitser2012identification}. 
% A nice corollary of our results is that the rules of do-calculus remain sound and complete for the c-gID problem.

\section{Preliminaries}
\subsection{Notation and definitions}
 We denote random variables by capital letters and their realization by their lower-case version. 
 Similarly, a set of random variables and their realizations are denoted by bold capital and bold lower-case letters, respectively. 
 For two integers $a\leq b$, we define $[a:b]:=\{a, a+1, \dots, b\}$.
 For any random variable $X$, we denote its domain set by $\dom{X}{}$ and for any set of random variables $\X$, we denote by $\dom{\X}{}$, the Cartesian product of the domains of the variables in $\X$. 
 Suppose that $\X$ and $\Y$ are arbitrary sets of random variables, then we say that realizations $\x$ and $\y$ are \textit{consistent}, if the values of $\X\cap\Y$ in $\x$ and $\y$ are the same. 
 Also, we use $\dom{\X}{\y}$ to denote a set of realizations of $\X$ that are consistent with $\y$.
  %the values of variables $\X \cap \Y$ in the realization of $\y$. 
 Suppose that $\X'\subseteq \X$ and $\x$ to be a realization of $\X$. Then, we use $\x[\X']$ to denote a realization of $\X'$ that is consistent with $\x$.
 %with the same values of $\X'$ as in the realization $\x$. 
 When it is clear from the context, we write $\x'$ instead of $\x[\X']$. 

\textbf{Causal Graph:}
Consider a directed graph $\G:=(\V\cup \U, \E)$ over node $\V\cup\U$ in which $\V$ and $\U$ denote the set of observed and hidden variables, respectively and $\E\subseteq (\V\cup \U)\times(\V\cup \U)$ denotes the set of directed edges. 
A causal graph $\G$ is a directed acyclic\footnote{It contains no directed cycle. } graph (DAG).
%over a set of vertices $\V\cup\U$ with directed edges from $\E$, where $\V$ is responsible for a set of observed random variables and $\U$ is responsible for a set of latent random variables. 
We say that node $X$ is a parent of another node $Y$ (subsequently, $Y$ is a child of $X$) if and only if there exists a direct edge from $X$  to $Y$ in $\G$, e.g. $(X, Y) \in \E$. 
Similarly, $X$ is said to be an ancestor of $Y$ (subsequently, $Y$ is a descendant of $X$) if and only if there is a directed path from $X$ to $Y$ in $\G$. 
We denote the set of parents, children, ancestors, and descendants of $X$ by $\Pa{X}{\G}$, $\Ch{X}{\G}$, $\Anc{X}{\G}, \De{X}{\G}$ respectively. We assume that $X$ belongs to all the aforementioned sets. 
%Further, we suppose that any node $X \in \V\cup \U$ is a child\textbackslash parent\textbackslash ancestor\textbackslash descendant of itself and we denote a set of parents, children, ancestors, and descendants of $X$ by $\Pa{X}{\G}$, $\Ch{X}{\G}$, $\Anc{X}{\G}, \De{X}{\G}$ respectively. 
Additionally, for a subset of nodes $\X$, we define $\Pa{\X}{\G}:=\bigcup_{X\in \X}\Pa{X}{\G}$ and analogously, define $\Ch{\X}{\G}$, $\Anc{\X}{\G}$ and $\De{\X}{\G}$. 

A causal graph $\G$ is called a semi-Markovian, if any node from $\U$ has exactly two children without any parents. 
Suppose that $\G$ is a semi-Morkovain graph and $\X\subseteq\V$. 
In this case, we use $\G[\X]$ to denote the induced subgraph of $\G$ over variables in $\X$ including all unobserved variables that have both children in $\X$. 
We also use $\widehat{\G}[\X]$ to denote the dual graph of $\G[\X]$ that is a mixed\footnote{It contains both directed and bidirected edges.} graph and it is constructed from $\G[\X]$ by replacing unobserved variables and their outgoing arrows with bidirected edges.
By the abuse of notation, we use $\G[\X]$ and $\widehat{\G}[\X]$ interchangeably.% to talk about the same subgraph but refer to different properties. 

%Other important structures that we need to define are c-component and c-forest. These structures play a key role in the identifiability problems.
\begin{definition}[c-component]
    C-components of a subset $\X$ in $\G$ are the connected components in $\widehat{\G}[\X]$ after removing all directed edges, i.e., nodes in each c-component are connected via bidirected edges. 
    $\X$ is called a single c-component if $\X$ has only one c-component, i.e., $\widehat{\G}[\X]$ is a connected graph after removing all directed edges.
\end{definition}

\begin{figure}
\hspace{1cm}
        \begin{tikzpicture}[
                roundnode/.style={circle, draw=black!60,, fill=white, thick, inner sep=1pt},
                dashednode/.style = {circle, draw=black!60, dashed, fill=white, thick, inner sep=1pt},
                ]
                    % Nodes
                    \node[roundnode]        (X1)        at (-1.5, 0)                   {$X_1$};
                    \node[roundnode]        (X2)        at (1.5, 0)                   {$X_2$};
                    \node[roundnode]        (Y1)        at (-1.5, -1.6)                  {$Y_1$};
                    \node[roundnode]        (Y2)        at (1.5, -1.6)                  {$Y_2$};
                    \node[dashednode]       (U1)        at (0, 0)                   {$U_1$};
                    \node[dashednode]       (U2)        at (-2.75, -.8)                   {$U_2$};
                    
                    %Edges
                    \draw[-latex] (X1.south) -- (Y1.north) ;
                    \draw[-latex] (X2) -- (Y2);
                    \draw[latex-, dashed] (X1) -- (U1);
                    \draw[-latex, dashed] (U1) -- (X2);
                    \draw[latex-, dashed] (Y1) -- (U2);
                    \draw[-latex, dashed] (U2) -- (X1);
                \end{tikzpicture}
\caption{A semi-Markovian DAG  over the set of observed variable $\V=\{X_1,X_2,Y_1,Y_2\}$ and the set of hidden variables $\U=\{U_1,U_2\}$. } 
    \label{fig: example}
\end{figure}
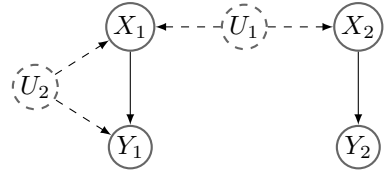

For instance in Figure \ref{fig: example}, c-components of $\{X_1,X_2,Y_2\}$ are $\{X_1,X_2\}$ and $\{Y_2\}$. In this DAG, $\{X_1,X_2,Y_1\}$ and $\{Y_2\}$ are each single c-components.

\begin{definition}[c-forest]
    Let $\F$ be a subgraph of $\G$ over a set of nodes $\X$. The maximal subset of $\X$ with no children in $\F$ is called the root set and denoted by $\R$. 
    $\mathcal{F}$ is a $\R$-rooted c-forest if $\X$ is a single c-component with root set $\R$, and all observable nodes in $\X$ have at most one child in $\F$.
\end{definition}

In Figure \ref{fig: example}, $\G[\{X_1,X_2,Y_1\}]$ is a $\{Y_1,X_2\}$-rooted c-forest.

\textbf{Causal Model:}
A causal model $\M$ is defined over a set of random variables $\V\cup \U$ via Structural Equation Model (SEM) \citep{pearl2009causality} with a causal graph $\G$. 
In a SEM with a causal graph $\G$, each variable $X \in \V\cup \U$ is determined by its parents and an exogenous variable $\epsilon_X$, i.e. $X=f_{X}(\Pa{X}{\G}, \epsilon_X)$. 
It is assumed that the set of exogenous variables, $\{ \epsilon_X | X \in \V \cup \U\}$, are mutually independent.
If graph $\G$ is a semi-Markovian, then $\M$ is said to be a semi-Markovian causal model. 
Because, the problem of identifiability in a DAG is equivalent to a relative identifiability problem in a semi-Markovian DAG \citep{huang2006identifiability}, in this work, we only consider the problem of identifiability in semi-Markovian models.

In a semi-Markovian causal model, by Markov property \citep{pearl2009causality}, the induced joint distribution can be factorized as follows
\begin{equation*}
    P^{\M}(\mathbf{v}) = \sum_{\textbf{u}}\prod_{X\in \V}P^{\M}(x|\Pa{X}{\G})\prod_{U \in \U}P^{\M}(u),
\end{equation*}
where the summation is over latent variables in $\U$. 
We use $\mathbb{M}(\G)$ to denote the set of causal models with graph $\G$ such that for any $\M \in \mathbb{M}(\G)$ and any realization $\mathbf{v} \in \dom{\mathbf{V}}{}$,  $P^{\M}(\mathbf{v})>0$.
In the remainder of this work, we assume that all causal models belong to $\mathbb{M}$.
This is known as the positivity assumption in the causality literature and as it is discussed in \cite{kivva2022revisiting}, it is crucial for developing sound identification algorithm. 

 In a causal model $\M$, post-interventional distribution is defined using $do$-operation. 
 An intervention $do(X=x)$ modifies the corresponding SEM by replacing the equation of $X=f_{X}(\Pa{X}{\G}, \epsilon_X)$ by $X=x$.
The conditional post-interventional distribution of $\textbf{y}$ given $\textbf{s}$ after intervening on $do(X=x)$ is denoted by $P(\textbf{y}| do(\X=\x), \textbf{s}):= P_{\x}(\textbf{y}|\textbf{s})$.

 Suppose that $\mathbf{S} = \mathbf{S}' \cup \mathbf{S}''$, where $\mathbf{S}'$ and $\mathbf{S}''$ are two disjoint subsets of observed variables $\V$. 
 We define $Q$-notations, $Q[\mathbf{S}](\cdot)$ and $Q[\mathbf{S}'|\mathbf{S}''](\cdot)$ as follows:
 \begin{align}
     & Q[\mathbf{S}](\mathbf{v}) := P_{\mathbf{v}\setminus \mathbf{s}}(\mathbf{s}), \\
     & Q[\mathbf{S}'|\mathbf{S}''](\mathbf{v}) := P_{\mathbf{v}\setminus \mathbf{s}}(\mathbf{s}'|\mathbf{s}''),
 \end{align}
 where $\mathbf{s} = \mathbf{v}[\mathbf{S}], \mathbf{v}\setminus \mathbf{s} = \mathbf{v}[\V \setminus \mathbf{S}], \mathbf{s}' = \mathbf{v}[\mathbf{S}']$, and $\quad \mathbf{s}'' = \mathbf{v}[\mathbf{S}'']$.
% \begin{equation*}
 %    \begin{split}
  %       & \mathbf{s} = \mathbf{v}[\mathbf{S}], \quad \mathbf{v}\setminus \mathbf{s} = \mathbf{v}[\V \setminus \mathbf{S}], \\
   %      & \mathbf{s}' = \mathbf{v}[\mathbf{S}'], \quad \mathbf{s}'' = \mathbf{v}[\mathbf{S}''].
    % \end{split}
% \end{equation*}
 Note that $Q[\V](\textbf{v})=P(\V=\textbf{v})$.
 By Markov property and basic probabilistic manipulation, we have
 \begin{align}
    \label{eq: Q}
     & Q[\mathbf{S}](\mathbf{v}) = \sum_{\U}\prod_{S \in \mathbf{S}}P(s|\Pa{S}{\G})\prod_{U \in \U}P(u)\\
     \label{eq: conditional Q}
     & Q[\mathbf{S}'|\mathbf{S}''](\mathbf{v}) = \frac{Q[\mathbf{S}](\mathbf{v})}{\sum_{\mathbf{v}' \in \dom{\V}{\mathbf{v}\setminus\mathbf{s}'}} Q[\mathbf{S}](\mathbf{v}')}
 \end{align}

%To link the independence properties of the random variables in graph $\G$ we use the well-known notion of d-separation \citep{pearl2009causality}.
\begin{definition}[Blocked path]
A path in $\G$ is a non-repeated sequence of connected nodes. 
A path $p$ in $\G$  is said to be blocked by a set of nodes in $\Z$ if and only if
    \begin{itemize}[leftmargin=*]
        \item $p$ contains a chain $X\rightarrow W \rightarrow Y$ or fork $X\leftarrow W \rightarrow Y$, such that $W \in \Z$, or
        \item $p$ contains a collider $X \rightarrow W \leftarrow Y$ (node $W$ is called a collider), such that $\Z \cap \De{W}{\G}=\emptyset$.
    \end{itemize}
\end{definition}
Two disjoint sets of nodes $\X$ and $\Y$ are d-separated by $\Z$ in $\G$ if any path between $\X$ and $\Y$ are blocked by $\Z$ and denote it by $(\X \independent \Y|\Z)_{\G}$. 
Using d-separation, we introduce rules of do-calculus \citep{pearl2009causality} as the main tools for causal effect identification.
\begin{itemize}[leftmargin=*]
    \item Rule 1: $P_{\x}(\y|\z,\w)=P_{\x}(\y|\w)$ if $(\Z \independent \Y|\X,\W)_{\G_{\overline{\X}}}$.
    
    \item Rule 2: $P_{\x,\z}(\y|\w)=P_{\x}(\y|\z,\w)$ if $(\Z \independent \Y|\X,\W)_{\G_{\overline{\X},\underline{\Z}}}$.
    
    \item Rule 3: $P_{\x,\z}(\y|\w)=P_{\x}(\y|\w)$ if $(\Z \independent \Y|\X,\W)_{\G_{\overline{\X},\overline{\Z_W}}}$,
\end{itemize}
where $\G_{\overline{\X},\underline{\Y}}$ denotes an edge subgraph of $\G$ where all incoming arrows into $\X$ and all outgoing arrows from $\Y$ are deleted and $\Z_W:=\Z\setminus Anc_{\G_{\overline{\X}}}(\W)$.

\subsection{Classical identifiability (ID) }

\begin{table*}[th]
\caption{Different types of identifiability problems.} \label{id problems}
\begin{center}
\begin{tabular}{|c|c|c|c|}
\hline
\textbf{Problem}  &\textbf{Target}  & \textbf{Input} & \textbf{Solved}\\
\hline
\hline
Causal effect identifiability (ID)   &   &   &\\
\color{blue}\cite{shpitser2006identification}   & $P_{\x}(\y)$   & $\G, P(\V)$   &\checkmark\\
\color{blue}\cite{huang2006identifiability}   &   &   &\\
\hline
Conditional causal effect identifiability (c-ID) & $P_{\x}(\y|\z)$ & $\G, P(\mathbf{V})$ & \checkmark\\
\color{blue}\cite{shpitser2012identification}   &   &   &\\
\hline
z-identifiability (zID) & $P_{\x}(\y)$ & $\G, P(\V), \{P_{\V \setminus \A'}(\A')| \forall \A' \subset \A\}$ & \checkmark\\
\color{blue}\cite{bareinboim2012causal}   &   &   &\\
\hline
g-identifiability (gID) &    &   &\\
\color{blue}\cite{lee2019general,kivva2022revisiting}   & $P_{\x}(\y)$   & $\G, \{P(\mathbf{A}_i|do(\mathbf{V}\setminus\mathbf{A}))\}_{i=0}^m$   & \checkmark\\
\hline
\textit{Conditional general identifiability (c-gID)}  & $P_{\x}(\y|\z)$ & $\G, \{P(\mathbf{A}_i|do(\mathbf{V}\setminus\mathbf{A}))\}_{i=0}^m$ & \checkmark our work \\
 \color{blue}\cite{lee2020general,correa2021nested} &   &   &\\
\hline
Generalized identifiability  & $P_{\x}(\y|\z)$ & $\G, \{P(\mathbf{A}_i|do(\mathbf{B}_i), \mathbf{C}_i)\}_{i=0}^m$ & ?\\
\hline
\end{tabular}
\end{center}
\end{table*}
Classical identifiability problem refers to computing a causal effect $P_{\x}(\y)$ from a given joint distribution $P(\V)$ in a causal graph $\G$. 
This problem was solved independently by \cite{shpitser2006identification} and \cite{huang2006identifiability}.  
\cite{shpitser2012identification} extended this result to identifiability of a conditional causal effect, i.e., $P_{\x}(\y|\z)$.

%{\color{red} Here, you can add a remark and say about the positivity violations of \cite{lee2020general,correa2021nested} and refer to the appendix for further details on how the positivity assumptions are ignored there and say that correcting their proofs are not straightforward and hence we provide an alternative proof. }
\begin{definition}[conditional ID]
    Suppose $\X$, $\Y$, and $\Z$ are three disjoint subsets of $\V$.
    The causal effect $P_{\x}(\y|\z)$ is said to be conditional ID in $\G$ if for any $\x \in \dom{\X}{}$, $\y \in \dom{\Y}{}$, and $\z \in \dom{\Z}{}$, $P^{\M}_{\x}(\y|\z)$ is uniquely computable from $P^{\M}(\mathbf{V})$ in any causal model $\mathcal{M} \in \mathbb{M}(\G)$.
    %consistent with the given observational distribution $P(\V)$.
\end{definition}
Knowing $P_{\x}(\y, \z)$, it is straightforward to uniquely compute $P_{\x}(\y|\z)$ from
$
    P_{\x}(\y|\z) = {P_{\x}(\y, \z)}/{\sum_{\Y'}P_{\x}(\y', \z)}.
$
On the other hand, \cite{tian2004identifying} showed that $P_{\x}(\y|\z)$ might be identifiable in $\G$ even if $P_{\x}(\y, \z)$ is not identifiable. 
This happens when the ``non-identifiable parts'' of $P_{\x}(\y, \z)$ in the nominator cancel out  with the non-identifiable parts of $P_{\x}(\z)$ in the denominator. Next example demonstrates such a scenario.

\begin{figure}
\centering
\begin{tikzpicture}[
        roundnode/.style={circle, draw=black!60,, fill=white, thick, inner sep=1pt, minimum size=0.65cm},
        dashednode/.style = {circle, draw=black!60, dashed, fill=white, thick, inner sep=1pt, minimum size=0.65cm},
        ]
            % Nodes
            \node[roundnode]        (Z1)        at (-1.5, 0)                 {$Z_1$};
            \node[roundnode]        (W1)        at (1.5, 0)                  {$W_1$};
            \node[roundnode]        (X1)        at (-1.5, 1.3)               {$X_1$};
            \node[roundnode]        (Y1)        at (1.5, 1.3)                {$Y_1$};
            \node[roundnode]        (Z2)        at (0, 0)                    {$Z_2$};
            
            %Edges
            \draw[-latex] (X1.south) -- (Z1.north);
            \draw[latex-] (Z1.east) -- (Z2.west);
            \draw[latex-] (Z2.east) -- (W1.west);
            \draw[latex-] (W1.north) -- (Y1.south);
            \draw[latex-latex, dashed] (X1) .. controls +(left:10mm ) and +(left:10mm) .. (Z1);
            % \draw[latex-latex, dashed] (Z1.north east) .. controls +(up:5mm) and +(up:5mm) .. (W1.north west);
            \draw[latex-latex, dashed] (W1) .. controls +(right:10mm ) and +(right:10mm) .. (Y1);
        \end{tikzpicture}
\caption{A semi-Markovian DAG over \mbox{$\V=\{X_1,Y_1,Z_1,Z_2, W_1\}$}} 
    \label{fig: example c-ID}
\end{figure}
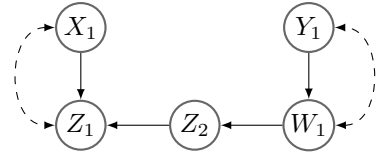
\textbf{Example:}
Consider the causal graph $\G$ as on Figure \ref{fig: example c-ID}. Assume that one wants to compute the causal effect $P_{\x}(\y|\z)$, where $\X = \{X_1\}$, $\Y = \{Y_1\}$ and $\Z = \{Z_1, Z_2\}$. Then,
\begin{equation*}
     P_{x_1}(y_1| z_1, z_2) = \frac{P_{x_1}(y_1, z_1, z_2)}{P_{x_1}(z_1, z_2)}.
\end{equation*}
where 
$$P_{x_1}(y_1, z_1, z_2) = \sum_{w_1 \in \dom{W_1}{}}P_{x_1}(y_1, w_1, z_1, z_2),$$
 and 
$$P_{x_1}(z_1, z_2) = \sum_{\substack{w_1 \in \dom{W_1}{}\\ y_1 \in \dom{Y_1}{}}}P_{x_1}(y_1, w_1, z_1, z_2).$$

%\begin{align*}
 %    P_{x_1}(y_1, z_1, z_2) =& \sum_{w_1 \in \dom{W_1}{}}P_{x_1}(y_1, w_1, z_1, z_2),\\
 %    P_{x_1}(z_1, z_2) =& \sum_{\substack{w_1 \in \dom{W_1}{}\\ y_1 \in \dom{Y_1}{}}}P_{x_1}(y_1, w_1, z_1, z_2).
%\end{align*}
In terms of $Q$-notation, we have
\begin{align*}
     P_{x_1}(y_1, z_1, z_2) =& \sum_{W_1}Q[Y_1, W_1, Z_1, Z_2],\\
     P_{x_1}(z_1, z_2) =& \sum_{W_1, Y_1}Q[Y_1, W_1, Z_1, Z_2].
\end{align*}
Using results of \cite{huang2006identifiability}, the above equations can be simplified as follows,
\begin{align*}
     P_{x_1}(y_1, z_1, z_2) =& Q[Z_1] \sum_{W_1} Q[Y_1, W_1, Z_2],\\
     P_{x_1}(z_1, z_2) = & Q[Z_1] \sum_{W_1, Y_1}Q[Y_1, W_1, Z_2].
\end{align*}
Results of \cite{huang2006identifiability, shpitser2006identification} imply that $Q[Z_1]$ is not ID from $\G$, however $Q[Y_1, W_1, Z_2]$ is ID in $\G$. Therefore, both causal effects $P_{\x}(\y, \z)$ and $P_{\x}(\z)$ are not ID in $\G$, but clearly 
$$
P_{\x}(\y|\z) = \frac{\sum_{W_1} Q[Y_1, W_1, Z_2]}{\sum_{W_1, Y_1}Q[Y_1, W_1, Z_2]}
$$
is identifiable in $\G$.

%The last theorem links the problem of causal effect identifiability to the problem of conditional causal effect identifiability.

\subsection{Generalized identifiability (gID)}

In this problem, the goal is to identify a causal effect in a given graph $\G$ from a set of observational and/or interventional distributions instead of only observational distribution $P(\V)$. 
This problem, to the best of our knowledge, remains open when the set of given distributions are arbitrary. In the special case, when the set of given distributions are in the form of $Q$-notations, the problem is called generalized identifiability (gID) (See below for a formal definition) and was solved by \citep{lee2019general,kivva2022revisiting}.
See Table \ref{id problems} for a summary of solved and unsolved problems in the causal identifiability context.

%Assume that instead of observational distribution $P(\V)$ are given interventional distributions (?). Then one can verify that the causal effect from \textbf{Example} become identifiable, although $P_{\x}(\y, \z)$ still non-identifiable. In general, the question of which observational and interventional distributions are enough to know in order to identify a causal effect $P_{\x}(\y|\z)$ is very complicated and remains open, i.e. "Generalized identifiability" in Table \ref{id problems}.
 
%\cite{lee2019general} and \cite{kivva2022revisiting} considered a relaxation of "Generalized identifiability" problem which is called by gID problem.

\begin{definition}[gID]
Suppose $\X$ and $\Y$ are two disjoint subsets of $\V$ and $\mathbb{A} := \{\A_i\}_{i=0}^m$ is a collection of subsets of $\V$, i.e., $\A_i \subseteq \V$ for all \mbox{$i \in [0:m]$}. 
The causal effect $P_{\x}(\y)$ is said to be gID from $(\mathbb{A}, \G)$ if for any $\x\in \dom{\X}{}$ and $\y\in \dom{\Y}{}$ if $P_{\x}^{\M}(\y)$ is uniquely computable from $\{Q^{\M}[\A_i]\}_{i=0}^{m}$ in any causal model $\M \in \mathbb{M}(\G)$.
%consistent with the set of given distributions $\{Q[\A_i]\}_{i=0}^{m}$. 
%$Q[\Y]$ is said to be gID from $(\mathbb{A}, \G)$ if the the causal effect $P_{\mathbf{v}\setminus\y}(\y)$ is gID from $(\mathbb{A}, \G)$.
\end{definition}

Note that the classical ID problem is a special case of the gID problem when $\mathbb{A}=\{\V\}$. 
More than a decade after \cite{shpitser2012identification} proposed a sound and complete algorithm for ID, \cite{kivva2022revisiting} solved the gID problem by showing that gID problem can be reduced to a series of separated ID problems. Formally, they showed the following result. % that can be solved separately.
\begin{theorem}[\cite{kivva2022revisiting}] \label{th: gID main}
    Suppose that $\mathbf{S}\subseteq \V$ is a single c-component in $\G$. Then, $Q[\mathbf{S}]$ is gID from $(\mathbb{A}, \G)$ if and only if there exists $\A \in \mathbb{A}$, such that $\mathbf{S} \subseteq \A$ and $Q[\mathbf{S}]$ is identifiable from $\G[\A]$.
\end{theorem}
%Intuitively this theorem says that if we want to identify something in one c-component then information regarding distributions outside of this c-component may not help. 

\subsection{Conditional generalized identifiability (c-gID) }

In this work, we address an extension of both conditional ID and g-ID problem in which the goal is to identify a conditional causal effect from a set of observational and/or interventional distributions.
\begin{definition}(c-gID)
    Suppose $\X$, $\Y$ and $\Z$ are three disjoint subsets of $\V$ and $\mathbb{A} := \{\A_i\}_{i=0}^m$ is a collection of subsets of $\V$, i.e., $\A_i \subseteq \V$ for all \mbox{$i \in [0:m]$}. 
    The causal effect $P_{\x}(\y|\z)$ is said to be c-gID from $(\mathbb{A}, \G)$ if for any $\x\in \dom{\X}{}$, $\y\in \dom{\Y}{}$, and $\z\in \dom{\Z}{}$, $P_{\x}^{\M}(\y|\z)$ is uniquely computable from $\{Q^{\M}[\A_i]\}_{i=0}^{m}$ in any causal model $\M \in \mathbb{M}(\G)$.
    %consistent with the set of given distributions $\{Q[\A_i]\}_{i=0}^{m}$. 
    %$Q[\Y]$ is said to be c-gID from $(\mathbb{A}, \G)$ if the the causal effect $P_{\mathbf{v}\setminus\y}(\y)$ is c-gID from $(\mathbb{A}, \G)$.
\end{definition}

From this definition, it is clear that c-gID covers both conditional ID and gID. 
Namely, when $\Z=\emptyset$, then c-gID reduces to the gID problem, studied by \cite{lee2019general,kivva2022revisiting}. 
When $\mathbb{A}=\{\V\}$, c-gID becomes the conditional ID problem studied by \cite{shpitser2012identification}. 
{
Both \cite{lee2020general} and \cite{correa2021nested} proposed algorithms for identification problems that can also be used for solving c-gID problem. 
However, the completeness of their algorithms rely on causal models that violate the positivity assumption. For more details see Appendix B. 
Additionally, they miss discussions on whether this issue in their proofs can be resolved. 
%they do not specify the positivity assumption in their work explicitly and the proof of completeness misses an explanation whether the issues with a positivity assumption can be resolved. For details look into Appendix \ref{sec: app B}.

Next we propose an alternative solution for the c-gID problem under the positivity assumption. 
The soundness and completeness of our solution are based on novel techniques that we believe they are important for further generalizations of identifiability problems. 
%allows us to prove the completeness part of the general identifiability of the conditional causal queries. Note that our results is important by itself, since they contain a new ideas that also might be used for the further generalizations of identifiability problems.
}

\section{Main result}

%In the previous section, we recap the main results of the conditional identifiability and g-identifiability problems. Unfortunately, the conditional ID problem is limited to the assumption that only the observation distribution $P(\V)$ is given when a relaxation of this condition is of much interest and importance. As we discussed earlier, \cite{lee2019general} and \cite{kivva2022revisiting} relax this assumption by giving access to multiple distributions of a specific type, but solve the problem only for the unconditional causal effects $P_{\x}(\y)$. In this paper we extend results of \citep{lee2019general, kivva2022revisiting} by solving the problem for any causal effect $P_{\x}(\y|\z)$.
% \begin{definition}(c-gID)
%     Suppose $\X$, $\Y$ and $\Z$ are disjoint subsets of $\V$ and $\mathbb{A} := \{\A_i\}_{i=0}^m$ such that $\A_i \subseteq \V$ for all \mbox{$i \in [0:m]$}. The causal effect $P_{\x}(\y|\z)$ is said to be c-gID from $(\mathbb{A}, \G)$ if for any $\x\in \dom{\X}{}$ and any  $\y\in \dom{\Y}{}$ one can uniquely compute $P_{\x}^{\M}(\y|\z)$ from the set of known distributions $\{Q[\A_i]\}_{i=0}^{m}$ in any causal model $\M \in \mathbb{M}(\G)$ consistent with observations. $Q[\Y]$ is said to be c-gID from $(\mathbb{A}, \G)$ if the the causal effect $P_{\mathbf{v}\setminus\y}(\y)$ is c-gID from $(\mathbb{A}, \G)$.
% \end{definition}
%Note that in case $\Z=\emptyset$ the c-gID definition is exactly the same as the gID definition. Therefore we may use c-gID and gID notation interchangeably for this setting.

The main idea presented in this work for solving the c-gID problem is to construct an equivalent gID problem and then use the results of  \citep{lee2019general,kivva2022revisiting} to solve the equivalent gID problem. 

Suppose $\X, \Y$ and $\Z$ are three disjoint subsets of $\V$ and $\mathbb{A}$ is a collection of subsets of $\V$. 
We are interested in identifying $P_{\x}(\y|\z)$ from $(\G,\mathbb{A})$.
To this end, we define $\W$ to be the maximal subset of $\Z$, such that $P_{\x}(\y|\z)=P_{\x, \w}(\y|\z\setminus\mathbf{w})$. 
 \cite{shpitser2012identification} proved that such a maximal set is unique and it is given by
 \begin{align}\label{eq: wset}
     \W=\bigcup_{W'\in\Z}\big\{W'|\ P_{\x}(\y|\z)=P_{\x, w'}(\y|\z\setminus \{w'\}) \big\}.
 \end{align}
 More precisely, they showed the following result. 
\begin{theorem}[\cite{shpitser2012identification}]\label{th: ships1}
    For a given graph $\G$ and any conditional effect $P_{\x}(\y|\z)$, there exists a unique maximal set $\W=\{ W \in \Z | P_{\x}(\y|\z)=P_{\x, w}(\y|\z\setminus \{w\}) \}$ such that rule 2 of do-calculus applies to $\W$ in $\G$ for $P_{\x}(\y|\z)$.
\end{theorem}

In a special case when $\W=\Z$, it is trivial that the equivalent gID problem boils down to identifying $P_{\x, \z}(\y)$ from $(\G,\mathbb{A})$.
In the next result, we present the form of an equivalent gID problem for a general c-gID problem.

\begin{theorem} \label{th: main}
    Let $\W$ be the maximal subset of $\Z$, such that $P_{\x}(\y|\z)=P_{\x, \w}(\y|\z\setminus\mathbf{w})$. Then, $P_{\x}(\y|\z)$ is c-gID from $(\mathbb{A}, \G)$ if and only if $P_{\x, \mathbf{w}}(\y, \z\setminus\mathbf{w})$ is gID from $(\mathbb{A},\G)$.
\end{theorem}
A sketch of the proof of this Theorem appears in Section \ref{sec:proof}.
This result extends the result of \cite{shpitser2012identification} for conditional ID to c-gID.
Furthermore, Theorem \ref{th: main} allows us to develop an algorithm for solving the c-gID problem.  
Algorithm \ref{algo: c-gID} summarizes the steps of the proposed algorithm.
The algorithm consists of  two main steps:
\\
    \textbf{1.} Find the maximal set $\W\subseteq \Z$ in $\G$, such that $P_{\x}(\y|\z)=P_{\x, \w}(\y|\z\setminus\mathbf{w})$. 
    For this part, we propose function \textbf{MaxBI} presented in Algorithm \ref{algo: c-gID} that is based on Equation \eqref{eq: wset}.
    \\
    \textbf{2.} Run any sound and complete gID algorithm (e.g., the proposed algorithm by \cite{kivva2022revisiting}) for checking the gID of $P_{\x, \mathbf{w}}(\y, \z\setminus\mathbf{w})$ from $(\mathbb{A}, \G)$.

%\cite{shpitser2012identification} presented the following important result that established a bridge between the conditional ID and classical ID problem.

%\begin{theorem}[\cite{shpitser2012identification}]
%    \label{th: c-ID main}
%    Let $\W \subset \Z$ be the maximal set such that $P_{\x}(\y|\z)=P_{\x, \w}(\y|\z\setminus\mathbf{w})$. Then, $P_{\x}(\y|\z)$ is identifiable in $\G$ if and only if $P_{\x, \mathbf{w}}(\y, \z\setminus\mathbf{w})$ is identifiable in $\G$.
%\end{theorem}

%The following theorem links the c-gID problem for the conditional causal effect to the  gID problem for the unconditional causal effect.

%Note that c-gID problem for the unconditional causal effect is equivalent to the gID problem. Therefore the algorithm which determines whether the causal effect $P_{\x}(\y|\z)$ is c-gID follows immediately from the Theorem \ref{th: main}. An algorithm \textbf{C-GID} has a straightforward structure:

\begin{theorem}
    Algorithm \ref{algo: c-gID} is sound and complete. 
\end{theorem}
\begin{proof}
    The result immediately follows from Theorem \ref{th: main} since the gID algorithm is sound and complete.
\end{proof}
\begin{corollary}
    Rules of do-calculus are sound and complete for the c-gID problems.
\end{corollary}

\begin{algorithm}[t]
    \caption{c-gID}
    \label{algo: c-gID}
    \begin{algorithmic}[1]
        \State \textbf{Function C-GID}($\X,\Y, \Z, \mathbb{A}=\{\A_i\}_{i=0}^m ,\G$)
        \State \textbf{Output:} True, if $P_{\x}(\y|\z)$ is c-gID from $(\mathbb{A},\G)$.
        \State $\W \leftarrow \textbf{MaxBI}(\X, \Y, \Z, \G)$
        \State \textbf{Return} $\textbf{GID}(\X\cup\W, \Y \cup (\Z\setminus\W), \mathbb{A}, \G)$
    \end{algorithmic}
    \hrulefill
    \begin{algorithmic}[1]
        \State \textbf{Function MaxBI}($\X, \Y, \Z, \G$)
        \State \textbf{Output:} set $\W$
        \State $\W \leftarrow \emptyset$
        \For{$Z$ in $\Z$}
            \If{$(\Y \independent \Z|\X, \Z\setminus\{Z\})_{\G_{\overline{\X}, \underline{\Z\setminus\{Z\}}}}$}
                \State $\W \leftarrow \W \cup \{Z\}$
            \EndIf
        \EndFor
        \State \textbf{Return} $\W$
    \end{algorithmic}
\end{algorithm}

\begin{remark}
    Algorithm \ref{algo: c-gID} is polynomial time in the input size.
\end{remark}
In subroutine \textbf{MaxBI}, a conditional independence test is performed for each variable in $\mathbf{Z}$. Subsequently, the problem is reduced to the gID problem, which can be solved in polynomial number of steps by using any of the algorithms proposed in \cite{lee2019general, kivva2022revisiting}.

\section{Proof of the Theorem \ref{th: main}}\label{sec:proof}
In this section, we present the main steps of proof of Theorem \ref{th: main}. Further details can be found in Appendix A.
Before going into the details and purely for simpler representation, we define the following notations, $\X' := \X \cup \W$, $\Y' := \Y$, and $\Z' := \Z\setminus\W$.
Note that by the definition of $\W$ and Theorem \ref{th: ships1}, we have
$P_{\x}(\y|\z)=P_{\x'}(\y'|\z') $.

The proof consists of two main parts: sufficiency and necessity. 
In the sufficiency part, which is more straightforward, we show that if $P_{\x'}(\y', \z')$ is gID from $(\mathbb{A}, \G)$, then $P_{\x}(\y|\z)$ is c-gID.
For the reverse, which is much more involved, we use a proof by contradiction. 
That is we show if $P_{\x'}(\y', \z')$ is not gID from $(\mathbb{A}, \G)$, then $P_{x}(\y|\z)$ is also not c-gID.

\textbf{Sufficiency:} Suppose $P_{\x'}(\y', \z')$ is gID from $(\mathbb{A}, \G)$, then  the result follows immediately from the Bayes rule and the fact that $P_{\x'}(\y'| \z') = P_{\x}(\y|\z)$, i.e., 
\begin{equation}\label{eq: bayes}
    P_{x}(\y|\z) = \frac{P_{\x'}(\y', \z')}{\sum_{\y'' \in \dom{\Y}{}}P_{\x'}(\y'', \z')}.
\end{equation}

\textbf{Necessity:} Suppose that $P_{\x'}(\y', \z')$ is not gID from $(\mathbb{A}, \G)$.
To show the non-identifiability of $P_{\x}(\y|\z)=P_{\x'}(\y'|\z')$ from $(\mathbb{A}, \G)$, we construct two causal models $\M_1$ and $\M_2$ from $\mathbb{M}(\G)$, such that for each $i\in [0:m]$ and any $\mathbf{v} \in \dom{\V}{}$,
    \begin{equation*}
        Q^{\M_1}[\mathbf{A}_i](\mathbf{v}) = Q^{\M_2}[\mathbf{A}_i](\mathbf{v}), 
    \end{equation*}
but there exists a triple $(\mathbf{x}',\mathbf{y}',\mathbf{z}') \in \dom{\X'}{}\times \dom{\Y'}{}\times \dom{\Z'}{}$, such that $P_{\x'}^{\M_1}(\y'|\z') \neq P_{\x'}^{\M_2}(\y'|\z').$
      %  \begin{equation*}
     %   P_{\x'}^{\M_1}(\y'|\z') \neq P_{\x'}^{\M_2}(\y'|\z').
    %\end{equation*}

\cite{huang2006identifiability} showed that that $P_{\x'}(\y', \z')$ can be written as follows
\begin{equation*}
    P_{\x'}(\y', \z') = \sum_{\mathbf{S} \setminus (\Y' \cup \Z')}Q[\mathbf{S}](\mathbf{v}),
\end{equation*}
where $\mathbf{S} := \Anc{\Y'\cup \Z'}{\G[\V \setminus \X']}$ and the marginalization is over all variables in set $\textbf{S}\setminus (\Y' \cup \Z')$.
Suppose that $\mathbf{S}_1, \mathbf{S}_2, \dots, \mathbf{S}_n$ are the c-components of $\mathbf{S}$ in a graph $\G[\mathbf{S}]$. 
It is known by \cite{huang2006identifiability} that 
$$
Q[\mathbf{S}](\mathbf{v}) =\prod_{i=1}^n Q[\mathbf{S}_i](\mathbf{v}).
$$
Since $P_{\x'}(\y', \z')$ is not gID from $(\mathbb{A}, \G)$, using Proposition 4 and Theorem 1 in \cite{kivva2022revisiting}, we conclude that there exists $i \in [1:n]$, such that for any $j\in[0:m]$, the causal effect $Q[\mathbf{S}_i]$ is not ID from $\G[\A_j]$. 
 %Using the well-known properties of $Q[\cdot]$ function we have 

Analogously, let $\mathbf{S}':=\Anc{\Z'}{\G[\V \setminus \X']}$ and assume $\mathbf{S}'_1, \mathbf{S}'_2, \dots, \mathbf{S}'_{n'}$ are the c-components of $\mathbf{S}'$ in graph $\G[\mathbf{S}']$. 
Then, we have 
\begin{align}\label{eq: denom}
   P_{\x'}(\z') =\sum_{\mathbf{S}'\setminus \Z'} \prod_{i=1}^{n'} Q[\mathbf{S}'_i](\mathbf{v}).
\end{align}
 Consequently, we obtain the following expression
\begin{equation*}
    P_{\x'}(\y'|\z') = \frac{
    \sum_{\mathbf{S}\setminus (\Y'\cup \Z')} \prod_{i=1}^n Q[\mathbf{S}_i](\mathbf{v})
    }{\sum_{\mathbf{S}'\setminus \Z'} \prod_{i=1}^{n'} Q[\mathbf{S}'_i](\mathbf{v})
    }.
\end{equation*}

Note that $\mathbf{S}' \subseteq \mathbf{S}$ and for any $i \in [1:n]$ and $j \in [1:n']$ either $\mathbf{S}'_j$ and $\mathbf{S}_i$ are disjoint or $\mathbf{S}'_j\subseteq\mathbf{S}_i$.

Depending on the relationships between $\{Q[\mathbf{S}_i]\}_{i=1}^n$ and $\{Q[\mathbf{S}'_j]\}_{j=1}^{n'}$ and which parts are gID, in the remainder, we consider two different cases and study each one separately.

\subsection{First case}\label{sec: first case}

In this case, we assume that there exists an index $i\in[1:n]$, such that both $Q[\mathbf{S}_i]$ is not gID from $(\mathbb{A},\G)$ and  $\mathbf{S}_i\neq\mathbf{S}'_j$ for all $j\in [1:n']$. 

If we show that $P_{\x'}(\y'|\z')$ remain not c-gID even after adding additional knowledge about the distributions $\{Q[\mathbf{S}'_j]\}_{j=1}^{n'}$ to $\{Q[\mathbf{A}_k]\}_{k=0}^m$, then, we can conclude that $P_{\x'}(\y'|\z')$ is also not c-gID from $(\mathbb{A},\G)$. 
To do so, let \mbox{$\mathbb{A}':=\mathbb{A}\cup(\bigcup_{j=1}^{n'}\{\mathbf{S}'_i\})$}.

%is not c-gID from $(\mathbb{A}', \G)$, where \mbox{$\mathbb{A}'=\mathbb{A}\cup(\bigcup_{j=1}^{n'}\{\mathbf{S}'_i\})$}, then we can conclude the result. 
%Let's consider whether $P_{\x'}(\y'|\z')$ is c-gID from $(\mathbb{A}', \G)$, where \mbox{$\mathbb{A}'=\mathbb{A}\cup(\bigcup_{j=1}^{n'}\{\mathbf{S}'_i\})$}, i.e.  distributions from a set $\{Q[\mathbf{S}'_j]\}_{j \in [1:n']}$ become known too. 

Clearly, $P_{\x'}(\z')$ is c-gID from $(\mathbb{A}', \G)$ as all the terms in \eqref{eq: denom} are given in $\mathbb{A}'$. 
On the other hand, $Q[\mathbf{S}_i]$ is not gID from $(\mathbb{A}', \G)$. 
This is due to the assumptions of this setting,  that are $Q[\mathbf{S}_i]$ is not gID from  $(\mathbb{A}, \G)$ and $\mathbf{S}_i\not\subset \mathbf{S}'_j$ for all $j\in [1:n']$. The latter assumption implies that none of the additional distributions $\{Q[\mathbf{S}'_j]\}_{j=1}^{n'}$ can be used to identify $Q[\mathbf{S}_i]$. 
Since, we have established that $Q[\mathbf{S}_i]$ and consequently $P_{\x'}(\y', \z')$ are not gID from $(\mathbb{A}', \G)$, there exists two models $\M_1,\M_2\in\mathbb{M}(\G)$, such that for any $\mathbf{v} \in \dom{\V}{}$,
\begin{align*}
    & Q^{\M_1}[\mathbf{A}_j](\mathbf{v}) = Q^{\M_2}[\mathbf{A}_j](\mathbf{v}),\quad  j\in [0:m],\\
    & Q^{\M_1}[\mathbf{S}_{j'}](\mathbf{v}) = Q^{\M_2}[\mathbf{S}_{j'}](\mathbf{v}), \quad j'\in [1:n'],
\end{align*}
and there exists $(\widehat{\mathbf{x}}',\widehat{\mathbf{y}}',\widehat{\mathbf{z}}' ) \in \dom{\X'}{}\times\dom{\Y'}{}\times \dom{\Z'}{}$, such that 
\begin{equation*}
    P_{\widehat{\x}'}^{\M_1}(\widehat{\y}', \widehat{\z}') \neq P_{\widehat{\x}'}^{\M_2}(\widehat{\y}', \widehat{\z}').
\end{equation*}
Because $P_{\x'}(\z')$ is gID from $(\mathbb{A}', \G)$ and from \eqref{eq: bayes}, we have
\begin{equation*}
    P_{\widehat{\x'}}^{\M_1}(\widehat{\y'}|\widehat{\z'}) \neq P_{\widehat{\x'}}^{\M_2}(\widehat{\y'}|\widehat{\z'}).
\end{equation*}
This implies that $P_{\x'}(\y'|\z')$ is not c-gID from $(\mathbb{A}', \G)$.

\subsection{Second case} \label{sec: second case}
Suppose that there is no $i\in[1:n]$, such that both $Q[\mathbf{S}_i]$ is not gID from $(\mathbb{A}, \G)$ and $\mathbf{S}_i \neq \mathbf{S}'_j$ for all $j \in [1:n']$. 

Without loss of generality, suppose that for some $k\leq n$, all $Q[\mathbf{S}_1], Q[\mathbf{S}_2], \dots, Q[\mathbf{S}_k]$ are not gID from $(\mathbb{A}, \G)$ and the remaining $Q[\mathbf{S}_{k+1}], \dots, Q[\mathbf{S}_n]$ are gID from $(\mathbb{A}, \G)$. 
By the assumption of this case, for each $i\in [1:k]$, there exists $j_i\in[1:n']$ such that $\mathbf{S}_i=\mathbf{S}'_{j_i}$. 
Without loss generality, suppose that $j_i=i$ for all $i\in [1:k]$, i.e., $\mathbf{S}_1 = \mathbf{S}'_1$, $\dots$, $\mathbf{S}_k = \mathbf{S}'_k$. 
Therefore, $\mathbf{S}_i \subset \mathbf{S}'=\Anc{\Z'}{\G[\V\setminus\X']}$, for all $i \in [1:k]$.

To establish the result, we further consider three different sub-cases:\\
%\begin{enumerate}
 %   \item $\Y' \cap \mathbf{S}_1 \neq \emptyset$. 
    %, w.l.g $i=1$;
    %\item  $\mathbf{S}_1 \subseteq \Z'$.
    %, w.l.g. $i=1$;
   % \item $\mathbf{S}_1 \setminus (\Z'\cup \Y')\neq \emptyset$.
    %, w.l.g. $i=1$.
%\end{enumerate}
1:  $\Y' \cap \mathbf{S}_1 \neq \emptyset$, \
2:  $\mathbf{S}_1 \subseteq \Z'$, and \
3:  $\mathbf{S}_1 \setminus (\Z'\cup \Y')\neq \emptyset$.

\begin{remark}
    Although, the above sub-cases may have non-empty intersection, it is easy to see that their union covers all possible scenarios of the second case.
\end{remark}

\subsubsection{Sub-case 1: $\Y'\cap \mathbf{S}_1 \neq \emptyset$} \label{sec: first subcase}
Let $Y$ denotes a random variable in $\Y' \cap \mathbf{S}_1$. 
Since $Y$ belongs to $\mathbf{S}_1=\mathbf{S}'_1$,  $Y$ is an ancestor of a variable in $\Z'$ in a graph $\G[\V\setminus \X']$, i.e. $Y \in \Anc{\Z'}{\G[\V \setminus \X']}=\mathbf{S}'$. 
This implies that
\begin{equation}\label{eq: sub_case1}
    \hspace{-.2cm} P_{\x'}(y|\z')\!\!=\! \frac{
        P_{\x'}(y, \z')
    }{
        P_{\x'}(\z')
    } \!=\! \frac{
        \sum_{\mathbf{S}'\setminus (\Z'\cup \{Y\})} \prod_{i=1}^{n'} Q[\mathbf{S}'_i](\mathbf{v})
    }{
        \sum_{\mathbf{S}'\setminus \Z'} \prod_{i=1}^{n'} Q[\mathbf{S}'_i](\mathbf{v})
    }
\end{equation}
We prove this sub-case by showing that $P_{\x'}(y|\z')$ is not c-gID from $(\mathbb{A}, \G)$ and subsequently $P_{\x'}(\y'|\z')$ is not c-gID from $(\mathbb{A}, \G)$. 
To this end, first, we prove \textbf{I:} $Q[\{Y\}| \mathbf{S}'_1\setminus \{Y\}]$ is not c-gID from $(\mathbb{A}, \G_{\underline{\{Y\}}})$, and then use it to show \textbf{II:} $Q[\{Y\}|\mathbf{S}'_1\setminus \{Y\}]$ is not c-gID from $(\mathbb{A}, \G)$. Finally, we show \textbf{III:} $P_{\x'}(y|\z')$ is not c-gID from $(\mathbb{A}, \G)$.

\textbf{I:} 
In graph $\G_{\underline{\{Y\}}}$ and using \eqref{eq: sub_case1}, we obtain
\begin{equation*}
\begin{split}
     Q\big[\{Y\}| \mathbf{S}'\setminus \{Y\}\big] &= \frac{
        \prod_{i=1}^{n'} Q[\mathbf{S}'_i]
    }{
        \sum_{Y}\prod_{i=1}^{n'} Q[\mathbf{S}'_i]
    }  \\
    & = \frac{
        Q[\mathbf{S}_1]
    }{
        \sum_{Y}Q[\mathbf{S}_1] 
    } = Q\big[\{Y\}|\mathbf{S}_1\setminus \{Y\}\big].
\end{split}
\end{equation*}
Recall that $\textbf{S}_1=\textbf{S}'_1$.
Next result shows that $Q[\{Y\}|\mathbf{S}_1\setminus \{Y\}]$ is not c-gID from $(\mathbb{A}, \G_{\underline{\{Y\}}})$ because $Q[\mathbf{S}_1]$ is not gID from $(\mathbb{A}, \G_{\underline{\{Y\}}})$. A proof is presented in Appendix A.

\begin{lemma} \label{lemma: construct models subcase 1}
    Suppose $\mathbf{L}\subseteq \V$ is a single c-component, such that $\mathbf{L} = \mathbf{L}'\cup\mathbf{L}''$ for some disjoint sets $\mathbf{L}'$ and $\mathbf{L}''$. 
    $Q[\mathbf{L}'|\mathbf{L}'']$ is c-gID from $(\mathbb{A}, \G)$ if and only if $Q[\mathbf{L}'\cup \mathbf{L}'']$ is gID from $(\mathbb{A}, \G)$.
\end{lemma}

\textbf{II:}
\cite{shpitser2006identification} showed the following result for a non-identifiable causal effect. 
\begin{lemma}[\cite{shpitser2006identification}]\label{lemma: Q ID}
    Suppose \mbox{$\mathbf{L} \subseteq\A \subseteq \V$}. 
    $Q[\mathbf{L}]$ is not identifiable from $\G[\A]$ if and only if there exists at least one $\mathbf{L}$-rooted c-forest  $\mathcal{F}$ with the set of observed variables $\B$ such that $\mathbf{L}\subsetneq \B \subseteq \A$, the bidirected edges of $\widehat{\F}[\B]$ form a spanning tree, and $\widehat{\mathcal{F}}[\mathbf{L}]$ is a connected graph with respect to the bidirected edges.
\end{lemma}

On the other hand, because $Q[\mathbf{S}_1]$ is not gID from $(\mathbb{A}, \G)$, by the results of \cite{kivva2022revisiting},  $Q[\mathbf{S}_1]$ is not ID from $\G[\A_i]$ for all $i\in [0:m]$. 
Lemma \ref{lemma: Q ID} implies that adding or removing outgoing edges from $Y \in \mathbf{S}_1$ will not affect the non-identifiability of $Q[\mathbf{S}_1]$ from $\G[\A_i]$ for all $i\in [0:m]$. Thus, we have $Q[\mathbf{S}_1]$ is not gID from $(\mathbb{A}, \G_{\underline{\{Y\}}})$. This means that exists two causal models $\M_1$ and $\M_2$ from $\mathbb{M}(\G_{\underline{\{Y\}}})$ which are consistent with all known distributions but disagree on the causal effect $Q[\mathbf{S}_1]$, i.e., there exists $\widetilde{\textbf{v}}\in\dom{\V}{}$ such that
$$
Q^{\mathcal{M}_1}[\mathbf{S}_1](\widetilde{\textbf{v}})\neq Q^{\mathcal{M}_2}[\mathbf{S}_1](\widetilde{\textbf{v}}).
$$
Note that $\mathbb{M}(\G_{\underline{\{Y\}}}) \subset \mathbb{M}(\mathbb{\G})$ which in combination with the above result yield that $Q[\{Y\}|\mathbf{S}\setminus \{Y\}]$ is not c-gID from $(\mathbb{A}, \G)$.

\textbf{III:}
To prove this part, we first present the following result. A proof is provided in Appendix A.
%\begin{itemize}
%    \item $Q[\{Y\}|\mathbf{S}'\setminus \{Y\}]$ is not c-gID from $(\mathbb{A}, \G)$.
%    \item $P_{\x'}(y|\z')$ is not c-gID from $(\mathbb{A}, \G)$.
%\end{itemize}
\begin{lemma}\label{lemma: eliminate var in cond}
    Suppose that $\X$, $\Y$ and $\Z$ are disjoint subsets of $\V$ in graph $\G$ and variables $Z_1 \in \Z$, $Z_2 \in \Y \cup \Z$, such that there is a directed edge from $Z_1$ to $Z_2$ in $\G$. If the causal effect $P_{\x}(\y|\z)$ is not c-gID from $(\mathbb{A}, \G)$, then the causal effect $P_{\x}(\y|\z\setminus\{z_1\})$ is also not c-gID from $(\mathbb{A}, \G)$. 
\end{lemma}

Note that $P_{\x'}(\mathbf{s}') = Q[\mathbf{S}']$ since $\mathbf{S}' = \Anc{\mathbf{S}'}{\G[\V \setminus \X']}$.  Therefore, by the definition of $Q$-notation, we have
\begin{equation*}
    Q\big[\{Y\}|\mathbf{S}'\setminus \{Y\}\big] = P_{\x'}(y|\mathbf{s}'\setminus\{y\}), 
\end{equation*}

which is shown to be not c-gID from $(\mathbb{A}, \G)$ in part \textbf{II}.
In the remainder of this part of our proof, we introduce a set of nodes in $\textbf{S}'$ that satisfy the condition in Lemma \ref{lemma: eliminate var in cond} and thus, can be eliminated without affecting the non-identifiability. Bellow, we show that the nodes in $\mathbf{S}'\setminus (\Z' \cup \{Y\})$ satisfy Lemma \ref{lemma: eliminate var in cond}'s condition and by deleting them, we conclude that $P_{\x'}(y|\z')$ is not c-gID from $(\mathbb{A}, \G)$.

Recall that $\mathbf{S}'=\Anc{\Z'}{\G[\V \setminus \X']}$ which means that from any node in $\mathbf{S}'\setminus (\Z' \cup \{Y\})$, there exists a directed path to a node in $\Z'$ in graph $\G[\V \setminus \X']$. 
We assign a real number to each node in $\mathbf{S}'\setminus (\Z' \cup \{Y\})$, namely, the length of its shortest path to set $\Z$.
Let $(W_1, W_2, \dots, W_\eta)$ denote the nodes in $\mathbf{S}'\setminus (\Z' \cup \{Y\})$ sorted in a descending order using their assigned numbers. 
Observe that for any $i \in [1:\eta]$, there is a direct edge from $W_i$ to a node in $\{Y\}\cup \Z' \bigcup_{j=i+1}^\eta \{W_j\}$. 
In other words, Lemma \ref{lemma: eliminate var in cond} allows us to delete $W_{i}$ from $\textbf{S}'\setminus\big(\{Y\}\bigcup_{j=1}^{i-1}\{W_j\}\big)$ without violating the non-identifiability. 
%that is why we can apply Lemma \ref{lemma: eliminate var in cond} to the $P_{\x'}(y|\mathbf{s}'\setminus\{y\})$ and eliminate variables $W_1, W_2, \dots, W_\eta$ one by one from $\mathbf{s}'\setminus\{y\}$. The latter proofs that $P_{\x'}(y|\z')$ is not c-gID from $(\mathbb{A}, \G)$ and thus $P_{\x'}(\y'|\z')$ is not c-gID from $(\mathbb{A}, \G)$ as well.

\subsubsection{Sub-case 2: $\mathbf{S}_1 \subseteq \Z'$} \label{sec: second subcase}

In this sub-case, we prove non-identifiability of $P_{\x'}(\y'|\z')$  from $(\mathbb{A}, \G)$ in two steps: \textbf{I:} we introduce a conditional causal effect that is not c-gID from $(\mathbb{A}, \G)$. \textbf{II:} Analogous to the previous sub-case, we apply Lemma \ref{lemma: eliminate var in cond} to prune this causal effect and conclude the result.

\textbf{I:} 
Let $Z'$ be a node in $\mathbf{S}_1$. 
Recall that $\W$ is the maximal set such that $P_{\x,\w}(\y|\z\setminus\w)=P_{\x}(\y|\z)$, which means that we can not apply the second rule of do-calculus to $Z'$ in $\G$ for $P_{\x'}(\y'|\z')$, i.e.,
$$
(\Y' \notindependent Z'|\X',\Z'\setminus\{Z'\})_{\G_{\overline{\X'},\underline{\{Z'\}}}}.
$$
This implies that there exists at least a unblocked backdoor path from $Z'$ to $\Y'$ given $\X'\cup \Z'\setminus\{Z'\}$.
We use $p$ to denote an unblocked path from $Z'$ to $\Y'$ with the least number of colliders. 
Path $p$ satisfies the following properties:\\
    1. If path $p$ contains a chain $W'\rightarrow W \rightarrow W''$ or a fork $W' \leftarrow W \rightarrow W''$, then node $W$ does not belong to any of the sets $\X'$, $\Z'$ or $\Y'$.\\
    2. If path $p$ contains a collider $W'\rightarrow W \leftarrow W''$, then there is a directed path $p_W$ from $W$ to a node in $\Z'$. 
    Moreover, none of the intermediate nodes in the path $p_W$ belong to the set $\X'\cup\Z'\cup\Y'$.\\
    3. Path $p$ does not contain any node from the set $\X'$.

Proofs of the above statements are provided in Appendix A.
Suppose $\mathbf{F}$ is a set of all colliders on the path $p$. 
We use $\mathcal{P}$ to denote a collection of paths $\{p\} \cup \{p_W|W \in \mathbf{F}\}$ and use $\D$ to denote the set of all nodes on the paths in $\mathcal{P}$ excluding the ones in $\Z'$.
Given the above definitions, we are ready to introduce the non-identifiable conditional causal effect in the next result.

\begin{lemma}\label{lemma: construct models subcase 2}
    Let $\mathbf{S}: = \Anc{\Y'\cup \Z'}{\G[\V \setminus \X']}$ and $\D$ denote the set defined above. Then,
    % {
    % \color{red}
    % \begin{equation*} 
    %     P_{\x'}(\dd\cap \y'|\mathbf{s}\setminus \dd)=\frac{
    %     \sum_{\D \setminus \Y'} Q[\mathbf{S}]
    %     }{
    %     \sum_{\D} Q[\mathbf{S}]
    %     }
    % \end{equation*}
    % }
    % {
    % \color{teal}
    \begin{equation}\label{eq: lemma-con}
        P_{\x'}(\dd|\mathbf{s}\setminus \dd)=\frac{ 
        Q[\mathbf{S}]
        }{
        \sum_{\D} Q[\mathbf{S}]
        } = Q[\D|\mathbf{S}\setminus \D]
    \end{equation}
    % }
    
    is not c-gID from $(\mathbb{A}, \G)$.
\end{lemma}
Proof of this lemma is presented in Appendix A.

\textbf{II:}
 In order to complete the proof of this part, besides Lemma \ref{lemma: eliminate var in cond}, we require the following technical lemmas. 
\begin{lemma} \label{lemma: move cond}
    Suppose that $\X$, $\Y$ and $\Z$ are disjoint subsets of $\V$ and $Z \in \Z$.
    If the conditional causal effect $P_{\x}(\y|\z)$ is not c-gID from $(\mathbb{A}, \G)$, the conditional causal effect 
    $P_{\x}(\y, z|\z\setminus\{z\})$ is not c-gID from $(\mathbb{A}, \G)$ as well. 
\end{lemma}
\begin{proof}
    Proof is by contradiction. Suppose that $P_{\x}(\y, z|\z\setminus\{z\})$ is c-gID from $(\mathbb{A}, \G)$. This implies that $P_{\x}( z|\z\setminus\{z\})$ is also c-gID from $(\mathbb{A}, \G)$. Applying Bayes rule yields
    \begin{equation*}
        P_{\x}(\y|\z) = \frac{P_{\x}(\y, z|\z\setminus\{z\})}{P_{\x}( z|\z\setminus\{z\})},
    \end{equation*}
    which results in c-gID of $P_{\x}(\y|\z)$ from $(\mathbb{A}, \G)$. This contradicts the non-identifiability assumption on $P_{\x}(\y|\z)$.
\end{proof}
% {
% \color{teal}
\begin{lemma}\label{lemma: for_the_main}
Suppose that $\X$, $\Y$ and $\Z$ are disjoint subsets of $\mathbf{V}$ in graph $\G$ and variables $Y_1 \in \Y$, $Y_2 \in \Y \cup \Z$, such that there is a directed edge from $Y_1$ to $Y_2$ in $\G$. If the causal effect $P_{\x}(\y|\z)$ is not c-gID from $(\mathbb{A}, \G)$, then the causal effect $P_{\x}(\y\setminus\{y_1\}|\z)$ is also not c-gID from $(\mathbb{A}, \G)$. 
\end{lemma}
Proof of this lemma is presented in Appendix A.
% }

Recall that the goal is to prune the conditional causal effect in \eqref{eq: lemma-con} to get $P_{\x'}(\y'|\z')$. 
We do this in two pruning steps: first using Lemma \ref{lemma: move cond} and then via Lemmas \ref{lemma: eliminate var in cond}, {\ref{lemma: for_the_main}}.
Let $\Y'' := \Y' \setminus \D$. 
Recall that $\mathbf{S}=\Anc{\Y', \Z'}{\G[\V\setminus\X']}$. 
It is easy to see that $\Y''$ is a subset of $\mathbf{S}\setminus \D$ and thus we can apply Lemma \ref{lemma: move cond} to the causal effect 
% {\color{red}$P_{\x'}(\dd\cap \y'|\mathbf{s}\setminus \{\dd\})$}
% {\color{teal}
$P_{\x'}(\dd|\mathbf{s}\setminus \dd)$
% } 
and conclude that 
% {\color{red}$P_{\x'}(\y'|\mathbf{s}\setminus(\dd \cup \y'))$}{
% \color{teal}
$P_{\x'}(\dd\cup\y'|\mathbf{s}\setminus(\dd \cup \y'))$
% }
is not c-gID from $(\mathbb{A}, \G)$.

To use Lemmas \ref{lemma: eliminate var in cond}, \ref{lemma: for_the_main} for the second pruning steps, we use similar type of argument as in the first sub-case. 
More precisely, using the fact that there exists a direct path for each node in $\mathbf{S} \setminus (\Z' \cup \Y')$ to a node in $\Z' \cup \Y'$, we sort the nodes in
% {
% \color{red}
% $$
% \W':=\mathbf{S} \setminus (\Z' \cup \Y' \cup \D)
% $$
% }
% {
% \color{teal}
$$
\W':=\mathbf{S} \setminus (\Z' \cup \Y')
$$
% }
in a descending order based on the length of their corresponding shortest direct path to the set $\Z' \cup \Y'$.
We denote these sorted nodes by $(W'_1, W'_2, \dots, W'_{\eta'})$. 
Note that for any $i \in [1:\eta']$, there exists a direct edge from $W'_i$ to a node in $\Y' \cup \Z' \cup \{W'_j\}_{j=i+1}^{\eta'}$. 

Since $\W'$ is a subset of $\mathbf{S}\setminus (\Z' \cup \Y')$, similar to the second sub-case, we apply Lemmas \ref{lemma: eliminate var in cond}, \ref{lemma: for_the_main} to the causal effect 
% {\color{red}$P_{\x'}(\y'|\mathbf{s}\setminus(\dd \cup \y'))$}
% {\color{teal}
$P_{\x'}(\dd\cup\y'|\mathbf{s}\setminus(\dd \cup \y'))$
% }
and remove variables $(W'_1, \dots, W'_{\eta'})$ one by one from the $P_{\x'}(\dd\cup\y'|\mathbf{s}\setminus(\dd \cup \y'))$. 
From definitions of $\D$ and $\Z'$, we have $\D\cap\Z'=\emptyset$, which means 
$$
\mathbf{S}\setminus(\W' \cup \Y'\cup \D) = \Z'.
$$
Therefore, after removing all nodes of $\W'$ from the set $\mathbf{S}\setminus(\D \cup \Y')$ without affecting the non-identifiability of 
% {\color{red}$P_{\x'}(\y'|\mathbf{s}\setminus(\dd \cup \y'))$} 
% {\color{teal}
$P_{\x'}(\dd\cup\y'|\mathbf{s}\setminus(\dd \cup \y'))$
% }
, we can claim that $P_{\x'}(\y'|\z')$ is not c-gID from $(\mathbb{A}, \G)$.

\subsubsection{Sub-case 3: $\mathbf{S}_1 \setminus (\Z'\cup\Y')\neq \emptyset$} \label{sec: third subcase}

The proof of this sub-case is quite similar to the second sub-case with a few twists. 
Let $T$ be an arbitrary node in $\mathbf{S}_1 \setminus (\Z'\cup\Y')$. Since $\textbf{S}_1$ is a subset of the ancestors of $\Z'$, then there exists a directed path from $T$ to the set $\Z'$. 
Let $p_T$ denote the shortest directed path from node $T$ to a node $Z'$ in the set $\Z'$. Analogous to the second sub-case, we define $\widetilde{p}$ to be an unblocked backdoor path from $Z'$ to $\Y'$ given $\X', 
\Z'\setminus \{Z'\}$ with the least number of colliders. 
Path $\widetilde{p}$ satisfies the following properties: \\
    1. Assume that path $\widetilde{p}$ contains a chain $W'\rightarrow W \rightarrow W''$ or a fork $W' \leftarrow W \rightarrow W''$, then $W$ does not belong to any of the sets $\X'$, $\Z'$ or $\Y'$.\\
    2. Assume that path $\widetilde{p}$ contains an inverted fork $W'\rightarrow W \leftarrow W''$, then there is a directed path $p_W$ from the node $W$ to a node in the set $\Z'$. Moreover, none of the intermediate nodes on this path $p_W$ belong to set $\X'\cup\Z'\cup\Y'$.\\
    3. Path $\widetilde{p}$ does not contain any node from the set $\X'$

Proofs of the above statements are provided in Appendix A.
Let $\widetilde{\mathbf{F}}$ be the set of all colliders on the path $\widetilde{p}$. 
Define $\widetilde{\mathcal{P}}:= \{\widetilde{p}\}\cup\{p_T\}\cup \{\widetilde{p}_W|W \in \widetilde{\mathbf{F}}\}$ and $\widetilde{\D}$ to be a set containing all nodes on the paths from $\widetilde{\mathcal{P}}$ excluding the nodes in $\Z'$.

\begin{lemma} \label{lemma: construct models subcase 3}
    Let $\mathbf{S}:=\Anc{\Y', \Z'}{\G[\V \setminus \X']}$ and $\widetilde{\D}$ denote the set defined above. Then,
    % { 
    % \color{red}
    % \begin{equation*}
    %     P_{\x'}(\widetilde{\dd}\cap \y'|\mathbf{s}\setminus \widetilde{\dd})=\frac{
    %     \sum_{\widetilde{\D} \setminus \Y'} Q[\mathbf{S}]
    %     }{
    %     \sum_{\widetilde{\D}} Q[\mathbf{S}]
    %     }
    % \end{equation*}
    % }
    % {
    % \color{teal}
    \begin{equation*}
        P_{\x'}(\widetilde{\dd}|\mathbf{s}\setminus \widetilde{\dd})=\frac{
        Q[\mathbf{S}]
        }{
        \sum_{\widetilde{\D}} Q[\mathbf{S}]
        } = Q[\widetilde{\D}|\mathbf{S}\setminus\widetilde{\D}]
    \end{equation*}
    % }
    is not c-gID from $(\mathbb{A}, \G)$.
\end{lemma}
A proof for this lemma is presented in Appendix A.
The remainder of the proof of this sub-case is identical to the proof of the second sub-case. 
 %Further we use the arguments completely identical as in the Section \ref{sec: second subcase}, and finally we obtain that $P_{\x'}(\y'|\z')$ is not c-gID from $(\mathbb{A}, \G)$. 

In both cases considered in Sections \ref{sec: first case}-\ref{sec: second case}, we proved that $P_{\x'}(\y'|\z')$ is not c-gID from $(\mathbb{A}, \G)$. 
Recall that $P_{\x}(\y|\z)=P_{\x'}(\y'|\z')$. 
This concludes the proof of the necessity part of \mbox{Theorem \ref{th: main}}.

\textbf{Summing up:}
Recall that the necessity part required us to show when $P_{\x'}(\y', \z')$ is not gID from $(\A, \G)$,  $P_{\x}(\y|\z)$ is not c-gID from $(\A, \G)$.
In the sufficiency part, had to show that $P_{\x}(\y|\z)$ is c-gID from $(\A, \G)$ whenever $P_{\x'}(\y', \z')$ is gID from $(\A, \G)$. 
These two results together conclude the proof of Theorem \ref{th: main}.

\section{Conclusion}
We considered the problem of identifying a conditional causal effect from a causal graph $\G$ and a particular set of known observational/interventional distributions in the form of $Q$-notations. 
We called this problem c-gID and showed that any c-gID problem has an equivalent g-ID problem. 
%There exists sound and complete algorithm for solving the latter problem in literature. 
%We proved that an c-gID problem is equivalent to the specific g-ID problem. 
Using this equivalency, we proposed the first sound and complete algorithm for solving c-gID problem.
%that generalizes both algorithms of \cite{kivva2022revisiting} and \cite{shpitser2012identification}. 

% References
\bibliography{kivva_47}

\begin{thebibliography}{15}
\providecommand{\natexlab}[1]{#1}
\providecommand{\url}[1]{\texttt{#1}}
\expandafter\ifx\csname urlstyle\endcsname\relax
  \providecommand{\doi}[1]{doi: #1}\else
  \providecommand{\doi}{doi: \begingroup \urlstyle{rm}\Url}\fi

\bibitem[Bareinboim and Pearl(2012)]{bareinboim2012causal}
Elias Bareinboim and Judea Pearl.
\newblock Causal inference by surrogate experiments: Z-identifiability.
\newblock In \emph{Proceedings of the Twenty-Eighth Conference on Uncertainty
  in Artificial Intelligence}, page 113–120, Arlington, Virginia, USA, 2012.
  AUAI Press.

\bibitem[Bareinboim and Pearl(2014)]{bareinboim2014transportability}
Elias Bareinboim and Judea Pearl.
\newblock Transportability from multiple environments with limited experiments:
  Completeness results.
\newblock \emph{Advances in neural information processing systems}, 27, 2014.

\bibitem[Bareinboim and Tian(2015)]{bareinboim2015recovering}
Elias Bareinboim and Jin Tian.
\newblock Recovering causal effects from selection bias.
\newblock In \emph{Proceedings of the AAAI Conference on Artificial
  Intelligence}, volume~29, 2015.

\bibitem[Correa et~al.(2021)Correa, Lee, and Bareinboim]{correa2021nested}
Juan Correa, Sanghack Lee, and Elias Bareinboim.
\newblock Nested counterfactual identification from arbitrary surrogate
  experiments.
\newblock \emph{Advances in Neural Information Processing Systems},
  34:\penalty0 6856--6867, 2021.

\bibitem[Huang and Valtorta(2006)]{huang2006identifiability}
Yimin Huang and Marco Valtorta.
\newblock Identifiability in causal bayesian networks: A sound and complete
  algorithm.
\newblock In \emph{AAAI}, pages 1149--1154, 2006.

\bibitem[Kivva et~al.(2022)Kivva, Mokhtarian, Etesami, and
  Kiyavash]{kivva2022revisiting}
Yaroslav Kivva, Ehsan Mokhtarian, Jalal Etesami, and Negar Kiyavash.
\newblock Revisiting the general identifiability problem.
\newblock In \emph{The 38th Conference on Uncertainty in Artificial
  Intelligence}, 2022.

\bibitem[Lee et~al.(2019)Lee, Correa, and Bareinboim]{lee2019general}
Sanghack Lee, Juan~D Correa, and Elias Bareinboim.
\newblock General identifiability with arbitrary surrogate experiments.
\newblock In \emph{Uncertainty in Artificial Intelligence}, pages 389--398.
  PMLR, 2019.

\bibitem[Lee et~al.(2020)Lee, Correa, and Bareinboim]{lee2020general}
Sanghack Lee, Juan Correa, and Elias Bareinboim.
\newblock General transportability--synthesizing observations and experiments
  from heterogeneous domains.
\newblock In \emph{Proceedings of the AAAI Conference on Artificial
  Intelligence}, volume~34, pages 10210--10217, 2020.

\bibitem[Mokhtarian et~al.(2022)Mokhtarian, Jamshidi, Etesami, and
  Kiyavash]{mokhtarian2022causal}
Ehsan Mokhtarian, Fateme Jamshidi, Jalal Etesami, and Negar Kiyavash.
\newblock Causal effect identification with context-specific independence
  relations of control variables.
\newblock In \emph{International Conference on Artificial Intelligence and
  Statistics}, pages 11237--11246. PMLR, 2022.

\bibitem[Pearl(1995)]{pearl1995causal}
Judea Pearl.
\newblock Causal diagrams for empirical research.
\newblock \emph{Biometrika}, 82\penalty0 (4):\penalty0 669--688, 1995.

\bibitem[Pearl(2009)]{pearl2009causality}
Judea Pearl.
\newblock \emph{Causality}.
\newblock Cambridge university press, 2009.

\bibitem[Shpitser and Pearl(2006{\natexlab{a}})]{shpitser2006identification}
Ilya Shpitser and Judea Pearl.
\newblock Identification of joint interventional distributions in recursive
  semi-markovian causal models.
\newblock In \emph{Proceedings of the National Conference on Artificial
  Intelligence}, volume~21, page 1219, 2006{\natexlab{a}}.

\bibitem[Shpitser and Pearl(2006{\natexlab{b}})]{shpitser2012identification}
Ilya Shpitser and Judea Pearl.
\newblock Identification of conditional interventional distributions.
\newblock \emph{Proceedings of the 22nd Conference on Uncertainty in Artificial
  Intelligence}, 2006{\natexlab{b}}.

\bibitem[Tian(2004)]{tian2004identifying}
Jin Tian.
\newblock Identifying conditional causal effects.
\newblock In \emph{Proceedings of the 20th conference on Uncertainty in
  artificial intelligence}, pages 561--568, 2004.

\bibitem[Tikka et~al.(2021)Tikka, Hyttinen, and Karvanen]{tikka2019causal}
Santtu Tikka, Antti Hyttinen, and Juha Karvanen.
\newblock Causal effect identification from multiple incomplete data sources: A
  general search-based approach.
\newblock \emph{Journal of Statistical Software}, 2021.

\end{thebibliography}

\appendix
\onecolumn

\title{On Identifiability of Conditional Causal Effects\\(Supplementary Material)}

\maketitle
\section{TECHNICAL PROOFS} \label{sec:app}

\subsection{NON c-gID CAUSAL EFFECTS}

For proving Lemma \ref{lemma: construct models subcase 1}, Lemma \ref{lemma: construct models subcase 2} and Lemma \ref{lemma: construct models subcase 3}, it suffices to introduce two models  that agree on the known distributions but disagree:
\begin{itemize}
    \item on the causal effect $Q[\mathbf{L}'|\mathbf{L}'']$ (for Lemma \ref{lemma: construct models subcase 1}),
    
    \item on the causal effect $P_{\x'}(\mathbf{d}|\mathbf{s}\setminus \mathbf{d})$ (for Lemma 4),
    
    \item on the causal effect $P_{\x'}(\widetilde{\mathbf{d}}|\mathbf{s}\setminus \widetilde{\mathbf{d}})$ (for Lemma 6).
\end{itemize}
To do so, we require a result from \citep{kivva2022revisiting} and couple of definitions and notations which we present in the next section. 

\subsubsection{Baseline Models} \label{sec: baseline}
In this section, we present two models which we use as our baseline models for proving the non-identifiability parts. 

\begin{theorem}[Theorem 1 \cite{kivva2022revisiting}]\label{th: gid main}
    Suppose $\widecheck{\mathbf{S}} \subseteq \V$ is a single c-component. 
    $Q[\widecheck{\mathbf{S}}]$ is gID from $(\mathbb{A}, \G)$ if and only if there exists $\A \in \mathbb{A}$ such that $\widecheck{\mathbf{S}} \subseteq \A$ and $Q[\widecheck{\mathbf{S}}]$ is ID in $\G[\A]$.
\end{theorem}

To introduce the baseline models, we use the models from the proof of Theorem 1 in \citep{kivva2022revisiting}.
Note that in the proof of Lemma \ref{lemma: construct models subcase 2} and \ref{lemma: construct models subcase 3}, we use $\mathbf{S}_1$ and $\widecheck{\mathbf{S}}$ interchangeably, i.e., $\widecheck{\mathbf{S}} = \mathbf{S}_1$. 

Suppose that $Q[\widecheck{\mathbf{S}}]$ is not gID from $(\mathbb{A}, \G)$ and there exists $i \in [0, m]$, such that $\widecheck{\mathbf{S}} \subset \mathbf{A}_i$. Without loss of generality, let $\widecheck{\mathbf{S}} \subset \mathbf{A}_i$ for $i \in [0, \widecheck{k}]$ and $\widecheck{\mathbf{S}}\nsubseteq 
\mathbf{A}_i$ for $i \in [\widecheck{k}+1, m]$.
This allows us to define a particular graph which we use throughout our proof. 
More precisely, under these assumptions, Lemma \ref{lemma: Q ID} and the above theorem guarantee that for each $i \in [0, \widecheck{k}]$, there exists a $\widecheck{\mathbf{S}}$-rooted c-forest $\F_i$ over a subset of observed variables $\B_i$ ($\widecheck{\mathbf{S}} \subsetneq \B_i \subseteq \A_i$) such that $\F_{0}[\widecheck{\mathbf{S}}]=\F_{j}[\widecheck{\mathbf{S}}]$ for $j\in[1,\widecheck{k}]$.
 In words, induced subgraphs of $\F_{i}$s over the set $\widecheck{\mathbf{S}}$ are the same. 
 We define graph $\G'$ as the union of all the subgraphs in $\{\F_{i}\}_{i=0}^{\widecheck{k}}$ with the observed variables $\widecheck{\V}:=\bigcup_{i=0}^{\widecheck{k}}\B_{i}$ and the  unobserved variables which we denoted by $\widecheck{\U}$.

To properly define a SEM $\M$ over a causal graph $\G$, it suffices to define the domain set of each node $X$ in $\G$ with its associated conditional distribution $P(X|\Pa{X}{\G})$. 
Note that if for some variable $X$ in $\G$, its domain $\dom{X}{}$ or $P(X|\Pa{X}{\G})$ are not specified, then by default, we assume \mbox{$\dom{X}{} := \{0\}$} and $P(X=0|\Pa{X}{\G})=1$. 

\renewcommand{\V}{\widecheck{\mathbf{V}}}
\renewcommand{\U}{\widecheck{\mathbf{U}}}
\renewcommand{\T}{\widecheck{\mathbf{T}}}
\renewcommand{\u}{\widecheck{\mathbf{u}}}
\renewcommand{\k}{\widecheck{k}}

Let $U_0$ be an unobserved variable from subgraph $\F_0$ that has one child in $\widecheck{\mathbf{S}}$ and one child in $\T:=\V\setminus \widecheck{\mathbf{S}}$.
In high-level, our baseline models $\M_1$ and $\M_2$ have the same distributions over all variables in graph $\G$ except the variable $U_0$. Especially, 
\begin{align}
    & P^{M_1}(V|\Pa{V}{\G}) = P^{\M_2}(V|\Pa{V}{\G}),\quad V \in \mathbf{V}, \forall \label{eq: M_1(V) = M_2(V)}\\
    & P^{\M_1}(U) = P^{\M_2}(U) = \frac{1}{|\dom{U}{}|}, \quad \forall U \in \mathbf{U}\setminus \{U_0\}, \label{eq: M_1(U) = M_2(U)}
\end{align}
where $|\cdot|$ denotes the cardinality of a given set. For the sake of brevity, we drop the superscripts $\M_1$ and $\M_2$ for the distributions in Equations (\ref{eq: M_1(V) = M_2(V)}) and (\ref{eq: M_1(U) = M_2(U)}). 
We denote the domain of variable $U_0$ to be $\dom{U_0}{} := \{ \gamma_1, \dots, \gamma_d\}$, where $\gamma_j$s are vectors and $d$ is an integer number to be defined later.  
In model $\M_1$, we define $U_0$ to have uniform distribution over $\dom{U_0}{}$, i.e., $P^{\M_1}(U_0 = \gamma_j) = 1/d$. 
In model $\M_2$, we define $P^{\M_2}(U_0 = \gamma_j) := p_j$, where $j \in [1:d]$ and 
\begin{align*}
    & \sum_{j=1}^d p_j = 1, \\
    & p_j > 0 \quad \forall j \in [1:d].
\end{align*}

For $i \in [0:m]$, $j \in [1:d]$, $u_0 = \gamma_j$ and any $\mathbf{v} \in \dom{\mathbf{V}}{}$, we define:
\begin{equation} \label{eq: theta and eta}
\begin{gathered}
    \theta_{i, j}(\mathbf{v}) := \sum_{\mathbf{U} \setminus \{U_0\}} \prod_{X \in \A_i} P(x \mid \Pa{X}{\G}) \prod_{U\in \mathbf{U} \setminus \{U_0\}} P(u), \\
    \eta_j(\mathbf{v}) := \sum_{\mathbf{U}\setminus \{U_0\}} \prod_{X \in \widecheck{\mathbf{S}}} P(x \mid \Pa{X}{\G}) \prod_{U\in \mathbf{U}\setminus \{U_0\}} P(u).
\end{gathered}
\end{equation}

Note that in the above equations, $u_0=\gamma_j$ may appear as a parent of an observed variable. 
Using the above definitions, we can re-write the $Q$-notation in Equation (\ref{eq: Q}) as follows
\begin{align}\label{eq: M1 Q[A_i] through thetha}
    & Q^{\M_1}[\mathbf{A}_i](\mathbf{v}) 
    = \sum_{j=1}^d \frac{1}{d} \theta_{i,j}(\mathbf{v}), \\ \label{eq: M_2 Q[A_i] through thetha}
    & Q^{\M_2}[\mathbf{A}_i](\mathbf{v}) 
    = \sum_{j=1}^d p_j \theta_{i,j}(\mathbf{v}), \\
    \label{eq: M_1 Q[S_1] through eta}
    & Q^{\M_1}[\widecheck{\mathbf{S}}](\mathbf{v}) 
    = \sum_{j=1}^d \frac{1}{d} \eta_j(\mathbf{v}), \\
    \label{eq: M_2 Q[S_1] through eta}
    & Q^{\M_2}[\widecheck{\mathbf{S}}](\mathbf{v}) 
    = \sum_{j=1}^d p_j \eta_j(\mathbf{v}).
\end{align}

Denote the set of unobserved variables in $\G'[\widecheck{\mathbf{S}}]$ by $\U^{\widecheck{\mathbf{S}}}$ and its complement set in $\U \setminus (\U^{\widecheck{\mathbf{S}}} \cup \{U_0\})$ by $\U^{\T}$. 
For each $i \in [0:\k]$, let $\widecheck{T}_i$ be a node in $\B_i\setminus \widecheck{\mathbf{S}}$ such that $\Ch{\widecheck{T}_i}{\F_i}\neq \emptyset$. Node $\widecheck{T}_i$ exists because $\F_i$ is a $\widecheck{\mathbf{S}}$-rooted c-forest. 
Figure \ref{fig:van} illustrates an example of the above definitions. 
\begin{figure}
    \centering
    \includegraphics[scale=0.2]{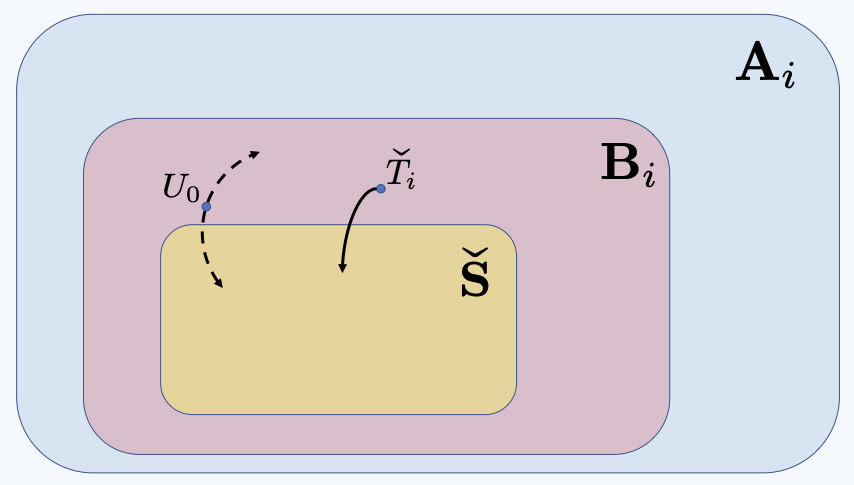}
    \caption{An illustration of the definitions $\B_i,U_0,\A_i,\widecheck{\mathbf{S}}$, and $\widecheck{T}_i$.}
    \label{fig:van}
\end{figure}

We define the domains of $X\in\V\cup \U$ as follows.
Note that $\V = \widecheck{\mathbf{S}} \cup \T$ and $\U = \U^{\widecheck{\mathbf{S}}}\cup \U^{\T}\cup \{ U_0\}$.
\begin{align*}
    &\dom{X}{} := [0:\kappa], \quad \forall X \in \widecheck{\mathbf{S}},\\
    &\dom{X}{} := [0:\kappa], \quad \forall X \in \U^{\widecheck{\mathbf{S}}},\\
    &\dom{X}{} := \{0, 1\}^{\alpha(T)}, \quad \forall X \in \T,\\
    &\dom{X}{} := \{0, 1\}^{\alpha(U)}, \quad \forall X \in \mathbf{U}^{\T},\\
    &\dom{X}{} := [0:\kappa]\times\{0, 1\}^{\alpha(U_0)-1}, \quad X=U_0,
\end{align*}
where $\kappa$ is an odd number greater than 4. 
Function $\alpha(X)$ is only defined for $X \in \T\cup\U^{\T}\cup \{U_0\}$ and it denotes the number of subgraphs in $\{\F_j\}_{j=0}^{\k}$ that contains $X$.
From the above definition, it is clear that $d$, the domain size of $U_0$ is equal to $(\kappa+1)2^{\alpha(U_0)-1}$. 
 %Note that each of the variables from $\mathbf{V}\cup \mathbf{U}$ we can represent as a vector of random variables.

Suppose that $X\in \T\cup\U^{\T}\cup\{U_0\}$ and $X$ belongs to $\F_{i_1}, \F_{i_2}, \cdots, \F_{i_{\alpha(X)}}$, where $i_1 < i_2 < \dots < i_{\alpha(X)}$. 
We use $(X[i_1], X[i_2], \dots, X[i_{\alpha{X}}])$ to represent $X$. Note that depending on where $X$ belongs to, its vector size is different.
%Further, we define the distributions $P(X|\Pa{X}{\G})$ for each variable $X$ in the set $(\V \cup \U) \setminus \{U_0\}$.- Where do you define $P(X|\Pa{X}{\G})$? If you define it by considering where $X$ comes from , then you have to mention it. like if $X\in \T$ then ... 
If $X\in\U^{\T}\cup\{U_0\}$, both its distribution and its domain are specified above. 
If $X\in\T$, we define the entries of its corresponding vector as
\begin{equation*}
    X[i_j] \equiv \left(\sum_{Y \in \Pa{X}{\F_{i_j}}} Y[i_j]\right) \pmod{2},
\end{equation*}
where $j\in[1:\alpha(X)]$. This specifies the distribution of $P(X|\Pa{X}{\G})$ for $X\in\T$.
What is left to specify is the domains and the distributions of variables in $\widecheck{\mathbf{S}}$. 
     
Recall that $U_0$ has one child in $\widecheck{\mathbf{S}}$ and one child in $\T$. We denote the child in $\widecheck{\mathbf{S}}$ by $S_0$.
For each $S \in \widecheck{\mathbf{S}}\setminus \{S_0\}$ and any realization of $\Pa{S}{\G'}$, 
we define $\mathbb{I}(S)$ to be one if there exists $i\in [0:k]$ such that
\begin{enumerate}
    \item $\widecheck{T}_i \in \Pa{S}{\G'}$ and $\widecheck{T}_i[i]=0$, or
    \item there exists $X \in \Pa{S}{\G'} \setminus (\U^{\widecheck{\mathbf{S}}}\cup \{\widecheck{T}_i\})$ such that $\F_i$ contains $X$ and $X[i]=1$,
\end{enumerate}
and zero, otherwise.
It is noteworthy that according to the definition of $\widecheck{T}_i$, it belongs to $\B_i\setminus\widecheck{\mathbf{S}}\subseteq\T$ which means $\widecheck{T}_i[i]$ is well-defined according to the above definition. 
Note that the above definition holds for all $S \in \widecheck{\mathbf{S}}\setminus \{S_0\}$. When $S=S_0$, we define $\mathbb{I}(S_0)$
to be one if there exists $i\in [0:k]$, such that
\begin{enumerate}
    \item $\widecheck{T}_i \in \Pa{S}{\G'}$ and $\widecheck{T}_i[i]=0$, or
    \item $i\neq0$, $\mathcal{F}_i$ contains $U_0$, and $U_0[i]=1$, or
    \item there exists $X \in \Pa{S}{\G'} \setminus (\U^{\widecheck{\mathbf{S}}}\cup \{\widecheck{T}_i,U_0\})$ such that $\F_i$ contains $X$ and $X[i]=1$.
\end{enumerate}

For each $S\in \widecheck{\mathbf{S}}$, we define $\dom{S}{}:=\{0,...,\kappa\}$ and for $s \in \dom{S}{}$,
\begin{equation} \label{eq: def P(S|Pa(S)) gid}
P(S=s \mid \Pa{S}{\G}):=\begin{cases} 
    \frac{1}{\kappa+1} & \text{ if } \mathbb{I}(S)=1\\
    1-\kappa\epsilon &  \text{ if } \mathbb{I}(S)=0  \text{ and } s \equiv M(S) \pmod{\kappa+1}, \\
    \epsilon &  \text{ if }  \mathbb{I}(S)=0  \text{ and } s\not\equiv M(S) \pmod{\kappa+1},
\end{cases}
\end{equation}
where $0<\epsilon<\frac{1}{\kappa}$ and 
\begin{equation}
M(S):=
\begin{cases} \label{eq: M(S)}
    \sum_{x\in \Pa{S}{\G'[\widecheck{\mathbf{S}}]}}x & \text{, if } S\in \widecheck{\mathbf{S}}\setminus \{S_0\}, \\
    u_0[0]+\sum_{x\in \Pa{S}{\G'[\widecheck{\mathbf{S}}]}}x & \text{, if $S=S_0$ }.
\end{cases}
\end{equation}
Note that $M(S)$ is an integer number because $\Pa{S}{\G'[\widecheck{\mathbf{S}}]}\subseteq \U^{\widecheck{\mathbf{S}}}$ and thus all terms in the summations in \eqref{eq: M(S)} belong to $[0:\kappa]$. With this, we finish defining the models and now we are ready to present some of their properties.

%Note that the models introduced in this section are identical to the ones introduced in the \cite[Theorem 1]{kivva2022revisiting}. 

Let $\mathbf{\Gamma}$ denote a subset of $\dom{U_0}{}=\{\gamma_1,...,\gamma_d\}$ with $\frac{\kappa+1}{2}$ elements that is given by
\begin{equation*}% module is redundant here
    \mathbf{\Gamma} := \Big\{(2x,0,\cdots,0)\!:\: x\in [0:\frac{\kappa-1}{2}]\Big\}.
\end{equation*}
Recall that for $\mathbf{v}\in \dom{\mathbf{V}}{}$ and $i\in[0:m]$, $\theta_i(\mathbf{v})$ and $\eta(\mathbf{v})$ are two vectors in $\mathbb{R}^d$ with $j$-th entry corresponding to $U_0=\gamma_j$. 
Suppose that $\mathbf{\Gamma}=\{\gamma_{j_1},...,\gamma_{j_{\frac{\kappa+1}{2}}}\}$. Next result shows that in the constructed models, all entries of $\theta_i(\textbf{v})$ with indices in $\{j_1,...,j_{\frac{\kappa+1}{2}}\}$ are equal.

\begin{lemma}[\cite{kivva2022revisiting}]
    \label{lemma: gid equal indices}
    For any $\mathbf{v} \in \dom{\mathbf{V}}{}$,  $i\in[0:m]$, and both models, we have
    \begin{equation*}
        \theta_{i,j_1}(\mathbf{v}) = \theta_{i,j_2}(\mathbf{v}) = \cdots= \theta_{i,j_{\frac{\kappa+1}{2}}}(\mathbf{v}).
    \end{equation*}
\end{lemma}

The next two lemmas are used to prove the existence of parameters $\epsilon$ and $\{p_j\}_{j=1}^d$ such that the constructed models $\M_1$ and $\M_2$ agree on the known distributions but disagree on the target causal effect.
\begin{lemma}[\cite{kivva2022revisiting}]
    \label{lemma: gid not equal indices}
    There exists $0<\epsilon<\frac{1}{\kappa}$ such that there exists $\mathbf{v}_0 \in \dom{\mathbf{V}}{}$ and $1\leq r <t\leq \frac{\kappa+1}{2}$ such that 
    \begin{equation*}
        \eta_{j_r}(\mathbf{v}_0) \neq \eta_{j_t}(\mathbf{v}_0).
    \end{equation*}
\end{lemma}

\begin{lemma}[\cite{kivva2022revisiting}]
    \label{lemma: lin indep formal}
    Consider a set of vectors $\{c_i\}_{i=1}^{n}$, where $c_i \in \mathbb{R}^d$. Assume $c\in \mathbb{R}^d$ is a vector that is linearly independent of $\{c_i\}_{i=1}^{n}$, then there is a vector $b\in \mathbb{R}^d$ such that
    \begin{align*}
        & \langle c_i, b \rangle = 0, \quad \forall i \in [1:n],\\
        & \langle c, b \rangle \neq 0.
    \end{align*}
\end{lemma}

\subsection{Proof of Lemma \ref{lemma: construct models subcase 1}}

\let\oldS\S
\renewcommand{\S}{\widecheck{\mathbf{S}}}
\newcommand{\s}{\widecheck{\mathbf{s}}}

Herein, we present the proof of our first lemma. But first, we need the following technical lemmas.
Assume that $\S'$ and $\S''$ are two disjoint non-empty subsets of $\S$, such that $\S = \S' \cup \S''$. 
% {\color{blue}
% what do you mean here?
% $u_0$ is changes for different $j$? or there exists a $j$ such that $u_0 = \gamma_j$ .
% }

% {
% \color{red} We make $u_0$ equal to $\gamma_j$ and after we define $\phi_j(\cdot)$. We do it because $\phi_j(\cdot)$ implicitly depends on the value of $u_0$, i.e. . 
% }
Let $\S^{\dagger}:=\mathbf{V} \setminus \S'$ and $\s^{\dagger} \in \dom{\S^{\dagger}}{}$. For $u_0=\gamma_j$, where $j \in [1:d]$, we define
 \begin{equation} \label{eq: phi subcase 1}
    \phi_j(\s^{\dagger}) := \sum_{\S'} \sum_{\U'\setminus\{U_0\}} \prod_{X\in \S} P(x | \Pa{X}{\G}) \prod_{U \in \U'\setminus \{U_0\}}P(u).
\end{equation}
Note that $U_0$ may appear as a parent of some observed variables in the above equation. 
Recall that $\mathbf{\Gamma}=\{\gamma_{j_1},...,\gamma_{j_{\frac{\kappa+1}{2}}}\}$. 

%First, we prove the following lemma.
\begin{lemma} \label{lemma: equal indices summation subcase 1}
    For any $\s^{\dagger} \in \dom{\S^{\dagger}}{}$ and both models,  we have
    \begin{equation*}
        \phi_{j_1}(\s^{\dagger}) = \phi_{j_2}(\s^{\dagger})=\dots = \phi_{j_{\frac{k+1}{2}}}(\s^{\dagger}).
    \end{equation*}
\end{lemma}
\begin{proof}
    We fix a realization $\s^{\dagger}$ of $\S^{\dagger}$.
    Suppose that $l_1$ and $l_2$ are two integers, such that
    \begin{equation*}
    \begin{split}
        & \gamma_{l_1} = (2x, 0, \dots, 0),\\
        & \gamma_{l_2} = (2x+2 \pmod{\kappa+1}, 0, \dots, 0),
    \end{split}
    \end{equation*}
    and $x$ is an integer in $[0 : \frac{\kappa-1}{2}]$.
   % Furthermore,  we show that $\phi_{ l_1}(\s'')=\phi_{l_2}(\s'')$.
    Recall that 
    \begin{equation*}
        \phi_j(\s^{\dagger}) := \sum_{\S'} \sum_{\U\setminus\{U_0\}} \prod_{X\in \S} P(x | \Pa{X}{\G}) \prod_{U \in \U\setminus \{U_0\}}P(u).
    \end{equation*}
    We consider two cases:
    \begin{enumerate}
        \item Suppose that there exists a variable $S \in \S$ such that $\mathbb{I}(S)=1$. Then, there is a sequence of variables $U_0, \hat{S}_1, \hat{U}_1, \hat{S}_2, \hat{U}_2, \dots, \hat{U}_l, S$, such that $U_0$ is a parent of $\hat{S}_1$, $S\in \S$ is a children of $\hat{U}_l \in \U^{\S}$ and $\hat{U}_j \in \U^{\S}$ is a parent of variables $\hat{S}_j$ and  $\hat{S}_{j+1}$ for $j \in [1:l-1]$. 
    Let $\hat{\mathbf{U}}:=\{\hat{U}_1, \dots, \hat{U}_l\}$. 
    For a given realization $\mathbf{u}_{1}$ of $\U^{\S}$,  we define $\mathbf{u}_{2} \in \dom{\U^{\S}}{}$ by
    \begin{equation}
        \begin{split}
            & \mathbf{u}_{2}[\hat{U}_j] := \mathbf{u}_{1}[\hat{U}_j] + 2(-1)^{j} \pmod{\kappa+1}, \quad j\in[1:l],\\
            & \mathbf{u}_{2}[U] := \mathbf{u}_{1}[U], \quad \forall U\in \U^{\S} \setminus \hat{\mathbf{U}}.
        \end{split}
    \end{equation}
    This implies
    \begin{equation*}
    P(\hat{s}|\Pa{\hat{S}}{\G})\Big|_{(\U^{\mathbf{S}},U_0)=(\textbf{u}_1,\gamma_{l_1})}=P(\hat{s}|\Pa{\hat{S}}{\G})\Big|_{(\U^{\mathbf{S}},U_0)=(\textbf{u}_2,\gamma_{l_2})},
    \end{equation*}
    for any $\hat{S} \in \S$ and consequently,  $\phi_{l_1}(\s^{\dagger})=\phi_{l_2}(\s^{\dagger})$.

    \item Suppose that there is no variable in $\S$ with $\mathbb{I}(\cdot)=1$. 
    Denote by $S$ a node in $\S'$ with the shortest path to the node $U_0$ by bidirected edges. Suppose $\hat{s}$ is a realization of $S$ and the shortest path is  $U_0, \hat{S}_1, \hat{U}_1, \hat{S}_2, \hat{U}_2, \dots, \hat{U}_l, S$, so that $U_0$ is a parent of $\hat{S}_1$, $S$ is a child of $\hat{U}_l \in \U^{\S}$ and $\hat{U}_j \in \U^{\S}$ is a parent of variables $\hat{S}_j$ and  $\hat{S}_{j+1}$ for $j \in [1:l-1]$. 
    Let $\hat{\mathbf{U}}:=\{\hat{U}_1, \dots, \hat{U}_l\}$. 
    For a given realization $\mathbf{u}_{1}$ of $\U^{S}$,  we define $\mathbf{u}_{2} \in \dom{\U^{\S}}{}$ by
    \begin{equation}
        \begin{split}
            & \mathbf{u}_{2}[\hat{U}_j] := \mathbf{u}_{1}[\hat{U}_j] + 2(-1)^{j} \pmod{\kappa+1}, \quad j\in[1:l],\\
            & \mathbf{u}_{2}[U] := \mathbf{u}_{1}[U], \quad \forall U\in \U^{\S} \setminus \hat{\mathbf{U}},
        \end{split}
    \end{equation}
    For a given realization $\s_1$ of $\S'$, we define $\s_2 \in \dom{\S'}{}$   as follows
    \begin{equation}
        \begin{split}
            & \s_2[S''] := \s_1[S''], \quad \forall  S''\in\S'\setminus\{S\},\\
            &  \s_2[S] := \s_1[S]+2(-1)^l \pmod{\kappa+1},
        \end{split}
    \end{equation}
    Note that with the above modifications for any $\widetilde{S} \in \S$, we get
    \begin{equation*}
        \widetilde{s}_1 - M(\widetilde{S}) \equiv \widetilde{s}_2 - M(\widetilde{S}) \pmod{\kappa+1},
    \end{equation*}
    where $\widetilde{s}_1$ is a realization of $\widetilde{S}\Big|_{(\S^{\dagger},\S')=(\s^{\dagger},\s_1)}$, $\widetilde{s}_2$ is a realization of $\widetilde{S}\Big|_{(\S^{\dagger}, \S')=(\s^{\dagger}, \s_2)}$, and $M(\cdot)$ is given by Equation (\ref{eq: M(S)}).
    Therefore:
    \begin{equation*}
    P(\widetilde{s}|\Pa{\widetilde{S}}{\G})\Big|_{(\U^{\mathbf{S}},U_0, S)=(\textbf{u}_1,\gamma_{l_1}, s_1)}=P(\widetilde{s}|\Pa{\widetilde{S}}{\G})\Big|_{(\U^{\mathbf{S}},U_0, S)=(\textbf{u}_2,\gamma_{l_2}, s_2)},
    \end{equation*}
    for any $\widetilde{S} \in \S$ and thus $\phi_{l_1}(\s^{\dagger})=\phi_{l_2}(\s^{\dagger})$.
    \end{enumerate}
    
    To summarize, we proved that $\phi_{l_1}(\s^{\dagger})=\phi_{l_2}(\s^{\dagger})$. By varying $x$ within $[0 : \frac{\kappa-1}{2}]$ in the definition of $\gamma_{l_1}$ and $\gamma_{l_2}$, we  conclude the lemma.
\end{proof}

%\paragraph{Proof of Lemma 1}

In order to have consistent notations in the appendix, we restate Lemma \ref{lemma: construct models subcase 1} using $\widecheck{\mathbf{S}}, \S', \S''$ instead of $\mathbf{L}, \mathbf{L}', \mathbf{L}''$ respectively.
\begin{customlem}{\ref{lemma: construct models subcase 1}}
    Suppose $\S\subseteq \V$ is a single c-component, such that $\S = \S'\cup\S''$ for some disjoint sets $\S'$ and $\S''$. 
    $Q[\S'|\S'']$ is c-gID from $(\mathbb{A}, \G)$ if and only if $Q[\S'\cup \S'']$ is gID from $(\mathbb{A}, \G)$.
\end{customlem}

\begin{proof}
\hfill\\
    \textbf{Sufficiency:} We use Assume that $Q[\S'\cup \S'']$ is gID from $(\mathbb{A}, \G)$, then $Q[\S'|\S'']$ is c-gID from $(\mathbb{A}, \G)$.  This is an immediate result of applying Equation \eqref{eq: conditional Q}, i.e., 
    \begin{equation*}
        Q[\S'|\S''](\mathbf{v}) = \frac{Q[\S](\mathbf{v})}{\sum_{\S'} Q[\S](\mathbf{v})}
    \end{equation*}

    \textbf{Necessity:}
    We prove this by contradiction. Assume that $Q[\S'\cup \S'']$ is not gID from $(\mathbb{A}, \G)$. We will show  that $Q[\S'|\S'']$ is not c-gID from $(\mathbb{A}, \G)$. 
    To this end, we will construct two models $\M_1$ and $\M_2$ such that for each $i \in [0:m]$ and any $\mathbf{v}\in \mathbf{V}$:
    \begin{align}
        \label{eq: equal known dist subcase 1}
        Q^{\M_1}[\A_i](\mathbf{v}) &= Q^{\M_2}[\A_i](\mathbf{v}),\\
        \label{eq: equal denom num subcase 1}
        \sum_{\S'}Q^{\M_1}[\S](\mathbf{v}') &= \sum_{\S'}Q^{\M_2}[\S](\mathbf{v}'),
    \end{align}
    but there exists $\mathbf{v}_0 \in \dom{\mathbf{V}}{}$ such that:
    \begin{equation}\label{eq: not equal num subcase 1}
        Q^{\M_1}[\S](\mathbf{v}_0) \neq Q^{\M_2}[\S](\mathbf{v}_0).
    \end{equation}
    Equations (\ref{eq: equal denom num subcase 1})-(\ref{eq: not equal num subcase 1}) yield
    \begin{equation*}
         Q[\S'|\S'']^{\M_1}(\mathbf{v}_0) \neq Q[\S'|\S'']^{\M_2}(\mathbf{v}_0).
    \end{equation*}
    This means that $Q[\S'|\S'']$ is not c-gID from $(\mathbb{A}, \G)$.

    We consider two cases.

    \paragraph{First case:} 
    Suppose that there exists $i \in [0, m]$, such that $\widecheck{\mathbf{S}} \subset \mathbf{A}_i$.
    For this, we consider the models constructed in the section \ref{sec: baseline}:
    \begin{align}
        & \sum_{\S'} Q[\S]^{\M_1}(\mathbf{v}) = \sum_{j=1}^d \frac{1}{d}\phi_{j}(\mathbf{v}[\S^{\dagger}]), \\
        & \sum_{\S'} Q[\S]^{\M_2}(\mathbf{v}) = \sum_{j=1}^d p_j\phi_{j}(\mathbf{v}[\S^{\dagger}]).
    \end{align}
    and according to the Equations (\ref{eq: theta and eta}) and \eqref{eq: phi subcase 1}, we have
    \begin{align}
        & Q^{\M_2}[\A_i](\mathbf{v}) - Q^{\M_1}[\A_i](\mathbf{v}) = \sum_{j=1}^d (p_j - \frac{1}{d}) \theta_{i,j}(\mathbf{v})
        \\
        & \hspace{-0.5cm}\sum_{\S'} Q[\S]^{\M_2}(\mathbf{v}) - \sum_{\S'} Q[\S]^{\M_1}(\mathbf{v}) = \sum_{j=1}^d (p_j - \frac{1}{d}) \phi_{j}(\mathbf{v}[\S^{\dagger}])
        \\
        & Q^{\M_2}[\S](\mathbf{v}_0) -  Q^{\M_1}[\S](\mathbf{v}_0) = \sum_{j=1}^d (p_j - \frac{1}{d}) \eta_{j}(\mathbf{v}_0)
        \\
        & \sum_{j=1}^d p_j - 1 = \sum_{j=1}^d (p_j - \frac{1}{d})
    \end{align}

     Therefore, it suffices to solve a system of linear equations over parameters $\{p_j\}_{j=1}^d$ and show that it admits a solution. 
    \begin{align}
        \label{eq: linear system start}
        & \sum_{j=1}^d (p_j - \frac{1}{d}) \theta_{i,j}(\mathbf{v}) = 0, \hspace{0.2cm}\forall \mathbf{v} \in \dom{\mathbf{V}}{}, i\in [0,m],
        \\
        & \sum_{j=1}^d (p_j - \frac{1}{d}) \phi_{j}(\s^{\dagger}) = 0, \hspace{0.2cm}\forall \s^{\dagger} \in \dom{\S^{\dagger}}{}, i\in [0,m],
        \\
        & \sum_{j=1}^d (p_j - \frac{1}{d}) \eta_{j}(\mathbf{v}_0) \neq 0, \hspace{0.2cm} \exists \mathbf{v}_0 \in \dom{\V}{},
        \\
        & (p_j - \frac{1}{d}) = 0,
        \\
        \label{eq: linear system end}
        & 0<p_j<1, \hspace{0.2cm} \forall j \in [1:d].
    \end{align}

    However, the system of linear equations (\ref{eq: linear system start})-(\ref{eq: linear system end}) admits a solution with respect to $\{p_j\}_{j=1}^d$ if and only if the following system of equations has a solution with respect to parameters $\{\beta_j\}_{j=1}^{d}$:
    \begin{align}
        \label{eq: homogenous linear system start}
        &\sum_{j=1}^d \beta_j \theta_{i,j}(\mathbf{v}) =0, \hspace{0.2cm}\forall \mathbf{v} \in \dom{\mathbf{V}}{}, i \in [0:m]\\
        &\sum_{j=1}^d \beta_j \phi_{j}(\s^{\dagger}) =0, \hspace{0.2cm}\forall \s^{\dagger} \in \dom{\S^{\dagger}}{}, i \in [0:m]\\
        \label{eq: homogenous linear system ineq}
        &\sum_{j=1}^d \beta_j \eta_j(\mathbf{v}_0) \neq 0, \hspace{0.2cm} \exists \mathbf{v}_0 \in \dom{\mathbf{V}}{}\\
        \label{eq: homogenous linear system end}
        & \sum_{j=1}^d \beta_j=0.
    \end{align}
    Clearly, if $\{\beta_j^*\}$ is a solution for system (\ref{eq: homogenous linear system start})-(\ref{eq: homogenous linear system end}), then 
    \begin{equation}
        p_j^*:=\frac{1}{d}+\frac{\beta_j^*}{2hd},
    \end{equation}
    is a solution for \eqref{eq: linear system start}-\eqref{eq: linear system end}, where $h := \underset{j \in [1:d]}{max} |\beta^*_j|$.

    According to Lemma \ref{lemma: gid equal indices} and Lemma \ref{lemma: equal indices summation subcase 1}, for any $i\in [0:m]$, $\mathbf{v} \in \dom{\mathbf{V}}{}$ and  $\s^{\dagger} \in \dom{\S^{\dagger}}{}$, we have
    \begin{align*}
        & \theta_{i,j_1}(\mathbf{v}) = \theta_{i,j_2}(\mathbf{v}) = \cdots= \theta_{i,j_{\frac{\kappa+1}{2}}}(\mathbf{v}), \\
        & \phi_{j_1}(\s^{\dagger}) = \phi_{j_2}(\s^{\dagger})=\dots = \phi_{j_{\frac{k+1}{2}}}(\s^{\dagger}),
    \end{align*}
    and by Lemma \ref{lemma: gid not equal indices}, we know that there exists $\mathbf{v}_0 \in \dom{\mathbf{V}}{}$ and $1\leq r <t\leq \frac{\kappa+1}{2}$ such that 
    \begin{equation*}
        \eta_{j_r}(\mathbf{v}_0) \neq \eta_{j_t}(\mathbf{v}_0).
    \end{equation*}
    The latter means that the vector $\big(\eta_{1}(\mathbf{v}_0), \eta_{2}(\mathbf{v}_0), \dots, \eta_{d}(\mathbf{v}_0)\big)$ is linearly independent from vectors:
    \begin{align}
        & \big(1, 1, \dots, 1\big),
        \\
        & \big(\theta_{i, 1}(\mathbf{v}), \theta_{i, 2}(\mathbf{v}), \dots, \theta_{i, d}(\mathbf{v})\big), \hspace{0.2cm} \forall \mathbf{v}\in \dom{\mathbf{V}}{},\quad \forall i\in[0:m],\\
        & \big(\phi_{1}(\s^{\dagger}), \phi_{2}(\s^{\dagger}), \dots, \phi_{d}(\s^{\dagger})\big), \hspace{0.2cm} \forall \s^{\dagger}\in \dom{\S^{\dagger}}{}.
        \\
    \end{align}
    Combining the last result with Lemma \ref{lemma: lin indep formal} imply the existence of a solution  $\{\beta_j^*\}$ and  subsequently the existence of two models $\M_1$ and $\M_2$ that satisfy equation \eqref{eq: equal known dist subcase 1}, \eqref{eq: equal denom num subcase 1} and \eqref{eq: not equal num subcase 1}. %This concludes the proof of the Necessity part.
    
    \paragraph{Second case:}
    Suppose that there is no $i \in [0, m]$, such that $\S \subset \mathbf{A}_i$. Suppose $S^*\in \S$ and denote by $\G^*$ the graph obtained from graph $\G$ through the following procedure:
    \begin{enumerate}
        \item Add nodes $T_0^*$ and $U_0^*$ to graph $\G$.
        \item Draw a direct edge from $T_0^*$ to $S^*$.
        \item Draw direct edges from $U_0^*$ to $S^*$ and $T_0^*$.
    \end{enumerate}
    We define $\mathbf{A}_{m+1} := \S \cup \{T_0^*\}$ and $\mathbb{A}^* := \mathbb{A}\cup \{\A_{m+1}\}$. To summarize, we have
    \begin{itemize}
        \item $\mathbf{V}$ is a set of all observed variables in graph $\G$;
        \item $\mathbf{U}$ is a set of all unobserved variables in graph $\G$;
        \item $\mathbf{V}^* = \mathbf{V}\cup \{T_0^*\}$;
        \item $\mathbf{U}^* = \mathbf{U}\cup \{U_0^*\}$.
    \end{itemize}
    Note that $Q[\S]$ is not identifiable in $\G^*[\A_{m+1}]$ and therefore $Q[\S]$ remains not gID from $(\mathbb{A}^*, G^*)$. Since $\S \in \A_{m+1}$, according to the \textbf{First case}, we can construct models $\M_1^*$ and $\M_2^*$ for the graph $\G^*$ and set $\mathbb{A}^*$. These two models satisfy the following properties
    \begin{itemize}
        \item $\dom{U_0^*}{} = [0:\kappa]$ and $d=\kappa+1$.
        \item $\dom{T_0^*}{}$ = \{0, 1\}.
    \end{itemize}
    For the graph $\G^*$, we define 
    \begin{align*}
        & \theta_{i, j}(\mathbf{v}, T_0^*=t_0) := \sum_{\mathbf{U}} \prod_{X \in \A_i} P(x \mid \Pa{X}{\G^*}) \prod_{U\in \mathbf{U}} P(u), \quad i\in [0:m], \; j\in [1:d], \\
        & \theta_{m+1, j}(\mathbf{v}, T_0^*=t_0) := \sum_{\mathbf{U}} P(T_0^*=t_0)\prod_{X \in \A_i} P(x \mid \Pa{X}{\G^*}) \prod_{U\in \mathbf{U}} P(u), \quad j\in [1:d], \\
        & \phi_j(\s^{\dagger}) := \sum_{\S'} \sum_{\U} \prod_{X\in \S} P(x | \Pa{X}{\G^*}) \prod_{U \in \U}P(u), \quad j \in [1:d] \\
        & \eta_{j}(\mathbf{v}, T_0^*=t_0) := \sum_{\mathbf{U}} \prod_{X \in \widecheck{\mathbf{S}}} P(x \mid \Pa{X}{\G^*}) \prod_{U\in \mathbf{U}} P(u).
    \end{align*} 
Now, we are ready to construct two models $\M_1$ and $\M_2$ for $\G$.
\begin{itemize}
    \item  For all $S \in \S \setminus \{S^{*}\}$, we define
    \begin{equation*}
        P^{\M_i}(S|\Pa{S}{\G}) := P(S|\Pa{S}{\G^*}), \quad i\in\{1,2\}.
    \end{equation*}
    \item For $S=S^*$, we define
    \begin{align*}
        & P^{\M_1}(S|\Pa{S}{\G}) := P(S|\Pa{S}{\G}, T_0^*=1, U_0^*=0), \\
        & P^{\M_2}(S|\Pa{S}{\G}) := P(S|\Pa{S}{\G}, T_0^*=1, U_0^*=2).
    \end{align*}
\end{itemize}
Suppose that $\gamma_{r_1} = 0$ and $\gamma_{r_2}=2$, then
\begin{itemize}
    \item In model $\M_1$:
    \begin{align*}
        & Q[\A_i]^{\M_1}(\mathbf{v}) = \theta_{r_1, j}(\mathbf{v}, T_0^*=1), \quad i \in [0, m],\\
        & \sum_{\S'}Q^{\M_1}[\S](\mathbf{v}) = \phi_{r_1}(\mathbf{v}[\S^{\dagger}]),\\
        & Q^{\M_1}[\S](\mathbf{v}) = \eta_{r_1}(\mathbf{v}).
    \end{align*}
    
    \item In model $\M_2$:
    \begin{align*}
        & Q[\A_i]^{\M_2}(\mathbf{v}) = \theta_{r_2, j}(\mathbf{v}, T_0^*=1), \quad i \in [0, m],\\
        & \sum_{\S'}Q^{\M_2}[\S](\mathbf{v}) = \phi_{r_2}(\mathbf{v}[\S^{\dagger}]),\\
        & Q^{\M_2}[\S](\mathbf{v}) = \eta_{r_2}(\mathbf{v}).
    \end{align*}
\end{itemize}
According to the Lemmas \ref{lemma: gid equal indices} and \ref{lemma: equal indices summation subcase 1} for any $\mathbf{v}\in \dom{\mathbf{V}}{}$
\begin{align*}
    Q^{\M_1}[\A_i](\mathbf{v}) &= Q^{\M_2}[\A_i](\mathbf{v}),
    \\
    \sum_{\S'}Q^{\M_1}[\S](\mathbf{v}') &= \sum_{\S'}Q^{\M_2}[\S](\mathbf{v}),
\end{align*}
however, using Lemma \ref{lemma: gid not equal indices} and for $\mathbf{v}_0 = (0,...,0)$, we get
\begin{equation*}
    Q^{\M_1}[\S](\mathbf{v}_0) \neq Q^{\M_2}[\S](\mathbf{v}_0).
\end{equation*}
\end{proof}

\subsection{Proof of the properties in Section \ref{sec: second subcase} \& Section \ref{sec: third subcase}} \label{sec: properties}

Recall that in Sections \ref{sec: second subcase} and \ref{sec: third subcase}, we present two sets of properties which we prove them here. 
We only present the formal proof of the set of properties in Sections \ref{sec: second subcase} since the other set of properties in Section \ref{sec: third subcase} can be shown similarly.

\begin{enumerate}
    \item If path $p$ contains a chain $W'\rightarrow W \rightarrow W''$ or a fork $W' \leftarrow W \rightarrow W''$, then node $W$ does not belong to any of the sets $\X'$, $\Z'$ or $\Y'$.
    
    \item If path $p$ contains a collider $W'\rightarrow W \leftarrow W''$, then there is a directed path $p_W$ from $W$ to a node in $\Z'$. 
    Moreover, none of the intermediate nodes in the path $p_W$ belong to the set $\X'\cup\Z'\cup\Y'$.
    
    \item Path $p$ does not contain any node from the set $\X'$.
\end{enumerate}
\begin{proof}
\hfill\\
1. The first property is obvious since path $p$ is not blocked by the set $\X'\cup (\Z' \setminus \{Z'\})\cup(\Z' \setminus \{Y'\})$. 

2. Suppose $W$ is a collider as defined and let assume that $R$ is the closest descendant of the variable $W$ that unblocks path $p$. Note that $R \notin \X'$ since it unblocks $p$ in the graph $\G_{\overline{\X'},\underline{\{Z'\}}}$, i.e. no incoming edges in $\X'$. 

All variables except $R$ in the shortest directed path from $W$ to $R$ do not belong to the set $\X' \cup \Y' \cup \Z'$. 
Assume that $R \in \Y'$ and $p'$ is a path obtained by combining two paths: one from $Z'$ to $W$ in $p$ and the other one from $W$ to $R$ (defined above). 
It is easy to see that $p'$ is also unblocked, but it contains less number of colliders than $p$. 
This is impossible according to the definition of the path $p$. 
Thus, $R$ must be in the set $\Z'$. This concludes the proof of the second property.

3. We prove this by contradiction. 
Suppose that there is a variable $R\in \X'$ on the path $p$. Since $p$ is unblocked, then $X$ is a collider or a descendant of a collider. This is impossible due to property 2.
% Recall that in Section \ref{sec: third subcase} we defined an unblocked path $\widetilde{p}$ from $Z'$ to some node $Y'\in\Y'$ given $\X', \Z'\setminus\{Z'\}$. Path $\widetilde{p}$ satisfies the following properties:
\end{proof}

\subsection{Proof of Lemma \ref{lemma: construct models subcase 2}} \label{sec: supportive materials subcase 2}

Recall that $\textbf{S}=\Anc{\Y'\cup \Z'}{\G[\V \setminus \X']}$ and it is assumed that is not gID from $(\mathbb{A}, \G)$.  $\textbf{S}$ consists of $\textbf{S}_1,...,\textbf{S}_n$ as its single c-components where $\textbf{S}_1$ is not gID.
Let $\S=\mathbf{S}_1$. 
Clearly, we can add $\{\mathbf{S}_i\}_{i=2}^{n}$ to the known distributions and $\S$ remains not gID, i.e., $\S$ is not gID from $(\mathbb{A}', \G)$, where $\mathbb{A}' := \mathbb{A}\cup\{\mathbf{S}_i\}_{i=2}^{n}$. 
For simplicity, we denote $\mathbb{A}' = \{\mathbf{A}'\}_{i=0}^{m'}$.
Hence, using the method in Section \ref{sec: baseline}, we can construct two models $\M_1$ and $\M_2$ that are the same over the known distributions and different over $Q[\textbf{S}_1]$. These models  disagree on the distribution $Q[\mathbf{S}]$ as well, because $Q[\mathbf{S}] = \prod_{i=1}^{n}Q[\mathbf{S}_i]$. Below, we use these two models to introduce two new models to prove Lemma \ref{lemma: construct models subcase 2}.

\renewcommand{\R}{\mathbf{R}}
\subsubsection{New models for Lemma \ref{lemma: construct models subcase 2}} \label{sec: appendix new models subcase 2}

Recall that $\mathcal{P}$ is a collection of paths $\{p\}\cup\{ p_W|W \in \mathbf{F}\}$, where $\mathbf{F}$ is a set of all colliders on the path $p$. Moreover, $\D$ is a set of all observed nodes on the paths in $\mathcal{P}$ excluding the ones in $\Z'$. Figure \ref{fig:va2} demonstrates some variables used in this proof and their relations for clarity.

\begin{figure}
    \centering
    \includegraphics[scale=0.2]{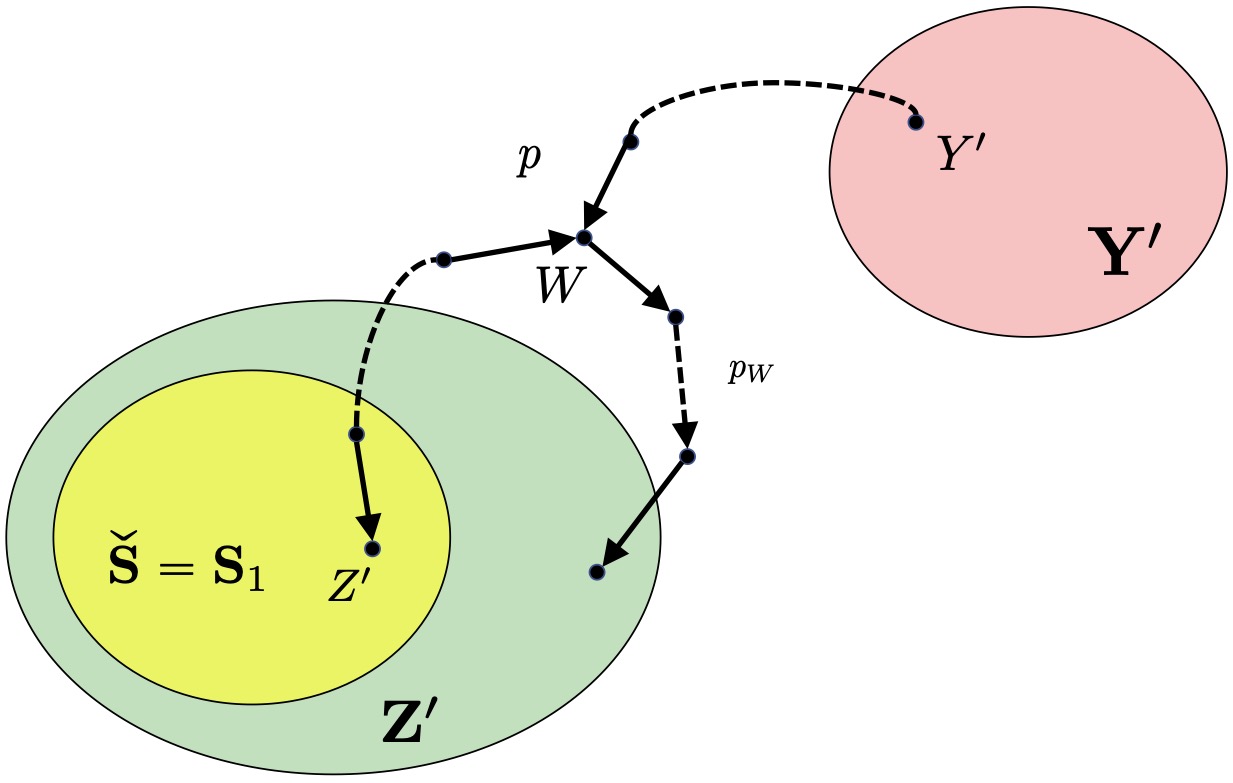}
    \caption{An illustration of the path $p$, collider $W$ and its corresponding path $p_W$.}
    \label{fig:va2}
\end{figure}

Herein, we define new models $\M_1'$ and $\M_2'$ based on the models $\M_1$ and $\M_2$. Let $\D_{\mathcal{P}}$ be the set of all variables (observed and unobserved) on the paths in $\mathcal{P}$. We say that a variable $D$ is a \textbf{starting node} of path $\hat{p}\in\mathcal{P}$ if
\begin{itemize}
    \item $D=Z'$ and $\hat{p}=p$ or
    \item $D \in \mathbf{F}$, i.e., it is a collider on path $p$ and $\hat{p}=p_D$.
\end{itemize}
Note that $D$ can be a starting node of only  one path. According to the definition of a starting node, if $D$ is a starting node for some path then either $D$ is a collider on the path $p$ or $D$ is $Z'$.  

For $R \in \mathbf{V}\cup \mathbf{U}$, let $\alpha_p(R)$ be the number of paths in $\mathcal{P}$ that contains $R$. Furthermore, we use $\dom{R}{}^{'}$ and $\dom{R}{}$ to denote its domain in  $\M'_1$ or $\M'_2$ (variables in different models have the same domains) and in $\M_1$ or $\M_2$, respectively.  
We define $\dom{R}{}'$ as follows:
\begin{itemize}
    \item If $R$ is a starting node for a path in $\mathcal{P}$:
    \begin{align*}
        & \dom{R}{}^{'} := \dom{R}{}^{}\times[0:\kappa]^{\alpha_{p}(R)-1}.
    \end{align*}
    % \item If $R \in (\mathbf{V}\cup \mathbf{U})\setminus\S$, then:
    % \begin{align*}
    %     & \dom{R}{}^{'} := \dom{R}{}^{}\times[0:\kappa]^{\alpha_{p}(R)}.
    % \end{align*}
    % \item If $R \in \S$ and $R$ belongs to one of the paths in $\mathcal{P}$, then:
    % \begin{align*}
    %     & \dom{R}{}^{'} := \dom{R}{}^{}\times[0:\kappa]^{\alpha_{p}(R)-1}.
    % \end{align*}
    \item If $R$ is not a starting node for any of the paths in $\mathcal{P}$, then:
    \begin{align*}
        & \dom{R}{}^{'} := \dom{R}{}^{}\times[0:\kappa]^{\alpha_{p}(R)}.
    \end{align*}
\end{itemize}
Consequently, if $R$ does not belong to any of the paths in $\mathcal{P}$, then $\dom{R}{}^{'} = \dom{R}{}^{}.$

Consider $R \in \mathbf{V}\cup\mathbf{U}$.
According to the domains definitions above, $R$ is a vector that is a concatenation of the vector coming from $\dom{R}{}^{}$ in model $\M_1$ (or $\M_2$) and some additional coordinates.
These additional coordinates are defined based on $\alpha_p(R)$. 
More precisely, if $R$ is not a starting node of a path $\hat{p}\in\mathcal{P}$, then there is a coordinate assigned to this path, denoted by $R[\hat{p}]$, otherwise, if $R$ is a starting node of  $\hat{p}\in\mathcal{P}$, then there is no coordinate assigned this path.

%to denote a respective entry in $R$ i.e. one of the entries of the vector $[0:\kappa]^{\alpha_{p}(R)}$ for $\dom{R}{}^{'} := \dom{R}{}^{}\times[0:\kappa]^{\alpha_{p}(R)}$ and one of the entries of the vector $[0:\kappa]^{\alpha_{p}(R)-1}$ for $\dom{R}{}^{'} := \dom{R}{}^{}\times[0:\kappa]^{\alpha_{p}(R)-1}$.

Let $\mathbf{O}\subseteq\mathbf{V}\cup \mathbf{U}$. 
For any realization $\mathbf{o}\in\dom{\mathbf{O}}{}^{'}$ of $\mathbf{O}$, we denote by $\mathbf{o}^\M\in\dom{\mathbf{O}}{}$, a realization of $\mathbf{O}$  that is consistent with $\mathbf{o}$. With a slight abuse of notation, we use $\textbf{O}$ and $\textbf{O}^\M$ to denote realizations of $\textbf{O}$ in models $\M'_i$ and $\M_i$, respectively. $\textbf{O}^\M$ means realizations in $\dom{\mathbf{O}}{}$ from model $\M_i$ that are consistent with realizations in $\dom{\mathbf{O}}{}^{'}$ from model $\M'_i$.

Recall that $\D_{\mathcal{P}}$ is a set of all variables on the paths in $\mathcal{P}$. Let $D\in \D_{\mathcal{P}}$. We denote by $\mathcal{P}_D$, the set of all paths $\hat{p}$, such that $\hat{p}\in \mathcal{P}$, $D$ belongs to path $\hat{p}$, and $D$ is not a starting node of path $\hat{p}$.
We are now ready to define the probabilities of $P^{\M'_i}(D|\Pa{D}{\G})$ for any $D \in \mathbf{V}\cup \mathbf{U}$ and $i \in \{1,2\}$.
\begin{itemize}
    \item If $D$ does not belong to the set $\D_{\mathcal{P}}$, we define
    \begin{align*}
        P^{\M'_i}(D|\Pa{D}{\G}) := P^{\M_i}(D^\M|\Pa{D}{\G}).
    \end{align*}
    \item If $D$ belongs to the set $\D_{\mathcal{P}}\setminus\{Z'\}$, we define
    \begin{equation*}
        P^{\M'_i}(D|\Pa{D}{\G}) := P^{\M_i}(D^{\M}|\Pa{D}{\G})\prod_{\hat{p}\in \mathcal{P}_D}f_{\hat{p}}(D[\hat{p}]|\Pa{D}{\hat{p}}),
    \end{equation*}
    where $\Pa{D}{\hat{p}}$ denotes the parents of $D$ on path $\hat{p}$ and $f_{\hat{p}}(D|\Pa{D}{\hat{p}})$ is given below.

    \paragraph{Definition of function} $f_{\hat{p}}(D[\hat{p}]|\Pa{D}{\hat{p}})$:
    \begin{itemize}
        \item When there exists a variable $W\in \mathbf{F}$, such that $\hat{p}=p_W$ and $D$ is a child of $W$ on path $p_W$ (i.e., $W\in\Pa{D}{p_W}$), then we define
        \begin{equation*}
        f_{p_W}(D[p_W]|\Pa{D}{p_W}) :=
        \begin{cases} 
            1-\kappa\epsilon &  \text{ if }  D[p_W]\equiv W[p] \pmod{\kappa+1},\\
            \epsilon &  \text{ if }  D[p_W]\not\equiv W[p] \pmod{\kappa+1}.
        \end{cases}
        \end{equation*}
        
        \item When $\Pa{D}{\hat{p}}=\emptyset$, 
        \begin{equation*}
            f_{\hat{p}}(D[\hat{p}]) := \frac{1}{\kappa+1}.
        \end{equation*}
        
        \item Otherwise,
        \begin{equation} \label{eq: f for not starting node}
        f_{\hat{p}}(D[\hat{p}]|\Pa{D}{\hat{p}}) :=
        \begin{cases} 
            1-\kappa\epsilon &  \text{ if }  D[\hat{p}]\equiv \sum_{D' \in \Pa{D}{\hat{p}}}D'[\hat{p}] \pmod{\kappa+1}\\
            \epsilon &  \text{ if }  D[\hat{p}]\not\equiv \sum_{D' \in \Pa{D}{\hat{p}}}D'[\hat{p}] \pmod{\kappa+1},
        \end{cases}
        \end{equation}
        Note that $P^{\M'_i}(D|\Pa{D}{\G})$ is a probability distribution since for different paths $\hat{p}_1$ and $\hat{p}_2$, $D[\hat{p}_1]$ and $D[\hat{p}_2]$ are different and also 
        \begin{equation*}
            \sum_{D[\hat{p}]\in \dom{D[\hat{p}]}{}} f_{\hat{p}}(D[\hat{p}]|\Pa{D}{\hat{p}}) = 1
        \end{equation*}
    \end{itemize}
    
    % \textbf{Distribution $P'(D|\Pa{D}{\G})$:}
    % \begin{itemize}
    %     \item When $D$ does not belong to the set $\S$,
    %     \begin{equation*}
    %         P'(D|\Pa{D}{\G}) := P^{\M_i}(D^\M|\Pa{D}{\G})
    %     \end{equation*}
    %     \item When $D\in \S$, analogous to the models $\M_1$ and $\M_2$, we define
    %     \begin{equation*}
    %     P'(D \mid \Pa{D}{\G})=:
    %     \begin{cases} 
    %         \frac{1}{\kappa+1} & \text{ if } \mathbb{I}(D)=1\\
    %         1-\kappa\epsilon &  \text{ if } \mathbb{I}(D)=0  \text{ and } D[D(\mathcal{P})] \equiv M'(D) \pmod{\kappa+1}, \\
    %         \epsilon &  \text{ if }  \mathbb{I}(D)=0  \text{ and } D[D(\mathcal{P})] \not\equiv M'(D) \pmod{\kappa+1},
    %     \end{cases}
    %     \end{equation*}
    %     where
    %     \begin{equation}
    %     M'(D):=
    %     \begin{cases} \label{eq: M'(D)}
    %         \sum_{\hat{D}\in\Pa{D}{D(\mathcal{P})}}\hat{D}[D(\mathcal{P})] + \sum_{\hat{d}'\in \Pa{D}{\G'[\widecheck{\mathbf{S}}]}}\hat{d}' & \text{, if } D \in \widecheck{\mathbf{S}}\setminus \{S_0\}, \\
    %         \sum_{\hat{D}\in\Pa{D}{D(\mathcal{P})}}\hat{D}[D(\mathcal{P})] + u_0[0]+\sum_{\hat{d}'\in \Pa{D}{\G'[\widecheck{\mathbf{S}}]}}\hat{d}' & \text{, if $D=S_0$ }.
    %     \end{cases}
    %     \end{equation}
        
    % \end{itemize}

    \item If $D=Z'$ and $W$ is a parent of $Z'$ in path $p$. Note that such $W$ exists because $p$ is an unblocked backdoor path in graph $\G_{\overline{\X'},\underline{\{Z'\}}}$. Recall that $Z'$ is a variable from the set $\S$. In this case, we define
    \begin{equation}
        P^{\M'_i}(Z'|\Pa{Z'}{\G}) := P'(Z'^{\M}|\Pa{Z'}{\G})\prod_{\hat{p}\in \mathcal{P}_{Z'}}f_{\hat{p}}(Z'[\hat{p}]|\Pa{Z'}{\hat{p}}),
    \end{equation}
    where $P'(\cdot|\cdot)$ is given by
    \begin{equation*}
    P'(Z'^{\M}=z' \mid \Pa{Z'}{\G}):=
    \begin{cases} 
        \frac{1}{\kappa+1} & \text{ if } \mathbb{I}(Z')=1,\\
        1-\kappa\epsilon &  \text{ if } \mathbb{I}(Z')=0  \text{ and } z' \equiv M'(Z') \pmod{\kappa+1}, \\
        \epsilon &  \text{ if }  \mathbb{I}(Z')=0  \text{ and } z'\not\equiv M'(Z') \pmod{\kappa+1},
    \end{cases}
    \end{equation*}
    and $M'(\cdot)$ is defined similar to \eqref{eq: M(S)} and is given by
    \begin{equation}
    M'(Z'):=
    \begin{cases} \label{eq: M'(S)}
        W[p] + \sum_{x\in \Pa{Z'}{\G'[\widecheck{\mathbf{S}}]}}x^{\M} & \text{, if } Z'\in \widecheck{\mathbf{S}}\setminus \{S_0\}, \\
        W[p] + u_0^{\M}[0]+\sum_{x\in \Pa{Z'}{\G'[\widecheck{\mathbf{S}}]}}x^{\M} & \text{, if $Z'=S_0$ }.
    \end{cases}
    \end{equation}
\end{itemize}

Note that for any $W \in (\mathbf{V} \cup \mathbf{U})\setminus \{U_0\}$, we have
$$
P^{\M'_1}(W|\Pa{W}{\G}) = P^{\M'_2}(W|\Pa{W}{\G}).
$$ 
Therefore, we will use $P^{\M'}(W|\Pa{W}{\G})$ instead of $P^{\M'_1}(W|\Pa{W}{\G})$ or $P^{\M'_2}(W|\Pa{W}{\G})$ for  $W \in (\mathbf{V} \cup \mathbf{U})\setminus \{U_0\}$. 
%Analogously we define $P^{\M}(W|\Pa{W}{\G})$.

We also have
\begin{equation}
\label{eq: prob U_0 subcase 2}
\begin{gathered}
    P^{\M'_1}(U_0) = \frac{1}{d} \prod_{\hat{p}\in \mathcal{P}_{U_0}}f_{\hat{p}}(U_0[\hat{p}]), \\
    P^{\M'_2}(U_0) = P^{\M_2}(U_0^{\M}) \prod_{\hat{p}\in \mathcal{P}_{U_0}}f_{\hat{p}}(U_0[\hat{p}]).
\end{gathered}
\end{equation}

Recall that $\mathbf{S}=\Anc{\Y', Z'}{\G[\mathbf{V}\setminus\X']}$. Let $\D':=\mathbf{S}\setminus\D$ and $\D^{\dagger} := \mathbf{V} \setminus \mathbf{D}$.
For $i \in [0:m']$, $j\in [1:d]$, $\mathbf{v} \in \dom{\mathbf{V}}{}^{'}$ and $\mathbf{d}^{\dagger} \in \dom{\D^{\dagger}}{}'$, we define $\theta_{i, j}^{'}(\mathbf{v})$, $\phi_{j}^{'}(\mathbf{d}^{\dagger})$ and $\eta_{j}^{'}(\mathbf{v})$ as follows:
\begin{align}
    \label{eq: def theta subcase 2}
    & \theta_{i, j}^{'}(\mathbf{v}) := \sum_{U_0[\mathcal{P}]}\prod_{\hat{p}\in \mathcal{P}_{U_0}}f_{\hat{p}}(U_0[\hat{p}])\sum_{\mathbf{U} \setminus \{U_0\}} \prod_{X \in \A_i'} P^{\M'}(x \mid \Pa{X}{\G}) \prod_{U\in \mathbf{U} \setminus \{U_0\}} P^{\M'}(u^{\M}),\\
    \label{eq: def phi subcase 2}
    & \phi_{j}^{'}(\mathbf{d}^{\dagger}) := \sum_{U_0[\mathcal{P}]}\prod_{\hat{p}\in \mathcal{P}_{U_0}}f_{\hat{p}}(U_0[\hat{p}])\sum_{\D}\sum_{\mathbf{U} \setminus \{U_0\}} \prod_{X \in \mathbf{S}} P^{\M'}(x \mid \Pa{X}{\G}) \prod_{U\in \mathbf{U} \setminus \{U_0\}} P^{\M'}(u^{\M}),\\
    \label{eq: def eta subcase 2}
    & \eta_{j}^{'}(\mathbf{v}) := \sum_{U_0[\mathcal{P}]}\prod_{\hat{p}\in \mathcal{P}_{U_0}}f_{\hat{p}}(U_0[\hat{p}])\sum_{\mathbf{U} \setminus \{U_0\}} \prod_{X \in \mathbf{S}} P^{\M'}(x \mid \Pa{X}{\G}) \prod_{U\in \mathbf{U} \setminus \{U_0\}} P^{\M'}(u^{\M}),
\end{align}
where $\sum_{U_0[\mathcal{P}]}$ is a summation over all realizations of the random variables $\{U_0[\hat{p}]|\ \hat{p}\in \mathcal{P}_{U_0}\}$. 

Next, we prove three lemmas similar to Lemmas \ref{lemma: gid equal indices}, \ref{lemma: gid not equal indices},  and \ref{lemma: equal indices summation subcase 1}  for the new models $\M'_1$ and $\M'_2$.

\begin{lemma}
    \label{lemma: theta equal indices subcase 2}
    For any $\mathbf{v} \in \dom{\mathbf{V}}{}^{'}$ and $i\in[0:m']$, we have
    \begin{equation*}
        \theta_{i,j_1}'(\mathbf{v}) = \theta_{i,j_2}'(\mathbf{v}) = \cdots= \theta_{i,j_{\frac{\kappa+1}{2}}}'(\mathbf{v}).
    \end{equation*}
\end{lemma}
\begin{proof}
    %According to the construction of models $\M_1'$ and $\M_2'$,
    By substituting $P^{\M'}$ from the above into Equation \eqref{eq: def theta subcase 2} and rearranging the terms, we obtain
    \begin{equation*}
    \begin{gathered}
        \theta_{i, j}^{'}(\mathbf{v}) = 
        \sum_{\mathbf{U}[\mathcal{P}]} 
        \prod_{\hat{U}[\hat{p}]\in \mathbf{U}[\mathcal{P}]} f_{\hat{p}}(\hat{U})
        \prod_{\hat{X}[\hat{p}] \in \mathbf{A}'_i[\mathcal{P}]} f_{\hat{p}}(\hat{X}[\hat{p}]|\Pa{\hat{X}}{\hat{p}})\\
        \times\Big( \sum_{\textbf{\underline{U}}^{\M}}
         P'(Z'|\Pa{Z'}{\G})
         \prod_{X \in \A'_i\setminus\{Z'\}} P^{\M}(x^{\M} \mid \Pa{X}{\G}) \prod_{U\in \mathbf{U} \setminus \{U_0\}}\!P^{\M}(u^{\M})\Big),
    \end{gathered}
    \end{equation*}
    where $\mathbf{U}[\mathcal{P}]:=\bigcup_{U\in\textbf{U}}\{U[\hat{p}]|\ \hat{p}\in \mathcal{P}_{U}\}$, $\underline{\textbf{U}}:=\mathbf{U} \setminus \{U_0\}$, and by definition $\textbf{\underline{U}}^{\M}$ is all realizations of elements in set $\textbf{\underline{U}}$ in $\dom{\textbf{\underline{U}}}{}$ that are consistent with realizations in $\dom{\textbf{\underline{U}}}{}^{'}$.
    Suppose variable $W$ belongs to the path $p$ and $Z'$ is a child of $W$ in that path. By the construction of $P(Z'|\Pa{Z'}{\G})$, we have
    \begin{equation}\label{eq:pp_to_pp}
        P'(Z'+W[p]|\Pa{Z'}{\G}) = P^{\M}(Z'|\Pa{Z'}{\G}).
    \end{equation}
    This is because $M'(Z'+W[p])=M(Z')$.
    Let $\mathbf{v}'\in \dom{\mathbf{V}}{}$ be a realization that is consistent with $\mathbf{v}^{\M}[\mathbf{V}\setminus\{Z'\}]$ and 
    $$
    \mathbf{v}'[Z'] = \mathbf{v}^{\M}[Z']-\mathbf{v}[W[p]].
    $$
    In this case, using \eqref{eq:pp_to_pp}, we have
        \begin{equation*}
    \begin{gathered}
        \theta_{i, j}^{'}(\mathbf{v}) = 
        \sum_{\mathbf{U}[\mathcal{P}]} 
        \prod_{\hat{U}[\hat{p}]\in \mathbf{U}[\mathcal{P}]} f_{\hat{p}}(\hat{U})
        \prod_{\hat{X}[\hat{p}] \in \mathbf{A}'_i[\mathcal{P}]} f_{\hat{p}}(\hat{X}[\hat{p}]|\Pa{\hat{X}}{\hat{p}})\\
        \times\Big( \sum_{\textbf{\underline{U}}^{\M}}
        P^{\M}({\mathbf{v}'}^{\M}[Z']|\Pa{Z'}{\G})
         \prod_{X \in \A'_i\setminus\{Z'\}} P^{\M}(x \mid \Pa{X}{\G}) \prod_{U\in \mathbf{U} \setminus \{U_0\}}\!P^{\M}(u^{\M})\Big).
    \end{gathered}
    \end{equation*}
    Note that the terms inside the big parenthesis is equal to $\theta_{i, j}(\mathbf{v}')$ given in \eqref{eq: theta and eta}, i.e., 
    \begin{equation*}
         \theta_{i, j}^{'}(\mathbf{v}) = \sum_{\mathbf{U}[\mathcal{P}]}
        \prod_{\hat{U}[\hat{p}]\in \mathbf{U}[\mathcal{P}]} f_{\hat{p}}(\hat{U})
        \prod_{\hat{X}[\hat{p}] \in \mathbf{A}'_i[\mathcal{P}]} f_{\hat{p}}(\hat{X}[\hat{p}]|\Pa{\hat{X}}{\hat{p}})
        \theta_{i, j}(\mathbf{v}').
    \end{equation*}
%    Note that from definition of $\theta_{i, j}$ defined for models $\M_1$ and $\M_2$ we have
    % \begin{equation*}
    %     \theta_{i, j}(\mathbf{v}') = \sum_{[\mathbf{U} \setminus \{U_0\}]^{\M}}
    %     P'({\mathbf{v}'}^{\M}[Z']|\Pa{Z'}{\G})
    %      \prod_{X \in \A'_i\setminus\{Z'\}} P^{\M}(x \mid \Pa{X}{\G}) \prod_{U\in \mathbf{U} \setminus \{U_0\}}\!P^{\M}(u^{\M}),
    % \end{equation*}
    %where the right side of equation written for the realization $\mathbf{v}$.The latter implies

    In the last equation, all terms on the right hand side except $\theta_{i, j}(\mathbf{v}')$ are independent of the realization of $\{U_0\}^{\M}$, i.e., independent of index $j$.
    For $j\in\{j_1,...,j_{\frac{\kappa+1}{2}}\}$ and using the result of Lemma \ref{lemma: gid equal indices} that says $\theta_{i, j_1}(\mathbf{v}')=...=\theta_{i, j_{\frac{\kappa+1}{2}}}(\mathbf{v}')$, we can conclude the result. 
    %The same holds for $\theta_{i,j_1}'(\mathbf{v}')$, $\theta_{i,j_2}'(\mathbf{v}')$, \dots, $\theta_{i,j_{\frac{\kappa+1}{2}}}'(\mathbf{v}')$. By combining this result and Lemma \ref{lemma: gid equal indices}, we conclude the result.
\end{proof}

\begin{lemma} \label{lemma: phi equal indices subcase 2}
     For any $\mathbf{d}^{\dagger} \in \dom{\mathbf{D}^{\dagger}}{}'$, we have
    \begin{equation*}
        \phi_{j_1}(\mathbf{d}^{\dagger}) = \phi_{j_2}(\mathbf{d}^{\dagger})=\dots = \phi_{j_{\frac{k+1}{2}}}(\mathbf{d}^{\dagger}).
    \end{equation*}
\end{lemma}
\begin{proof}
Similar to the previous lemma, by substituting $P^{\M'}$ from their definitions into Equation \eqref{eq: def phi subcase 2} and rearranging the terms, we obtain
    \begin{equation}\label{eq: phi simplification subcase 2}
    \begin{gathered}
        \phi_{j}^{'}(\mathbf{d}^{\dagger}) := \sum_{\mathbf{U}[\mathcal{P}]}
        \sum_{\textbf{\underline{U}}^{\M}}
        \sum_{\D}
        \prod_{\hat{p}\in \mathcal{P}_{U_0}}f_{\hat{p}}(U_0[\hat{p}])
        \prod_{\hat{X}[\hat{p}] \in \mathbf{S}[\mathcal{P}]} f_{\hat{p}}(\hat{X}[\hat{p}]|\Pa{\hat{X}}{\hat{p}}) \\
        \times\Big(
         P'(Z'|\Pa{Z'}{\G})
         \prod_{X \in \mathbf{S}\setminus\{Z'\}} P^{\M}(x^{\M} \mid \Pa{X}{\G}) \prod_{U\in \mathbf{U} \setminus \{U_0\}}\!P^{\M}(u^{\M})\Big),
    \end{gathered}
    \end{equation}
    where $\mathbf{U}[\mathcal{P}]:=\bigcup_{U\in\textbf{U}}\{U[\hat{p}]|\ \hat{p}\in \mathcal{P}_{U}\}$, $\underline{\textbf{U}}:=\mathbf{U} \setminus \{U_0\}$. Suppose that $l_1$ and $l_2$ are two integers such that
    \begin{equation*}
    \begin{split}
        & \gamma_{l_1} := (2x, 0, \dots, 0)),\\
        & \gamma_{l_2} := (2x+2 \pmod{\kappa+1}, 0, \dots, 0),
    \end{split}
    \end{equation*}
    and $x$ is an integer in $[0 : \frac{\kappa-1}{2}]$.
    We will prove that $\phi_{ l_1}(\mathbf{d}^{\dagger})=\phi_{l_2}(\mathbf{d}^{\dagger})$.

    Suppose that path $p$ is the sequence of variables: $Z'$, $D'_1$, $D'_2$ \dots $D'_{k'_1}$, $D'_{k'_1+1}:=Y'$. 
    Note that there is a direct edge between any consecutive nodes in this path and furthermore, the direct edge between $Z'$ and $D'_1$ is pointing toward $Z'$, i.e., $Z'\leftarrow D_1'$. 
    
    On the other hand, since $Z'$ and $U_0$ are both in $\S$ ($\S=\mathbf{S}_1$ by construction), then there exists a shortest path  $U_0, \hat{S}_1', \hat{U}_1', \hat{S}_2', \hat{U}_2', \dots, \hat{U}_{l'}', Z'$, such that $U_0$ is a parent of $\hat{S}_1'\in \S$, $Z'$ is a child of $\hat{U}_l'\in \U^{\S}$, and $\hat{U}_j'\in \U^{\S}$ is a parent of variables $\hat{S}_j'\in \S$ and  $\hat{S}_{j+1}'\in \S$ for $j \in [1:l'-1]$. 
    Let $\hat{\mathbf{U}}':=\{\hat{U}_1', \dots, \hat{U}_l'\}$, i.e., unobserved nodes in this shortest path except $U_0$. 
    For a given realization $\mathbf{o}_{1}$ of $\mathbf{U}\cup\mathbf{D}$,  we define $\mathbf{o}_{2} \in \dom{\mathbf{U}\cup\mathbf{D}}{}'$ as follows
    %\paragraph{Definition of $\textbf{o}_2$:}
    \begin{equation}
        \begin{split}
            & \mathbf{o}_{2}^{\M}[\hat{U}'_j] := \mathbf{o}_{1}^{\M}[\hat{U}'_j] + 2(-1)^{j} \pmod{\kappa+1}, \quad j\in[1:l'],\\
            % & \mathbf{u}_{2}^{\M}[U] := \mathbf{u}_{1}^{\M}[U], \quad \forall U\in \U^{\S} \setminus \hat{\mathbf{U}},
        \end{split}
    \end{equation}
    For $D_1'$, we have
    \begin{equation}\label{eq:56}
        \begin{split}
            & \mathbf{o}_2[D_1'[p]] = \mathbf{o}_1[D_1'[p]]-2(-1)^{l'} \pmod{\kappa+1}.
        \end{split}
    \end{equation}
    Note that with these modifications, for any $\widetilde{S} \in \S\setminus\{Z'\}$, we have
    \begin{equation*}
        \widetilde{s}_2 - M(\widetilde{S}) \equiv \widetilde{s}_1 - M(\widetilde{S}) \pmod{\kappa+1},
    \end{equation*}
     where $\widetilde{s}_1$ is a realization of $\widetilde{S}\Big|_{(\mathbf{U}\cup \D, \D^{\dagger}, U_0^\M)=(\mathbf{o}_1, \dd^{\dagger}, \gamma_{l_1})}$, $\widetilde{s}_2$ is a realization of $\widetilde{S}\Big|_{(\mathbf{U}\cup\D, \D^{\dagger}, U_0^\M)=(\mathbf{o_2}, \dd^{\dagger}, \gamma_{l_2})}$, and $M(\cdot)$ is given by Equation (\ref{eq: M(S)}). Additionally, 
    \begin{equation*}
        \mathbf{o}_2^{\M}[Z'] - M'(Z') \equiv \mathbf{o}_1^{\M}[Z'] - M'(Z') \pmod{\kappa+1},
    \end{equation*}
    where $M'(\cdot)$ is defined in Equation (\ref{eq: M'(S)}).
    This implies that for any $\widetilde{S}\in \mathbf{S}$, we have
    \begin{equation*}
    P^{\M}(\widetilde{s}|\Pa{\widetilde{S}}{\G})\Big|_{(\mathbf{U}\cup\mathbf{D}, \D^{\dagger}, U_0^{\M})=(\textbf{o}_1, \dd^{\dagger}, \gamma_{l_1})}=P^{\M}(\widetilde{s}|\Pa{\widetilde{S}}{\G})\Big|_{(\mathbf{U}\cup\mathbf{D}, \D^{\dagger}, U_0^{\M})=(\textbf{o}_2, \dd^{\dagger}, \gamma_{l_2})}.
    \end{equation*}
    Let $c := -2(-1)^{l'}$, then Equation \eqref{eq:56} becomes
    \begin{equation}\label{eq:after_56_1}
        \mathbf{o}_2^{\M}[D_1'[p]] = \mathbf{o}_1^{\M}[D_1'[p]]+c \pmod{\kappa+1}.
    \end{equation}
    Suppose that $D_j'$ is not a collider on the path $p$ and  $j\in [2:k_1'+1]$. We define $\mu(D_j')$ to be the number of colliders on a part of the path $p$ from $D_1'$ to $D_{j-1}'$.
    Thus, for those $j\in[2:k_1'+1]$ that $D_j'$ is not a collider, we define
    \begin{equation}\label{eq:after_56_2}
        \mathbf{o}_2^{\M}[D_j'[p]] := \mathbf{o}_1^{\M}[D_j'[p]] + c(-1)^{\mu(D_j')}.
    \end{equation}
    % and
    % \begin{equation*}
    %     \mathbf{o}_2^{\M}[Y'[p]] := \mathbf{o}_1^{\M}[Y'[p]] + c(-1)^{\mu(Y')}.
    % \end{equation*}
    Note that the modifications in \eqref{eq:after_56_2} might only affect the function $f_{p}(\cdot|\cdot)$.  
    Next, we show that after these modifications, function $f_{p}(\cdot|\cdot)$ remains unchanged. 
    To do so, for $j \in [1:k'_1+1]$, we consider four different cases:
    \begin{enumerate}
        \item If $D_j'$ has no parents, then it is obvious that
         \begin{equation*}
            f_{p}(D_j'[p])\Big|_{(\mathbf{U}\cup\mathbf{D}, U_0^{\M})=(\textbf{o}_1, \gamma_{l_1})} = f_{p}(D_j'[p])\Big|_{(\mathbf{U}\cup\mathbf{D}, U_0^{\M})=(\textbf{o}_2, \gamma_{l_2})}.
        \end{equation*}
        \item $D_j'$ is a collider, then $\mu(D_{j+1}') = \mu(D_{j-1}')+1$ and
        \begin{align*}
            & \mathbf{o}_1[D_{j+1}'[p]] + \mathbf{o}_1[D_{j-1}'[p]] = \mathbf{o}_2[D_{j+1}'[p]] + \mathbf{o}_2[D_{j-1}'[p]],
        \end{align*}
        and hence, according to  Equation \eqref{eq: f for not starting node}, we have
        \begin{equation*}
          f_{p}(D_j'[p]|\Pa{D_j'}{p})\Big|_{(\mathbf{U}\cup\mathbf{D}, U_0^{\M})=(\textbf{o}_1, \gamma_{l_1})} = f_{p}(D_j'[p]|\Pa{D_j'}{p})\Big|_{(\mathbf{U}\cup\mathbf{D}, U_0^{\M})=(\textbf{o}_2, \gamma_{l_2})}.
        \end{equation*}
        
        \item $D'_j$ is a child of $D'_{j+1}$, then $\mu(D'_{j}) = \mu(D'_{j+1})$ and
        \begin{align*}
            & \mathbf{o}_1[D_{j}'[p]] - \mathbf{o}_1[D_{j+1}'[p]] = \mathbf{o}_2[D_{j}'[p]] - \mathbf{o}_2[D_{j+1}'[p]].
        \end{align*}
        According to  Equation \eqref{eq: f for not starting node}, we imply that
        \begin{equation*}
           f_{p}(D_j'[p]|\Pa{D_j'}{p})\Big|_{(\mathbf{U}\cup\mathbf{D}, U_0^{\M})=(\textbf{o}_1, \gamma_{l_1})} = f_{p}(D_j'[p]|\Pa{D_j'}{p})\Big|_{(\mathbf{U}\cup\mathbf{D}, U_0^{\M})=(\textbf{o}_2, \gamma_{l_2})}.
        \end{equation*}
        
        \item $D'_j$ is a child of $D'_{j-1}$, then $\mu(D_{j}') = \mu(D_{j-1}')$ and
        \begin{align*}
            & \mathbf{o}_1[D_{j}'[p]] - \mathbf{o}_1[D_{j-1}'[p]] = \mathbf{o}_2[D_{j}'[p]] - \mathbf{o}_2[D_{j-1}'[p]].
        \end{align*}
        Similarly, according to Equation \eqref{eq: f for not starting node}, we get
        \begin{equation*}
           f_{p}(D_j'[p]|\Pa{D_j'}{p})\Big|_{(\mathbf{U}\cup\mathbf{S}, U_0)=(\textbf{o}_1, \gamma_{l_1})} = f_{p}(D_j'[p]|\Pa{D_j'}{p})\Big|_{(\mathbf{U}\cup\mathbf{S}, U_0)=(\textbf{o}_2, \gamma_{l_2})}.
        \end{equation*}
    \end{enumerate}

    This concludes that for any $j\in[1:k_1'+1]$,
    \begin{equation*}
        f_{p}(D_j'[p]|\Pa{D'_j}{p})\Big|_{(\mathbf{U}\cup\mathbf{D}, U_0^{\M})=(\textbf{o}_1, \gamma_{l_1})} = f_{p}(D_j'[p]|\Pa{D_j'}{p})\Big|_{(\mathbf{U}\cup\mathbf{D}, U_0^{\M})=(\textbf{o}_2, \gamma_{l_2})}.
    \end{equation*}
    % and 
    % \begin{equation*}
    %     f_{p}(Y'[p]|\Pa{Y'}{p})\Big|_{(\mathbf{U}\cup\mathbf{D}, U_0^{\M})=(\textbf{o}_1, \gamma_{l_1})} = f_{p}(Y'[p]|\Pa{Y'}{p})\Big|_{(\mathbf{U}\cup\mathbf{D}, U_0^{\M})=(\textbf{o}_2, \gamma_{l_2})}.
    % \end{equation*}

    Note that the aforementioned transformation of $\mathbf{o}_1$ affects only those realizations of variables that are used for the marginalization in the Equation \eqref{eq: phi simplification subcase 2}. Putting the above results together implies that the terms in \eqref{eq: phi simplification subcase 2} remain unchanged, i.e., 
    \begin{align*}
        &\prod_{\hat{p}\in \mathcal{P}_{U_0}}f_{\hat{p}}(U_0[\hat{p}])
        \prod_{\hat{X}[\hat{p}] \in \mathbf{S}[\mathcal{P}]} f_{\hat{p}}(\hat{X}[\hat{p}]|\Pa{\hat{X}}{\hat{p}}) \\
        &\times\Big(
         P'(Z'|\Pa{Z'}{\G})
         \prod_{X \in \mathbf{S}\setminus\{Z'\}} P^{\M}(x^{\M} \mid \Pa{X}{\G}) \prod_{U\in \mathbf{U} \setminus \{U_0\}}\!P^{\M}(u^{\M})\Big)\Big|_{(\mathbf{U}\cup\mathbf{D}, \D^{\dagger}, U_0^{\M})=(\textbf{o}_1, \dd^{\dagger}, \gamma_{l_1})}  \\
        & = \prod_{\hat{p}\in \mathcal{P}_{U_0}}f_{\hat{p}}(U_0[\hat{p}])
        \prod_{\hat{X}[\hat{p}] \in \mathbf{S}[\mathcal{P}]} f_{\hat{p}}(\hat{X}[\hat{p}]|\Pa{\hat{X}}{\hat{p}}) \\
        &\times\Big(
         P'(Z'|\Pa{Z'}{\G})
         \prod_{X \in \mathbf{S}\setminus\{Z'\}} P^{\M}(x^{\M} \mid \Pa{X}{\G}) \prod_{U\in \mathbf{U} \setminus \{U_0\}}\!P^{\M}(u^{\M})\Big)\Big|_{(\mathbf{U}\cup\mathbf{D}, \D^{\dagger}, U_0^{\M})=(\textbf{o}_2, \dd^{\dagger}, \gamma_{l_2})}
    \end{align*}
    
    This implies that $\phi_{l_1}(\mathbf{d}^{\dagger})=\phi_{l_2}(\mathbf{d}^{\dagger})$. By varying $x$ within $[0 : \frac{\kappa-1}{2}]$ in the definition of $\gamma_{l_1}$ and $\gamma_{l_2}$, we  obtain the result.
\end{proof}

\begin{lemma}
    \label{lemma: eta not equal indices subcase 2}
    There exists $0<\epsilon<\frac{1}{\kappa}$, such that there exists $\mathbf{v}_0 \in \dom{\mathbf{V}}{}'$ and $1\leq r <t\leq \frac{\kappa+1}{2}$ such that
    \begin{equation*}
        \eta'_{j_r}(\mathbf{v}_0) \neq \eta'_{j_t}(\mathbf{v}_0).
    \end{equation*}
\end{lemma}
\begin{proof}
By substituting $P^{\M'}$ from their definitions into Equation \eqref{eq: def eta subcase 2} and rearranging the terms, we obtain
    \begin{equation}\label{eq: eta simplification subcase 2}
    \begin{gathered}
        \eta_{j}^{'}(\mathbf{v}_0) = 
        \sum_{\mathbf{U}^{\M} \setminus \{U_0\}^{\M}}
        \sum_{\mathbf{U}[\mathcal{P}]}
        \prod_{\hat{U}[\hat{p}]\in \mathbf{U}[\mathcal{P}]} f_{\hat{p}}(\hat{U})
        \prod_{\hat{X}[\hat{p}] \in \mathbf{S}[\mathcal{P}]} f_{\hat{p}}(\hat{X}[\hat{p}]|\Pa{\hat{X}}{\hat{p}}) \\
        \times\Big(
         P'(Z'|\Pa{Z'}{\G})
         \prod_{X \in \mathbf{S}\setminus\{Z'\}} P^{\M}(x^{\M} \mid \Pa{X}{\G}) \prod_{U\in \mathbf{U} \setminus \{U_0\}}\!P^{\M}(u^{\M})\Big),
    \end{gathered}
    \end{equation}
    Next, we define $\mathbf{v}_0\in \dom{\mathbf{V}}{}'$ such that the conditions in the lemma hold.
    \begin{itemize}
        \item For any path $\hat{p}\in \mathcal{P}$ and any node $W$ on the path $\hat{p}$ that is not a starting node for path $\hat{p}$, we define
        \begin{equation*}
            \mathbf{v}_0[W[\hat{p}]]:=0.
        \end{equation*}
        \item For any variable $S \in \S$, we define
        \begin{equation*}
            \mathbf{v}_0^{\M}[S] := 0.
        \end{equation*}
        \item For the remaining part of $\textbf{v}_0$, we choose a realization such that for the selected $\textbf{v}_0$, there exists a realization for the unobserved variables $\mathbf{U}$ that ensures $\mathbb{I}(S) = 0$ for all $S \in \S$. This is clearly possible due to the definition of $\mathbb{I}(S)$.
        
        % \begin{equation*}
        %     \mathbb{I}(S) = 0.
        % \end{equation*}
    \end{itemize}
    
    Assume $r$ and $t$ are such that $\gamma_{j_r} := (0, 0, \dots, 0)$ and $\gamma_{j_t} := (2, 0, \dots, 0)$.  
    To finish the proof of the lemma, it is enough to show that $\eta'_{j_r}(\mathbf{v}_0)$ and $\eta'_{j_t}(\mathbf{v}_0)$ are two different polynomial functions of parameter $\epsilon$.
    We prove that these two polynomials are different by showing that $\eta'_{j_r}(\mathbf{v}_0)\neq\eta'_{j_t}(\mathbf{v}_0)$ for $\epsilon=0$.
    
    We only need to consider the non-zero terms in Equation \eqref{eq: eta simplification subcase 2}. From \eqref{eq: eta simplification subcase 2}, we have
    %written for $\eta'_r(\mathbf{v}_0)$ and $\eta'_t(\mathbf{v}_0)$, that is
    \begin{equation}
    \label{eq: main term lemma subcase 2}
    \begin{gathered}
         \prod_{\hat{U}[\hat{p}]\in \mathbf{U}[\mathcal{P}]} f_{\hat{p}}(\hat{U})
        \prod_{\hat{X}[\hat{p}] \in \mathbf{S}[\mathcal{P}]} f_{\hat{p}}(\hat{X}[\hat{p}]|\Pa{\hat{X}}{\hat{p}}) \\
        \times\Big(
         P'(Z'|\Pa{Z'}{\G})
         \prod_{X \in \mathbf{S}\setminus\{Z'\}} P^{\M}(x^{\M} \mid \Pa{X}{\G}) \prod_{U\in \mathbf{U} \setminus \{U_0\}}\!P^{\M}(u^{\M})\Big).
    \end{gathered}
    \end{equation}
    Note that $f_{\hat{p}}(\hat{U})=\frac{1}{\kappa+1}$ and $f_{\hat{p}}(\hat{X}|\Pa{\hat{X}}{\hat{p}'})$ is non-zero only 
    \begin{itemize}
       % \item $\hat{X}$ is a collider and $\hat{p}'=p$, 
        
        \item when there exists a variable $W\in \mathbf{F}$ such that $\hat{p}'=p_W$, $\hat{X}$ is a child of $W$ in path $p_W$, and
        $$
        \hat{X}[\hat{p}'] \equiv W[p] \pmod{\kappa+1}.
        $$
    
        \item when the following holds
        $$
        \hat{X}[\hat{p}'] \equiv \sum_{\hat{X}'\in \Pa{\hat{X}}{\hat{p}'}\setminus\{W\}}\hat{X}'[\hat{p}'] \pmod{\kappa+1}.
        $$
    \end{itemize}
    Similarly, $P^\M(X|\Pa{X}{\G})$ is non-zero
    \begin{itemize}
        \item if $\mathbb{I}(X)=1$ (i.e. $P^\M(X|\Pa{X}{\G})=\frac{1}{\kappa+1}$), or
        
        \item if $X \equiv M(X) \pmod{\kappa+1}$ for $P^\M(X|\Pa{X}{\G})$,

    \end{itemize}
    $P'(Z'|\Pa{Z'}{\G})$ is non-zero 
    \begin{itemize}
         \item if $\mathbb{I}(Z')=1$ (i.e. $P^\M(Z'|\Pa{Z'}{\G})=\frac{1}{\kappa+1}$), or
        
        \item $Z' \equiv M'(Z') \pmod{\kappa+1}$ for $P'(Z'|\Pa{Z'}{\G})$.
    \end{itemize}
    Let fix a realization $\mathbf{u} \in \dom{\mathbf{U}\setminus\{U_0^{\M}\}}{}'$. 
    We consider two scenarios:
    
    \textbf{I)}
    Assume that for this realization, there is a variable $S\in \S$, such that $\mathbb{I}(S)=1$ and $S$ is the closest variable to $U_0$ considering only paths with bidirected edges in $\G'[\S]$. 
    The value of $S^{\M}$ does not depend on its parents because of $\mathbb{I}(S)=1$ and Equation \eqref{eq: def P(S|Pa(S)) gid}.
    Additionally in the graph $\G'[\S]$, there exists a path $U_0, \hat{S}_1', \hat{U}_1', \hat{S}_2', \hat{U}_2', \dots, \hat{U}_{l'}', S$, such that $U_0$ is a parent of $\hat{S}_1'\in \S$, $S$ is a child of $\hat{U}_l'\in \U^{\S}$, and $\hat{U}_j'\in \U^{\S}$ is a parent of variables $\hat{S}_j'\in \S$ and  $\hat{S}_{j+1}'\in \S$ for $j \in [1:l'-1]$. 
    Let $\hat{\mathbf{U}}':=\{\hat{U}_1', \dots, \hat{U}_l'\}$. 
    We define $\mathbf{u}' \in \dom{\mathbf{U}\setminus\{U_0^{\M}\}}{}'$ 
    that is consistent with $\mathbf{u}$ except the variables in $\hat{\mathbf{U}}'$. For these variables, we define
    \begin{equation}
        \begin{split}
            & \mathbf{u}'^{\M}[\hat{U}_j] := \mathbf{u}^{\M}[\hat{U}_j] + 2(-1)^{j} \pmod{\kappa+1}, \quad j\in[1:l'],\\
            % & \mathbf{u}_{2}^{\M}[U] := \mathbf{u}_{1}^{\M}[U], \quad \forall U\in \U^{\S} \setminus \hat{\mathbf{U}},
        \end{split}
    \end{equation}
    
    With this modification for any $\widetilde{S}\in \mathbf{S}$, we have
    \begin{equation*}
    P^{\M}(\widetilde{s}|\Pa{\widetilde{S}}{\G})\Big|_{(\mathbf{U})=(\mathbf{u}, \gamma_{j_r})}=P(\widetilde{s}|\Pa{\widetilde{S}}{\G})\Big|_{(\mathbf{U})=(\textbf{u}', \gamma_{j_t})}.
    \end{equation*}
    Therefore for all such realizations of $\mathbf{u}$, the summation of the following terms for both $\eta'_{j_r}(\mathbf{v}_0)$ and $\eta'_{j_t}(\mathbf{v}_0)$ will be the same,
    \begin{equation}
    \label{eq: main term eta subcase 2}
    \begin{gathered}
         \prod_{\hat{U}[\hat{p}]\in \mathbf{U}[\mathcal{P}]} f_{\hat{p}}(\hat{U})
        \prod_{\hat{X}[\hat{p}] \in \mathbf{S}[\mathcal{P}]} f_{\hat{p}}(\hat{X}[\hat{p}]|\Pa{\hat{X}}{\hat{p}}) \\
        \times\Big(
         P'(Z'|\Pa{Z'}{\G})
         \prod_{X \in \mathbf{S}\setminus\{Z'\}} P^{\M}(x^{\M} \mid \Pa{X}{\G}) \prod_{U\in \mathbf{U} \setminus \{U_0\}}\!P^{\M}(u^{\M})\Big).
    \end{gathered}
    \end{equation}
    %accumulate in the the same impact for both .

    \textbf{II)} Assume that for all $S\in \S$, we have $\mathbb{I}(S)=0$.
    We consider a realization $U_0^{\M}=\gamma_{j_r}$ and $\mathbf{u}$ such that:
    \begin{itemize}
        %\item for all $S\in \S$ we have $\mathbb{I}(S)=0$.
        \item $\mathbf{u}[\U^{\S}] = \mathbf{0}$, and
        \item for all $U\in \mathbf{U}$ and any path $\hat{p}\in \mathcal{P}$ which contains $U$, $\mathbf{u}[U[\hat{p}]] = 0.$
    \end{itemize}
    We claim that for such $\mathbf{u}$, 
    \begin{equation*}
    \begin{gathered}
        \prod_{\hat{U}[\hat{p}]\in \mathbf{U}[\mathcal{P}]} f_{\hat{p}}(\hat{U})
        \prod_{\hat{X}[\hat{p}] \in \mathbf{S}[\mathcal{P}]} f_{\hat{p}}(\hat{X}[\hat{p}]|\Pa{\hat{X}}{\hat{p}}) \\
        \times\Big(
         P'(Z'|\Pa{Z'}{\G})
         \prod_{X \in \mathbf{S}\setminus\{Z'\}} P^{\M}(x^{\M} \mid \Pa{X}{\G}) \prod_{U\in \mathbf{U} \setminus \{U_0\}}\!P^{\M}(u^{\M})\Big).
    \end{gathered}
    \end{equation*}
    is non-zero. To prove this claim, we consider four cases:
    \begin{itemize}
        \item  assume that $\hat{p}\in \mathcal{P}$ and exists a variable $W$ such that $\hat{p}=p_W$. Let $\hat{X}$ be a child of $W$ in path $p_W$. From the definitions of $\mathbf{u}$ and $\mathbf{v}_0$, we get 
        $$
        \hat{X}[\hat{p}] \equiv W[\hat{p}] \pmod{\kappa+1},
        $$ 
        and therefore $f_{\hat{p}}(\hat{X}[\hat{p}]|\Pa{\hat{X}}{\hat{p}})=1$.
        The above holds because $\hat{X}[\hat{p}]=0=W[\hat{p}] \pmod{\kappa+1}$.
        
        \item assume that $\hat{p}\in \mathcal{P}$ and $\hat{X}$ is a variable on this path such that it is neither a starting node on $\hat{p}$ nor a child of a starting node on path $\hat{p}$. 
        Then, from the definitions of $\mathbf{v}_0$ and $\mathbf{u}$ we get  
        $$
        \hat{X}[\hat{p}] \equiv \sum_{\hat{X}\in \Pa{\hat{X}}{\hat{p}}}\hat{X}[\hat{p}] \pmod{\kappa+1},
        $$
        and therefore $f_{\hat{p}}(\hat{X}[\hat{p}]|\Pa{\hat{X}}{\hat{p}})=1$. 
        The above holds because all the variables in the above equation are zero.
        
        \item assume $X\in \S\setminus\{Z'\}$. From the definitions of $\mathbf{v}_0$ and $\mathbf{u}$, we get 
        $$
        X^{\M} \equiv M(X) \pmod{\kappa+1},
        $$
        and therefore $P^{\M}(x^{\M}|\Pa{X}{\G})=1$.
        Again, the above holds because all the terms are zero.
        \item assume $X=Z'$, then
        $$
        Z' \equiv M'(Z') \pmod{\kappa+1},
        $$
        and consequently $P'(Z'|\Pa{Z'}{\G})=1$.
    \end{itemize}

    Now, we consider the case when $U_0^\M = \gamma_{j_t}$. 
    Assume that $W\in \mathbf{F}$ and $W'$ is the last descendant of $W$ on the path $p_W$.
    From the properties which we proved in Section \ref{sec: properties}, we have  $W'\in\mathbf{Z'}$ and by the definition of $\mathbf{v}_0$, we have $W'[p_w]=0$. 
    Assume $W''$ is a parent of $W'$ on the path $p_W$. 
    Note that  $f_{p_W}(W'[p_W]|\Pa{W'}{p_W})\neq0$  if and only if $W''[p_W]=0$. 
    Repeating the above reasoning for variables from  $W'$ to $W$, we conclude that $W[p]$ must be equal to 0, otherwise, there would be a term in Equation \eqref{eq: main term eta subcase 2} that is zero and this contradicts with the fact that Equation \eqref{eq: main term eta subcase 2} is non-zero.
    
    Assume that $Z', W'_1, W'_2, \dots, W'_{k'}, W'_{k'+1}:=Y'$ are the nodes on the path $p$. Next, we prove by backward induction that $W'_{i}=0$ for all $i\in [1:k'+1]$. 
    By definition of $\mathbf{v_0}$, we know that $Y'[p]=0$. 
    If $W'_{i}[p]=0$ for all $i \in [k''+1:k']$, we will prove that $W''_{k''}[p]=0$ as well.
    To do so, we consider the following three cases:
    \begin{itemize}
        \item $W'_{k''}$ is a collider on a path $p$. Then the fact that $W''_{k''}[p]=0$ follows immediately from the aforementioned reasoning and the fact that Equation \eqref{eq: main term eta subcase 2} is non-zero. 
        %the observations made before.
        
        \item $W'_{k''}$ is a child of $W'_{k''+1}$ and it is not a collider. Then, $f_p(W'_{k''}|\Pa{W''_{k''}}{p})\neq0$ if and only if $W'_{k''}[p] = W'_{k''+1}[p] = 0$.
        
        \item $W'_{k''}$ is a parent of $W'_{k''+1}$. Then, $f_p(W'_{k''+1}|\Pa{W''_{k''}}{p})\neq0$  if and only if $0 = W'_{k''+1}[p] = W'_{k''}[p]$.
    \end{itemize}

    This implies that  $W'_1=0$. Therefore, $P'(Z'|\Pa{Z'}{\G}) = P^{\M}(Z'|\Pa{Z'}{\G})$ because $M'(Z')=M(Z')$. 
    Furthermore, the above arguments show that all $f_{\hat{p}}(\cdot|\cdot)$ terms in Equation \eqref{eq: main term eta subcase 2} are equal to one. This simplifies the Equation \eqref{eq: main term eta subcase 2} to 
    \begin{equation*}
        \prod_{X \in \mathbf{S}} P^{\M}(x^{\M} \mid \Pa{X}{\G}) \prod_{U\in \mathbf{U} \setminus \{U_0\}}\!P^{\M}(u^{\M}).
    \end{equation*}
    However by the proof of Lemma 6 \cite{kivva2022revisiting}, we know that there is no realization of $\mathbf{U}^{\M}$ consistent with $U_0=\gamma_{j_t}$ such that:
    \begin{itemize}
        \item $\mathbb{I}(S) = 0$ for all $S \in \S$, and
%        \item $U_0^{\M}=\gamma_{j_t}$, and
        \item $x^{\M} \equiv M(X) \pmod{\kappa + 1}$ for all $X\in \S$. The latter is a necessary condition for $P^{\M}(x|\Pa{X}{\G})$ being non-zero. 
    \end{itemize}
    To summarize, we showed that for $U_0^{\M}=\gamma_{j_r}$, Equation \eqref{eq: main term eta subcase 2} is non-zero while it is zero for $U_0^{\M}=\gamma_{j_t}$. This implies that $\eta'_{j_r}(\mathbf{v}_0)\neq\eta_{j_t}'(\mathbf{v}_0)$ for $\epsilon=0$. 
    %as a corollary $\eta_{\gamma_{j_r}'(\mathbf{v}_0)}$ and $\eta_{\gamma_{j_t}'(\mathbf{v}_0)}$ are not equal as polynomial. The latter one guarantees the existence of $\epsilon \in (0, \frac{1}{\kappa})$.
\end{proof}

\subsubsection{Proof of Lemma \ref{lemma: construct models subcase 2}}
\begin{customlem}{\ref{lemma: construct models subcase 2}}
    Let $\mathbf{S}: = \Anc{\Y', \Z'}{\G[\mathbf{V} \setminus \X']}$ and $\D$ is a set of all nodes on the paths in $\mathcal{P}$ excluding $\Z'$. Then,
    \begin{equation}
        P_{\x'}(\dd|\mathbf{s}\setminus \dd)=\frac{
        Q[\mathbf{S}]
        }{
        \sum_{\D} Q[\mathbf{S}]
        } = Q[\D|\mathbf{S}\setminus\D]
    \end{equation}
    is not c-gID from $(\mathbb{A}, \G)$.
\end{customlem}
\begin{proof}
    We will show  that $Q[\mathbf{D}|\mathbf{S}\setminus\mathbf{D}]$ is not c-gID from $(\mathbb{A}', \G)$, where $\mathbb{A}' := \mathbb{A}\cup\{\mathbf{S}_i\}_{i=2}^{n}$. 
    To this end, we will construct two models $\M_1$ and $\M_2$ such that for each $i \in [0:m']$ and any $\mathbf{v}\in \mathbf{V}$:
    \begin{align}
        \label{eq: equal known dist subcase 2}
        Q^{\M_1}[\A'_i](\mathbf{v}) &= Q^{\M_2}[\A'_i](\mathbf{v}),\\
        \label{eq: equal denom num subcase 2}
        \sum_{\D}Q^{\M_1}[\mathbf{S}](\mathbf{v}) &= \sum_{\D}Q^{\M_2}[\mathbf{S}](\mathbf{v}'),
    \end{align}
    but there exists $\mathbf{v}_0 \in \dom{\mathbf{V}}{}'$ such that:
    \begin{equation}\label{eq: not equal num subcase 2}
        Q^{\M_1}[\mathbf{S}](\mathbf{v}_0) \neq Q^{\M_2}[\mathbf{S}](\mathbf{v}_0).
    \end{equation}
    Note that Equations (\ref{eq: equal denom num subcase 2})-(\ref{eq: not equal num subcase 2}) yield
    \begin{equation*}
         Q[\D|\mathbf{S}\setminus\D]^{\M_1}(\mathbf{v}_0) \neq Q[\D|\mathbf{S}\setminus\D]^{\M_2}(\mathbf{v}_0).
    \end{equation*}
    This means that $Q[\D|\mathbf{S}\setminus\D]$ is not c-gID from $(\mathbb{A}', \G)$.

    To this end, we consider two cases.

    \textbf{First case:} \\
    Suppose that there exists $i \in [0, m]$, such that $\widecheck{\mathbf{S}} \subset \mathbf{A}_i$.
    Then, consider the models $\M_1'$ and $\M_2'$ constructed in the section \ref{sec: appendix new models subcase 2}. According to the definitions of models $\M_1'$ and $\M_2'$ for any $\mathbf{v}\in \dom{\mathbf{V}}{}'$, and any $i\in [0:m']$, and any $g\in \{1, 2\}$, we have
    \begin{align*}
        & Q[\mathbf{A}'_i]^{\M_g'}(\mathbf{v}) := \sum_{U_0^{\M}}P^{\M_g}(u_0^{\M})\sum_{U_0[\mathcal{P}]}\prod_{\hat{p}\in \mathcal{P}_{U_0}}f_{\hat{p}}(U_0[\hat{p}])\sum_{\mathbf{U} \setminus \{U_0\}} \prod_{X \in \A_i'} P^{\M'}(x \mid \Pa{X}{\G}) \prod_{U\in \mathbf{U} \setminus \{U_0\}} P^{\M'}(u^{\M}),\\
        \sum_{\D}&Q[\mathbf{S}]^{\M_g'}(\mathbf{v}) := \sum_{U_0^\M}P^{\M_g}(u_0^{\M})\sum_{U_0[\mathcal{P}]}\prod_{\hat{p}\in \mathcal{P}_{U_0}}f_{\hat{p}}(U_0[\hat{p}])\sum_{\D}\sum_{\mathbf{U} \setminus \{U_0\}} \prod_{X \in \mathbf{S}} P^{\M'}(x \mid \Pa{X}{\G}) \prod_{U\in \mathbf{U} \setminus \{U_0\}} P^{\M'}(u^{\M}),\\
        & Q[\mathbf{S}]^{\M_g'}(\mathbf{v}) := \sum_{U_0^\M}P^{\M_g}(u_0^{\M})\sum_{U_0[\mathcal{P}]}\prod_{\hat{p}\in \mathcal{P}_{U_0}}f_{\hat{p}}(U_0[\hat{p}])\sum_{\mathbf{U} \setminus \{U_0\}} \prod_{X \in \mathbf{S}} P^{\M'}(x \mid \Pa{X}{\G}) \prod_{U\in \mathbf{U} \setminus \{U_0\}} P^{\M'}(u^{\M}).
    \end{align*}
    We can re-writing the above equations using the notations of $\theta'_{i, j}$, $\phi'_{j}$, and $\eta'_{j}$,
    \begin{align*}
        & Q[\mathbf{A}'_i]^{\M_1}(\mathbf{v}) = \sum_{j=1}^d \frac{1}{d}\theta'_{i, j}(\mathbf{v}), \\
        & Q[\mathbf{A}'_i]^{\M_2}(\mathbf{v}) = \sum_{j=1}^d p_j\theta'_{i, j}(\mathbf{v}), \\
        \sum_{\D} &Q[\mathbf{S}]^{\M_1}(\mathbf{v}) = \sum_{j=1}^d \frac{1}{d}\phi'_{j}(\mathbf{v}[\D^{\dagger}]), \\
        \sum_{\D} 
        &Q[\mathbf{S}]^{\M_2}(\mathbf{v}) = \sum_{j=1}^d p_j\phi'_{j}(\mathbf{v}[\D^{\dagger}]), \\
        & Q[\mathbf{S}]^{\M_1}(\mathbf{v}) = \sum_{j=1}^d \frac{1}{d}\eta'_{j}(\mathbf{v}), \\
        & Q[\mathbf{S}]^{\M_2}(\mathbf{v}) = \sum_{j=1}^d p_j\eta'_{j}(\mathbf{v}).
    \end{align*}
    The above equations imply the following equations.
    \begin{align*}
        & Q^{\M_2}[\A'_i](\mathbf{v}) - Q^{\M_1}[\A'_i](\mathbf{v}) = \sum_{j=1}^d (p_j - \frac{1}{d}) \theta_{i,j}'(\mathbf{v})
        \\
        \sum_{\D} & Q[\mathbf{S}]^{\M_2}(\mathbf{v}) - \sum_{\D} Q[\mathbf{S}]^{\M_1}(\mathbf{v}) = \sum_{j=1}^d (p_j - \frac{1}{d}) \phi_{j}'(\mathbf{v}[\D^{\dagger}])
        \\
        & Q^{\M_2}[\S](\mathbf{v}_0) -  Q^{\M_1}[\S](\mathbf{v}_0) = \sum_{j=1}^d (p_j - \frac{1}{d}) \eta_{j}'(\mathbf{v}_0)
        \\
        & \sum_{j=1}^d p_j - 1 = \sum_{j=1}^d (p_j - \frac{1}{d}).
    \end{align*}

     To prove the statement of the lemma it suffices to solve the following system of linear equations over parameters $\{p_j\}_{j=1}^d$ and show that it admits a solution. 
    \begin{align*}
        & \sum_{j=1}^d (p_j - \frac{1}{d}) \theta_{i,j}'(\mathbf{v}) = 0, \hspace{0.2cm}\forall \mathbf{v} \in \dom{\mathbf{V}}{}', i\in [0:m'],
        \\
        & \sum_{j=1}^d (p_j - \frac{1}{d}) \phi'_{j}(\dd^{\dagger}) = 0, \hspace{0.2cm}\forall \dd^{\dagger} \in \dom{\D^{\dagger}}{}, i\in [0:m'],
        \\
        & \sum_{j=1}^d (p_j - \frac{1}{d}) \eta'_{j}(\mathbf{v}_0) \neq 0, \hspace{0.2cm} \exists \mathbf{v}_0 \in \dom{\mathbf{V}}{}',
        \\
        & (p_j - \frac{1}{d}) = 0,
        \\
        & 0<p_j<1, \hspace{0.2cm} \forall j \in [1:d].
    \end{align*}
    Analogous to the proof of Lemma \ref{lemma: construct models subcase 1}, we use Lemmas \ref{lemma: theta equal indices subcase 2},  \ref{lemma: phi equal indices subcase 2}, and  \ref{lemma: eta not equal indices subcase 2} instead of Lemmas \ref{lemma: gid equal indices}, \ref{lemma: equal indices summation subcase 1} and \ref{lemma: gid not equal indices} respectively and conclude the result.  

    \textbf{Second case:}\\
     Suppose that there is no $i \in [0, m]$, such that $\S \subset \mathbf{A}_i$. This case is identical to the \textbf{Second case} of the Lemma \ref{lemma: construct models subcase 1}.
\end{proof}

\subsection{Proof of Lemma \ref{lemma: construct models subcase 3}}\label{sec: supportive materials subcase 3}

Recall that $\textbf{S}=\Anc{\Y'\cup \Z'}{\G[\V \setminus \X']}$ and it is assumed that is not gID from $(\mathbb{A}, \G)$.  $\textbf{S}$ consists of $\textbf{S}_1,...,\textbf{S}_n$ as its single c-components where $\textbf{S}_1$ is not gID.
Let $\S=\mathbf{S}_1$.
Clearly, we can add $\{\mathbf{S}_i\}_{i=2}^{n}$ to the known distributions and $\S$ remains not gID, i.e., $\S$ is not gID from $(\mathbb{A}', \G)$, where $\mathbb{A}' := \mathbb{A}\cup\{\mathbf{S}_i\}_{i=2}^{n}$. 
For simplicity, we denote $\mathbb{A}' = \{\mathbf{A}'\}_{i=0}^{m'}$.
Hence, using the method in Section \ref{sec: baseline}, we can construct two models $\M_1$ and $\M_2$ that are the same over the known distributions and different over $Q[\textbf{S}_1]$. These models  disagree on the distribution $Q[\mathbf{S}]$ as well, because $Q[\mathbf{S}] = \prod_{i=1}^{n}Q[\mathbf{S}_i]$. Below, we use these two models to introduce two new models to prove Lemma \ref{lemma: construct models subcase 3}.  

\subsubsection{New models for Lemma \ref{lemma: construct models subcase 3}} \label{sec: appendix new models subcase 3}
\renewcommand{\D}{\widetilde{\mathbf{D}}}
 Recall that $T$ is a node in $\S\setminus(\Z' \cup \Y')$, $p_T$ is a shortest directed path from node $T$ to the node $Z'$, $\widetilde{\mathbf{F}}$ is a set of all colliders on the path $\widetilde{p}$, $\widetilde{\mathcal{P}} := \{\widetilde{p}\} \cup \{p_T\} \cup \{ \widetilde{p}_W|W \in \widetilde{\mathbf{F}}\}$ and $\D$ is a set of all nodes on the paths from $\widetilde{\mathcal{P}}$ excluding the nodes in $\Z'$. Let $\D_{\mathcal{P}}$ be a set of all variables that belong to at least one path in $\mathcal{P}$. 
 
 Similar to the Section \ref{sec: appendix new models subcase 2} further we define new models $\widetilde{\M}_1$ and $\widetilde{\M}_2$ based on the models $\M_1$ and $\M_2$ defined in Section \ref{sec: baseline}.
 We say that a variable $D$ is a \textbf{starting node} of the path $\hat{p}\in\widetilde{\mathcal{P}}$ if
\begin{itemize}
    \item $D = T$ and $\hat{p} = p_T$, or
    \item $D = Z'$ and $\hat{p} = \widetilde{p}$, or
    \item $D \in \widetilde{\mathbf{F}}$, i.e., it is a collider on path $\widetilde{p}$, and $\hat{p} = \widetilde{p}_D$.
\end{itemize}
Note that $D$ can be a starting node of only one path.

\newcommand{\domm}[2]{\widetilde{\mathfrak{X}}_{#2}(#1)}

For $R \in \mathbf{V}\cup \mathbf{U}$, let $\widetilde{\alpha}_{p}(R)$ be the number of paths in $\widetilde{\mathcal{P}}$ that contains $R$. Furthermore, we use $\domm{R}{}$ and $\dom{R}{}$ to denote its domain in $\widetilde{\M}_1$ or $\widetilde{\M}_2$ (variables in different models have the same domains) and in $\M_1$ or $\M_2$ respectively. We define $\domm{R}{}$ as follows:
\begin{itemize}
    \item If $R$ is a starting node for one of the paths in $\mathcal{P}$
    \begin{align*}
        & \domm{R}{} := \dom{R}{}\times[0:\kappa]^{\widetilde{\alpha}_{p}(R)-1}.
    \end{align*}
    \item If $R$ is not a starting node for any of the paths in $\mathcal{P}$, then:
    \begin{align*}
        & \domm{R}{} := \dom{R}{}^{}\times[0:\kappa]^{\widetilde{\alpha}_{p}(R)}.
    \end{align*}
\end{itemize}
Consequently, if $R$ does not belong to any of the paths in $\widetilde{\mathcal{P}}$, then $\domm{R}{} := \dom{R}{}^{}$.

Consider $R \in \mathbf{V}\cup \mathbf{U}$. According to the domain's definitions above, $R$ is a vector that is a concatenation of the vector coming from $\dom{R}{}$ in model $\M_1$ (or $\M_2$) and some additional coordinates. These additional coordinates are defined based on $\widetilde{\alpha}_{p}(R)$. More precisely, if $R$ is not a starting node of a path $\hat{p}\in \widetilde{\mathcal{P}}$, then there is a coordinate assigned to this path, denoted by $R[\hat{p}]$, otherwise, if R is a starting node of $\hat{p} \in \widetilde{\mathcal{P}}$, then there is no coordinate assigned this path.

Let $\mathbf{O} \in \mathbf{V}\cup \mathbf{U}$. 
For any realization $\mathbf{o} \in \domm{\mathbf{O}}{}$ of $\mathbf{O}$ we denote by $\mathbf{o}^\M \in \dom{O}{}$ a realization of $\mathbf{O}$ that is consistent with $\mathbf{o}$. With slight abuse of notation, we use $\mathbf{O}$ and $\mathbf{O}^{\M}$ to denote realizations of $\mathbf{O}$ in models $\widetilde{\M}_i$ and $\M_i$, respectively. $\mathbf{O}^{\M}$ means realizations in $\dom{O}{}$ that are consistent with realizations in $\domm{\mathbf{O}}{}$.

Recall that $\mathbf{D}_{\widetilde{\mathcal{P}}}$ is a set of all variables on the paths in $\widetilde{\mathcal{P}}$. Let $D\in \D_{\mathcal{P}}$. We denote by $\widetilde{\mathcal{P}}_D$ the set of all paths $\hat{p}$, such that $\hat{p}\in \widetilde{\mathcal{P}}$, $D$ belongs to the path $\hat{p}$, and $D$ is not a starting node of path $\hat{p}$. We are ready now to define the probabilities of $P^{\widetilde{\M}_i}(D|\Pa{D}{\G})$ for any $D \in \mathbf{V}\cup\mathbf{U}$ and $i \in \{1, 2\}$.

\begin{itemize}
    \item If $D$ does not belong to the set $\D_{\mathcal{P}}$, we define
    \begin{align*}
        P^{\widetilde{\M}_i}(D|\Pa{D}{\G}) := P^{\M_i}(D|\Pa{D}{\G}).
    \end{align*}
    \item If $D$ belongs to the set $\D_{\mathcal{P}}\setminus \{Z'\}$, we define
    \begin{equation}
        P^{\widetilde{\M}_i}(D|\Pa{D}{\G}) := P^{\M_i}(D^{\M}|\Pa{D}{\G})\prod_{\hat{p}\in \widetilde{\mathcal{P}}_D}f_{\hat{p}}(D[\hat{p}]|\Pa{D}{\hat{p}}),
    \end{equation}
    where 
    \begin{itemize}
        \item if $D\neq Z'$ or $\hat{p} \neq p_T$ then $\Pa{D}{\hat{p}}$ is a parents of $D$ in a path $\hat{p}$;
        \item if $D = Z'$  and $\hat{p} = p_T$ then $\Pa{Z'}{\hat{p}}$ is a parents of $Z'$ on the paths $\hat{p}$ and $p_T$;
        \item $f_{\hat{p}}(D|\Pa{D}{\hat{p}})$ is given below.
    \end{itemize}

    \textbf{Definition of function} $f_{\hat{p}}(D[\hat{p}]|\Pa{D}{\hat{p}})$:
    \begin{itemize}
        \item When there exists a variable $W \in \widetilde{\mathbf{F}}$ such that $\hat{p}=\widetilde{p}_W$ and $D$ is a child of $W$ on path $p_W$,
        \begin{equation*}
        f_{\hat{p}}(D[\hat{p}]|\Pa{D}{\hat{p}}) :=
        \begin{cases} 
            1-\kappa\epsilon &  \text{ if }  D[\hat{p}]\equiv W[\widetilde{p}] \pmod{\kappa+1}\\
            \epsilon &  \text{ if }  D[\hat{p}]\not\equiv W[\widetilde{p}] \pmod{\kappa+1}.
        \end{cases}
        \end{equation*}
        \item When $\hat{p}=p_T$ and $D$ is a child of $T$ on path $\hat{p}$,
        \begin{equation*}
        f_{\hat{p}}(D[\hat{p}]|\Pa{D}{\hat{p}}) :=
        \begin{cases} 
            1-\kappa\epsilon &  \text{ if }  D[\hat{p}]\equiv T^{\M} \pmod{\kappa+1}\\
            \epsilon &  \text{ if }  D[\hat{p}]\not\equiv T^{\M} \pmod{\kappa+1}.
        \end{cases}
        \end{equation*}
        \item When $\hat{p}=p_T$ and $D=Z'$. Suppose $Z'$ is a child of $W'$ on a path $p_T$ and is a child of $W''$ on a path $\widetilde{p}$,
        \begin{equation*}
        f_{\hat{p}}(Z'[p_T]|\Pa{D}{\hat{p}}) :=
        \begin{cases} 
            1-\kappa\epsilon &  \text{ if }  Z'[p_T]\equiv  W'[p_T] + W''[\widetilde{p}]\pmod{\kappa+1}\\
            \epsilon &  \text{ if }  Z'[p_T]\not\equiv W'[p_T] + W''[\widetilde{p}]  \pmod{\kappa+1}.
        \end{cases}
        \end{equation*}
        \item When $\Pa{D}{\hat{p}}=\emptyset$, 
        \begin{equation*}
            f_{\hat{p}}(D[\hat{p}]) := \frac{1}{\kappa+1}.
        \end{equation*}
        
        \item Otherwise,
        \begin{equation} \label{eq: define f otherwise subcase 3}
        f_{\hat{p}}(D[\hat{p}]|\Pa{D}{\hat{p}}) :=
        \begin{cases} 
            1-\kappa\epsilon &  \text{ if }  D[\hat{p}]\equiv \sum_{D' \in \Pa{D}{\hat{p}}}D'[\hat{p}] \pmod{\kappa+1}\\
            \epsilon &  \text{ if }  D[\hat{p}]\not\equiv \sum_{D' \in \Pa{D}{\hat{p}}}D'[\hat{p}] \pmod{\kappa+1},
        \end{cases}
        \end{equation}
        Note that $P^{\widetilde{\M}_i}(D|\Pa{D}{\G})$ is a probability distribution since for different paths $\hat{p}_1$ and $\hat{p}_2$, $D[\hat{p}_1]$ and $D[\hat{p}_2]$ are  different and also 
        \begin{equation*}
            \sum_{D[\hat{p}]\in \domm{D[\hat{p}]}{}} f_{\hat{p}}(D[\hat{p}]|\Pa{D}{\hat{p}}) = 1
        \end{equation*}
    \end{itemize}
\end{itemize}

Note that for any $W \in (\mathbf{V}\cup \mathbf{U})\setminus \{U_0\}$, we have
$$
P^{\widetilde{\M}_1}(W|\Pa{W}{\G}) = P^{\widetilde{\M}_2}(W|\Pa{W}{\G}).
$$ 
Therefore, we will use $P^{\widetilde{\M}}(W|\Pa{W}{\G})$ instead of $P^{\widetilde{\M}_1}(W|\Pa{W}{\G})$ or $P^{\widetilde{\M}_2}(W|\Pa{W}{\G})$ for  $W \in (\mathbf{V} \cup \mathbf{U})\setminus \{U_0\}$. 

We also have
\begin{equation}
\label{eq: prob U_0 subcase 3}
\begin{gathered}
    P^{\widetilde{\M}_1}(U_0) = \frac{1}{d} \prod_{\hat{p}\in \widetilde{\mathcal{P}}_{U_0}}f_{\hat{p}}(U_0[\hat{p}]), \\
    P^{\widetilde{\M}_2}(U_0) = P^{\M_2}(U_0^{\M}) \prod_{\hat{p}\in \widetilde{\mathcal{P}}_{U_0}}f_{\hat{p}}(U_0[\hat{p}]).
\end{gathered}
\end{equation}

Recall that $\mathbf{S}=\Anc{\Y'\cup Z'}{\G[\mathbf{V}\setminus\X']}$. Let $\D':=\mathbf{S}\setminus\D$ and $\D^{\dagger}:=\mathbf{V}\setminus\D$.  
For $i \in [0:m']$, $j\in [1:d]$, $\mathbf{v} \in \domm{\mathbf{V}}{}$ and $\mathbf{d}^{\dagger} \in \domm{\D^{\dagger}}{}$, we define $\theta_{i, j}^{'}(\mathbf{v})$, $\phi_{j}^{'}(\mathbf{d}^{\dagger})$ and $\eta_{j}^{'}(\mathbf{v})$ as follows:
\begin{align}
    \label{eq: def theta subcase 3}
    & \widetilde{\theta}_{i, j}(\mathbf{v}) := \sum_{U_0[\widetilde{\mathcal{P}}]}\prod_{\hat{p}\in \widetilde{\mathcal{P}}_{U_0}}f_{\hat{p}}(U_0[\hat{p}])\sum_{\mathbf{U} \setminus \{U_0\}} \prod_{X \in \A_i'} P^{\widetilde{\M}}(x \mid \Pa{X}{\G}) \prod_{U\in \mathbf{U} \setminus \{U_0\}} P^{\widetilde{\M}}(u^{\M}),\\
    \label{eq: def phi subcase 3}
    & \widetilde{\phi}_{j}(\mathbf{d}^{\dagger}) := \sum_{U_0[\widetilde{\mathcal{P}}]}\prod_{\hat{p}\in \widetilde{\mathcal{P}}_{U_0}}f_{\hat{p}}(U_0[\hat{p}])\sum_{\D}\sum_{\mathbf{U} \setminus \{U_0\}} \prod_{X \in \mathbf{S}} P^{\widetilde{\M}}(x \mid \Pa{X}{\G}) \prod_{U\in \mathbf{U} \setminus \{U_0\}} P^{\widetilde{\M}}(u^{\M}),\\
    \label{eq: def eta subcase 3}
    & \widetilde{\eta}_{j}(\mathbf{v}) := \sum_{U_0[\widetilde{\mathcal{P}}]}\prod_{\hat{p}\in \widetilde{\mathcal{P}}_{U_0}}f_{\hat{p}}(U_0[\hat{p}])\sum_{\mathbf{U} \setminus \{U_0\}} \prod_{X \in \mathbf{S}} P^{\widetilde{\M}}(x \mid \Pa{X}{\G}) \prod_{U\in \mathbf{U} \setminus \{U_0\}} P^{\widetilde{\M}}(u^{\M}),
\end{align}
where $\sum_{U_0[\widetilde{\mathcal{P}}]}$ is a summation over realizations of the random variables $\{U_0[\hat{p}]|\;\hat{p}\in \widetilde{\mathcal{P}}_{U_0}\}$. 

Next, we prove three lemmas similar to Lemmas \ref{lemma: theta equal indices subcase 2}, \ref{lemma: phi equal indices subcase 2}, and \ref{lemma: eta not equal indices subcase 2} for the new models $\widetilde{\M}_1$ and $\widetilde{\M}_2$.

\begin{lemma}
    \label{lemma: theta indices subcase 3}
    For any $\mathbf{v} \in \domm{\mathbf{V}}{}$ and $i\in[0:m']$,
    \begin{equation*}
        \widetilde{\theta}_{i,j_1}(\mathbf{v}) = \widetilde{\theta}_{i,j_2}(\mathbf{v}) = \cdots= \widetilde{\theta}_{i,j_{\frac{\kappa+1}{2}}}(\mathbf{v}).
    \end{equation*}
\end{lemma}
\begin{proof}
    By substituting $P^{\widetilde{\M}}$ from the above into Equation \eqref{eq: def theta subcase 2} and rearranging the terms, we obtain
    \begin{equation*}
    \begin{gathered}
        \widetilde{\theta}_{i, j}(\mathbf{v}) = 
        \sum_{\mathbf{U}[\widetilde{\mathcal{P}}]} 
        \prod_{\hat{U}[\hat{p}]\in \mathbf{U}[\widetilde{\mathcal{P}}]} f_{\hat{p}}(\hat{U})
        \prod_{\hat{X}[\hat{p}] \in \mathbf{A}'_i[\widetilde{\mathcal{P}}]} f_{\hat{p}}(\hat{X}[\hat{p}]|\Pa{\hat{X}}{\hat{p}}) \times\\
        \times\Big( \sum_{\mathbf{U}^{\M} \setminus \{U_0\}^{\M}}
         \prod_{X \in \A'_i} P^{\M}(x^{\M} \mid \Pa{X}{\G}) \prod_{U\in \mathbf{U} \setminus \{U_0\}}\!P^{\M}(u^{\M})\Big)
    \end{gathered}
    \end{equation*}
    
    % Let $\mathbf{v}'\in \dom{\mathbf{V}}{}$ be a realization that is consistent with $\mathbf{v}^{\M}$.

    Note that the terms inside the big parenthesis of the above equation is equal to $\theta_{i, j}$ given by \ref{eq: theta and eta}, i.e.,
    % \begin{equation*}
    %     \theta_{i, j}(\mathbf{v}') = \sum_{\mathbf{U} \setminus \{U_0\}^{\M}}
    %      \prod_{X \in \A'_i} P^{\M}(x^{\M} \mid \Pa{X}{\G}) \prod_{U\in \mathbf{U} \setminus \{U_0\}}\!P^{\M}(u^{\M}),
    % \end{equation*}
    % where the right side of equation written for the realization $\mathbf{v}$.
    % The latter implies
    \begin{equation*}
        \widetilde{\theta}_{i, j}(\mathbf{v}) = \sum_{\mathbf{U}[\widetilde{\mathcal{P}}]} 
        \prod_{\hat{U}[\hat{p}]\in \mathbf{U}[\widetilde{\mathcal{P}]}} f_{\hat{p}}(\hat{U})
        \prod_{\hat{X}[\hat{p}] \in \mathbf{A}'_i[\widetilde{\mathcal{P}}]} f_{\hat{p}}(\hat{X}[\hat{p}]|\Pa{\hat{X}}{\hat{p}})
        \theta_{i, j}(\mathbf{v}^{\M}).
    \end{equation*}
    In the last equation, all terms on the right-hand side except $\theta_{i, j}(\mathbf{v}^{\M})$ are independent of the realization of $\{U_0\}^{\M}$, i.e., independent of index j. For $j \in \{j_1, j_2, \dots, j_{\frac{\kappa+1}{2}}\}$ and using the result of Lemma \ref{lemma: gid equal indices} that says $\theta_{i, j_1}(\mathbf{v})=\theta_{i, j_2}(\mathbf{v})=\dots=\theta_{i, j_{\frac{\kappa+1}{2}}}(\mathbf{v})$, we can conclude the result.
\end{proof}

\begin{lemma} \label{lemma: phi indices subcase 3}
     For any $\mathbf{d}^{\dagger} \in \domm{\D^{\dagger}}{}$:
    \begin{equation*}
        \widetilde{\phi}_{j_1}(\mathbf{d}^{\dagger}) = \widetilde{\phi}_{j_2}(\mathbf{d}^{\dagger})=\dots = \widetilde{\phi}_{j_{\frac{k+1}{2}}}(\mathbf{d}^{\dagger}).
    \end{equation*}
\end{lemma}
\begin{proof}
    Similar to the previous lemma, by substituting $P^{\widetilde{\M}}$ from their definitions into Equation \eqref{eq: def phi subcase 3} and rearranging the terms, we obtain:
    \begin{equation}\label{eq: phi simplification subcase 3}
    \begin{gathered}
        \widetilde{\phi}_{j}(\mathbf{d}^{\dagger}) := \sum_{U_0[\widetilde{\mathcal{P}}]}
        \sum_{\mathbf{U}^{\M} \setminus \{U_0\}^{\M}}
        \sum_{\D}
        \prod_{\hat{p}\in \widetilde{\mathcal{P}}_{U_0}}f_{\hat{p}}(U_0[\hat{p}])
        \prod_{\hat{X}[\hat{p}] \in \mathbf{S}[\widetilde{\mathcal{P}}]} f_{\hat{p}}(\hat{X}[\hat{p}]|\Pa{\hat{X}}{\hat{p}})\times \\
        \times\Big(
         \prod_{X \in \mathbf{S}} P^{\M}(x^{\M} \mid \Pa{X}{\G}) \prod_{U\in \mathbf{U} \setminus \{U_0\}}\!P^{\M}(u^{\M})\Big)
    \end{gathered}
    \end{equation}

    Suppose that $l_1$ and $l_2$ are two integers such that
    \begin{equation*}
    \begin{split}
        & \gamma_{l_1} = (2x, 0, \dots, 0)),\\
        & \gamma_{l_2} = (2x+2 \pmod{\kappa+1}, 0, \dots, 0),
    \end{split}
    \end{equation*}
    and $x$ is an integer in $[0 : \frac{\kappa-1}{2}]$.
    We will prove that $\widetilde{\phi}_{ l_1}(\mathbf{d}^{\dagger})=\widetilde{\phi}_{l_2}(\mathbf{d}^{\dagger})$.

    Suppose that path $p$ is the sequence of variables: $Z'$, $D'_1$, $D'_2$ \dots, $D'_{k'_1}$, $D'_{k'_1+1}:=Y'$ and path $p_T$ is a sequence of variables: $T_0:=T$, $T_1$, \dots, $T_{k'_2}$, $T_{k'_2+1}:=Z'$. Note that direct edge between $Z'$ and $D'_1$ is pointing toward $Z'$, i.e., $Z'\leftarrow D_1'$ and for all $i\in [0, k'_2]$ variable $T_{i}$ is a parent of $T_{i+1}$ on the path $p_T$.
    
    On the other hand, since $T$ and $U_0$ are both in $\S$ ($\S=\mathbf{S}_1$ by construction), then there exists a shortest path  $U_0, \hat{S}_1', \hat{U}_1', \hat{S}_2', \hat{U}_2', \dots, \hat{U}_{l'}', T$, such that $U_0$ is a parent of $\hat{S}_1'\in \S$, $T$ is a child of $\hat{U}_l'\in \U^{\S}$, and $\hat{U}_j'\in \U^{\S}$ is a parent of variables $\hat{S}_j'\in \S$ and  $\hat{S}_{j+1}'\in \S$ for any $j \in [1:l'-1]$. 
    Let $\hat{\mathbf{U}}':=\{\hat{U}_1', \dots, \hat{U}_l'\}$, i.e., unobserved nodes in this shortest path except $U_0$. For any given realization $\mathbf{o}_{1} \in \domm{\mathbf{U}\cup\D}{}$,  we define $\mathbf{o}_{2} \in \domm{\mathbf{U}\cup\D}{}$ as follows,
    \begin{equation}
        \begin{split}
            & \mathbf{o}_{2}^{\M}[\hat{U}'_j] := \mathbf{o}_{1}^{\M}[\hat{U}'_j] + 2(-1)^{j} \pmod{\kappa+1}, \quad \forall j\in[0:l'],\\
            & \mathbf{o}_{2}^{\M}[T] := \mathbf{o}_{1}^{\M}[\hat{U}'_j] + 2(-1)^{l'} \pmod{\kappa+1}.
        \end{split}
    \end{equation}
    Note that if $\mathbf{o}_1^{\M}[U_0] = \gamma_{l_1}$ then $\mathbf{o}_2^{\M}[U_0] = \gamma_{l_2}$.
    With these modifications for any $S \in \S$, we obtain
    \begin{equation*}
        s_1 - M(S) \equiv s_2 - M(S) \pmod{\kappa+1},
    \end{equation*}
    where $s_1$ is a realization of $S\Big|_{(\mathbf{U}\cup\D, \D^{\dagger})=(\mathbf{o}_1, \dd^{\dagger})}$, $s_2$ is a realization of $S\Big|_{(\mathbf{U}\cup\D, \D^{\dagger})=(\mathbf{o}_2, \dd^{\dagger})}$, $M(\cdot)$ is given by Equation (\ref{eq: M(S)}). 
    This implies for any $S\in \mathbf{S}$, we have
    \begin{equation*}
    P^{\M}(s|\Pa{S}{\G})\Big|_{(\mathbf{U}\cup\D, \D^{\dagger})=(\textbf{o}_1, \dd^{\dagger})}=P(s|\Pa{S}{\G})\Big|_{(\mathbf{U}\cup\D, \D^{\dagger})=(\textbf{o}_2, \dd^{\dagger})}.
    \end{equation*}
    Let $c = -2(-1)^{l'}$ and we define
    \begin{align*}
        & \mathbf{o}_2^{\M}[T_j[p_T]] := \mathbf{o}_1^{\M}[T_j[p_T]] - c \pmod{\kappa+1},\quad \forall j \in [1:k'_2],
        \\
        & \mathbf{o}_2^{\M}[D'_1[\widetilde{p}]] :=  \mathbf{o}_1^{\M}[D'_1[\widetilde{p}]] + c \pmod{\kappa+1}.
    \end{align*}
    This implies that for all $j \in [1:k'_2+1]$ we have
    \begin{equation*}
        f_{p_T}(T_j|\Pa{T_j}{p_T})\Big|_{(\mathbf{U}\cup\D, \D^{\dagger})=(\mathbf{o}_1, \dd^{\dagger})} = f_{p_T}(T_j|\Pa{T_j}{p_T})\Big|_{(\mathbf{U}\cup\D, \D^{\dagger})=(\mathbf{o}_2, \dd^{\dagger})}.
    \end{equation*}
    Assume that $D_j'$ is not a collider on the path $\widetilde{p}$ and $j\in [2:k_1'+1]$. We define $\mu(D_j')$ to be the number of colliders on a part of the path $\widetilde{p}$ from $D_1'$ to $D_{j-1}'$. Thus, for those $j\in[2:k_1'+1]$ that $D_j'$ is not a collider, we define
    \begin{equation}\label{eq: phi modifications p subcase 3}
        \mathbf{o}_2^{\M}[D_j'[\widetilde{p}]] := \mathbf{o}_1^{\M}[D_j'[\widetilde{p}]] + c(-1)^{\mu(D_j')}.
    \end{equation}
    Note that the modifications in \eqref{eq: phi modifications p subcase 3} might only affect the function $f_{\widetilde{p}}(\cdot|\cdot)$. Next, we show that after these modifications, function $f_{\widetilde{p}}(\cdot|\cdot)$ remains unchanged. To do so, for $j \in [1:k_1'+1]$ we consider four different cases:
    \begin{enumerate}
        \item If $D_j'$ has no parents, then it is obvious that
         \begin{equation*}
            f_{\widetilde{p}}(D_j'[\widetilde{p}])\Big|_{(\mathbf{U}\cup\D)=(\textbf{o}_1)} = f_{\widetilde{p}}(D_j'[\widetilde{p}])\Big|_{(\mathbf{U}\cup\D)=(\textbf{o}_2)}.
        \end{equation*}
        \item If $D_j'$ is a collider, then $\mu(D_{j+1}') = \mu(D_{j-1}')+1$ and
        \begin{align*}
            & \mathbf{o}_1[D_{j+1}'[\widetilde{p}]] + \mathbf{o}_1[D_{j-1}'[\widetilde{p}]] = \mathbf{o}_2[D_{j+1}'[\widetilde{p}]] + \mathbf{o}_2[D_{j-1}'[\widetilde{p}]],
        \end{align*}
        and hence, according to the Equation \eqref{eq: define f otherwise subcase 3}, we have
        \begin{equation*}
          f_{\widetilde{p}}(D_j'[\widetilde{p}]|\Pa{D_j'}{\widetilde{p}})\Big|_{(\mathbf{U}\cup\D)=(\mathbf{o}_1)} = f_{\widetilde{p}}(D_j'[\widetilde{p}]|\Pa{D_j'}{\widetilde{p}})\Big|_{(\mathbf{U}\cup\D)=(\textbf{o}_2)}.
        \end{equation*}
        
        \item If $D'_j$ is a child of $D'_{j+1}$, then $\mu(D'_{j}) = \mu(D'_{j+1})$ and
        \begin{align*}
            & \mathbf{o}_1[D_{j}'[\widetilde{p}]] - \mathbf{o}_1[D_{j+1}'[\widetilde{p}]] = \mathbf{o}_2[D_{j}'[\widetilde{p}]] - \mathbf{o}_2[D_{j+1}'[\widetilde{p}]].
        \end{align*}
        According to Equation \eqref{eq: define f otherwise subcase 3}, we imply that
        \begin{equation*}
           f_{\widetilde{p}}(D_j'[\widetilde{p}]|\Pa{D_j'}{\widetilde{p}})\Big|_{(\mathbf{U}\cup\D)=(\textbf{o}_1)} = f_{\widetilde{p}}(D_j'[\widetilde{p}]|\Pa{D_j'}{\widetilde{p}})\Big|_{(\mathbf{U}\cup\D)=(\textbf{o}_2)}.
        \end{equation*}
        
        \item If $D'_j$ is a child of $D'_{j-1}$, then $\mu(D_{j}') = \mu(D_{j-1}')$ and
        \begin{align*}
            & \mathbf{o}_1[D_{j}'[\widetilde{p}]] - \mathbf{o}_1[D_{j-1}'[\widetilde{p}]] = \mathbf{o}_2[D_{j}'[\widetilde{p}]] - \mathbf{o}_2[D_{j-1}'[\widetilde{p}]].
        \end{align*}
        Similarly, according to Equation \eqref{eq: define f otherwise subcase 3}, we get
        \begin{equation*}
           f_{\widetilde{p}}(D_j'[\widetilde{p}]|\Pa{D_j'}{\widetilde{p}})\Big|_{(\mathbf{U}\cup\D)=(\textbf{o}_1)} = f_{\widetilde{p}}(D_j'[\widetilde{p}]|\Pa{D_j'}{\widetilde{p}})\Big|_{(\mathbf{U}\cup\D)=(\textbf{o}_2)}.
        \end{equation*}
    \end{enumerate}

    This concludes that for any $j\in[1:k_1'+1]$,
    \begin{equation*}
        f_{\widetilde{p}}(D_j'[\widetilde{p}]|\Pa{D_j'}{\widetilde{p}})\Big|_{(\mathbf{U}\cup\D)=(\textbf{o}_1)} = f_{\widetilde{p}}(D_j'[\widetilde{p}]|\Pa{D_j'}{\widetilde{p}})\Big|_{(\mathbf{U}\cup\D)=(\textbf{o}_2)}.
    \end{equation*}

    Note that the aforementioned transformation of $\mathbf{o}_1$ affects only those realizations of variables that are used for marginalization in the Equation \eqref{eq: phi simplification subcase 3}. Putting the above results together implies that the terms in Equation \eqref{eq: phi simplification subcase 3} remain unchanged, i.e.,
    \begin{align*}
        & \prod_{\hat{p}\in \widetilde{\mathcal{P}}_{U_0}}f_{\hat{p}}(U_0[\hat{p}])
        \prod_{\hat{X}[\hat{p}] \in \mathbf{S}[\widetilde{\mathcal{P}}]} f_{\hat{p}}(\hat{X}[\hat{p}]|\Pa{\hat{X}}{\hat{p}}) \\
        & \times\Big(
        \prod_{X \in \mathbf{S}} P^{\M}(x^{\M} \mid \Pa{X}{\G}) \prod_{U\in \mathbf{U}}\!P^{\M}(u^{\M})\Big)\Big|_{(\mathbf{U}\cup\D)=(\textbf{o}_1)} = \\
        & = \prod_{\hat{p}\in \widetilde{\mathcal{P}}_{U_0}}f_{\hat{p}}(U_0[\hat{p}])
        \prod_{\hat{X}[\hat{p}] \in \mathbf{S}[\widetilde{\mathcal{P}}]} f_{\hat{p}}(\hat{X}[\hat{p}]|\Pa{\hat{X}}{\hat{p}}) \\
        & \times\Big(
        \prod_{X \in \mathbf{S}} P^{\M}(x^{\M} \mid \Pa{X}{\G}) \prod_{U\in \mathbf{U} \setminus \{U_0\}}\!P^{\M}(u^{\M})\Big)\Big|_{(\mathbf{U}\cup\D)=(\textbf{o}_2)}
    \end{align*}
    
    This implies that $\widetilde{\phi}_{l_1}(\mathbf{d}^{\dagger})=\widetilde{\phi}_{l_2}(\mathbf{d}^{\dagger})$. By varying $x$ within $[0 : \frac{\kappa-1}{2}]$ in the definition of $\gamma_{l_1}$ and $\gamma_{l_2}$, we  obtain the result.
\end{proof}

\begin{lemma}
    \label{lemma: eta indices subcase 3}
    There exists $0<\epsilon<\frac{1}{\kappa}$, such that there exists $\mathbf{v}_0 \in \domm{\mathbf{V}}{}$ and $1\leq r <t\leq \frac{\kappa+1}{2}$ such that
    \begin{equation*}
        \widetilde{\eta}_{j_r}(\mathbf{v}_0) \neq \widetilde{\eta}_{j_t}(\mathbf{v}_0).
    \end{equation*}
\end{lemma}
\begin{proof}
    By substituting $P^{\widetilde{\M}}$ from their definitions into Equation \eqref{eq: def eta subcase 3} and rearranging the terms, we obtain
    \begin{equation}\label{eq: eta simplification subcase 3}
    \begin{gathered}
        \widetilde{\eta}_{j}(\mathbf{v}_0) = 
        \sum_{\mathbf{U}^{\M}\setminus \{U_0\}^{\M}}
        \sum_{\mathbf{U}[\widetilde{\mathcal{P}}]}
        \prod_{\hat{U}[\hat{p}]\in \mathbf{U}[\widetilde{\mathcal{P}}]} f_{\hat{p}}(\hat{U})
        \prod_{\hat{X}[\hat{p}] \in \mathbf{S}[\widetilde{\mathcal{P}}]} f_{\hat{p}}(\hat{X}[\hat{p}]|\Pa{\hat{X}}{\hat{p}}) \times\\
        \times\Big(
         \prod_{X \in \mathbf{S}} P^{\M}(x^{\M} \mid \Pa{X}{\G}) \prod_{U\in \mathbf{U} \setminus \{U_0\}}\!P^{\M}(u^{\M})\Big),
    \end{gathered}
    \end{equation}

    Next, we define $\mathbf{v}_0\in \domm{\mathbf{V}}{}$ such that the conditions in the lemma hold.
    
    \begin{itemize}
        \item For any path $\hat{p}\in \widetilde{\mathcal{P}}$ and any node $W$ on the path $\hat{p}$ that is not a starting node for path $\hat{p}$ we define:
        \begin{equation*}
            \mathbf{v}_0[W[\hat{p}]]:=0;
        \end{equation*}
        \item For any variable $S \in \S$, we define
        \begin{equation*}
            \mathbf{v}_0^{\M}[S] := 0;
        \end{equation*}
        \item For the remaining part of $\mathbf{v}_0$, we choose a realization such that for the selected $\mathbf{v}_0$, there exists a realization for the unobserved variables $\mathbf{U}$ that ensures $\mathbb{I}(I)(S)=0$ for all $S\in \S$. This is clearly possible due to the definition of $\mathbb{I}(S)$.
    \end{itemize}
    
    Assume $r$ and $t$ are such that $\gamma_{j_r} = (0, 0, \dots, 0)$ and $\gamma_{j_t} = (2, 0, \dots, 0)$.  To finish the proof of the lemma, it is enough to show that $\widetilde{\eta}_{j_r}(\mathbf{v}_0)$ and $\widetilde{\eta}_{j_t}(\mathbf{v}_0)$ are two different polynomial functions of parameter $\epsilon$. We prove that those two polynomials are different by showing that $\widetilde{\eta}_{j_r}(\mathbf{v}_0)\neq\widetilde{\eta}_{j_t}(\mathbf{v}_0)$ for $\epsilon=0$. 
    
    We only need to consider the non-zero terms in Equation \eqref{eq: eta simplification subcase 3}. From \eqref{eq: eta simplification subcase 3}, we have
    \begin{equation}
    \label{eq: main term lemma subcase 3}
    \begin{gathered}
         \prod_{\hat{U}[\hat{p}]\in \mathbf{U}[\widetilde{\mathcal{P}}]} f_{\hat{p}}(\hat{U})
        \prod_{\hat{X}[\hat{p}] \in \mathbf{S}[\widetilde{\mathcal{P}}]} f_{\hat{p}}(\hat{X}[\hat{p}]|\Pa{\hat{X}}{\hat{p}})\\
        \times\Big(
         \prod_{X \in \mathbf{S}} P^{\M}(x^{\M} \mid \Pa{X}{\G}) \prod_{U\in \mathbf{U} \setminus \{U_0\}}\!P^{\M}(u^{\M})\Big).
    \end{gathered}
    \end{equation}
    Note that $f_{\hat{p}}(\hat{U})=\frac{1}{\kappa+1}$ and $f_{\hat{p}}(\hat{X}|\Pa{\hat{X}}{\hat{p}})$ is non-zero only:
    \begin{itemize}
        \item when $\hat{p} = p_T$, $\hat{X}$ is a child of $T$ on the path $p_T$, and
        \begin{equation*}
            \hat{X}[\hat{p}] \equiv T^{\M} \pmod{\kappa + 1}.
        \end{equation*}

        \item when $\hat{p} = p_T$, $\hat{X}=Z'$, and
        \begin{equation*}
            Z'[\hat{p}] \equiv W'[\hat{p}] + W''[\widetilde{p}] \pmod{\kappa + 1},
        \end{equation*}
        where $W'$ is a parent of $Z'$ on the path $p_T$ and $W''$ is a parent of $Z'$ on the path $\widetilde{p}$.
        
        \item when there exists a variable $W\in \mathbf{F}$ such that $\hat{p}'=\widetilde{p}_W$, $\hat{X}$ is a child of $W$ in path $\widetilde{p}_W$, and
        $$
        \hat{X}[\hat{p}] \equiv W[\widetilde{p}] \pmod{\kappa+1}.
        $$
        
        \item when the following holds
        $$
        \hat{X}[\hat{p}] \equiv \sum_{\hat{X}'\in \Pa{\hat{X}}{\hat{p}}\setminus\{W\}}\hat{X}'[\hat{p}] \pmod{\kappa+1}.
        $$
    \end{itemize}
    Similarly, $P^\M(X|\Pa{X}{\G})$ is non-zero
    \begin{itemize}
        \item if $\mathbb{I}(X)=1$ (i.e. $P^\M(X^{\M}|\Pa{X}{\G})=\frac{1}{\kappa+1}$), or
        
        \item if $X^{\M} \equiv M(X) \pmod{\kappa+1}$ for $P^\M(X^{\M}|\Pa{X}{\G})$.
    \end{itemize}

    Let fix a realization $\mathbf{u} \in \domm{\mathbf{U}\setminus\{U_0^{M}\}}{}$. We consider two scenarios:
    
    \textbf{I)} Assume that for this realization, there is a variable $S\in \S$, such that $\mathbb{I}(S)=1$ and $S$ is the closest variable to $U_0$ considering only paths with bidirected edges in $\G'[\S]$. The value of $S^{\M}$ does not depend on its parents because of $\mathbb{I}(S)=1$ and Equation \eqref{eq: def P(S|Pa(S)) gid}.
    Additionally in the graph $\G'[\S]$ there exists a path $U_0, \hat{S}_1', \hat{U}_1', \hat{S}_2', \hat{U}_2', \dots, \hat{U}_{l'}', S$, such that $U_0$ is a parent of $\hat{S}_1'\in \S$, $S$ is a child of $\hat{U}_l'\in \U^{\S}$, and $\hat{U}_j'\in \U^{\S}$ is a parent of variables $\hat{S}_j'\in \S$ and  $\hat{S}_{j+1}'\in \S$ for $j \in [1:l'-1]$. 
    Let $\hat{\mathbf{U}}':=\{\hat{U}_1', \dots, \hat{U}_l'\}$. 
    We define $\mathbf{u}' \in \domm{\mathbf{U}\setminus\{U_0^{\M}\}}{}$ 
    that is consistent with $\mathbf{u}$ except the variables in $\U$. For these variables, we define
    \begin{equation}
        \begin{split}
            & \mathbf{u}'^{\M}[\hat{U}_j] := \mathbf{u}^{\M}[\hat{U}_j] + 2(-1)^{j} \pmod{\kappa+1}, \quad j\in[1:l'],\\
            % & \mathbf{u}_{2}^{\M}[U] := \mathbf{u}_{1}^{\M}[U], \quad \forall U\in \U^{\S} \setminus \hat{\mathbf{U}},
        \end{split}
    \end{equation}
    
    With this modification for any $\widetilde{S}\in \mathbf{S}$, we have
    \begin{equation*}
    P^{\M}(\widetilde{s}|\Pa{\widetilde{S}}{\G})\Big|_{(\mathbf{U})=(\mathbf{u}, \gamma_{l_1})}=P(\widetilde{s}|\Pa{\widetilde{S}}{\G})\Big|_{(\mathbf{U})=(\textbf{u}', \gamma_{l_2})}.
    \end{equation*}
    Therefore for all such realizations of $\mathbf{u}$ the summation of the following terms for both $\widetilde{\eta}_{j_r}(\mathbf{v}_0)$ and $\widetilde{\eta}_{j_t}(\mathbf{v}_0)$ will be the same,
    \begin{equation}
    \label{eq: main term eta subcase 3}
    \begin{gathered}
         \prod_{\hat{U}[\hat{p}]\in \mathbf{U}[\widetilde{\mathcal{P}}]} f_{\hat{p}}(\hat{U})
        \prod_{\hat{X}[\hat{p}] \in \mathbf{S}[\widetilde{\mathcal{P}}]} f_{\hat{p}}(\hat{X}[\hat{p}]|\Pa{\hat{X}}{\hat{p}}) \times\\
        \times\Big(
         \prod_{X \in \mathbf{S}} P^{\M}(x^{\M} \mid \Pa{X}{\G}) \prod_{U\in \mathbf{U} \setminus \{U_0\}}\!P^{\M}(u^{\M})\Big).
    \end{gathered}
    \end{equation}

    \textbf{II)} Assume that for all $S\in \S$, we have $\mathbb{I}(S)=0$. We consider a realization $U_0^{\M}=\gamma_{j_r}$ and $\mathbf{u}$ such that:
    \begin{itemize}
        \item $\mathbf{u}[\U^{\S}] = \mathbf{0}$.
        \item for all $U\in \mathbf{U}$ and any path $\hat{p}\in \mathcal{P}$ which contains $U$, $\mathbf{u}[U[\hat{p}]] = 0$.
    \end{itemize}
    We claim that for such $\mathbf{u}$,
    \begin{equation*}
    \begin{gathered}
        \prod_{\hat{U}[\hat{p}]\in \mathbf{U}[\widetilde{\mathcal{P}}]} f_{\hat{p}}(\hat{U})
        \prod_{\hat{X}[\hat{p}] \in \mathbf{S}[\widetilde{\mathcal{P}}]} f_{\hat{p}}(\hat{X}[\hat{p}]|\Pa{\hat{X}}{\hat{p}}) \times\\
        \times\Big(
         \prod_{X \in \mathbf{S}} P^{\M}(x^{\M} \mid \Pa{X}{\G}) \prod_{U\in \mathbf{U} \setminus \{U_0\}}\!P^{\M}(u^{\M})\Big).
    \end{gathered}
    \end{equation*}
    is non-zero. To prove this claim we consider 5 cases:
    \begin{itemize}
        \item assume that $\hat{p}=p_T$ and $\hat{X}=Z'$.  Denote by $W'$ parent of $Z'$ on the path $p_T$ and by $W''$ parent of $Z'$ on the path $\widetilde{p}$. From the definition of $\mathbf{u}$ and $\mathbf{v_0}$, we get
        \begin{equation*}
            \hat{X}[\hat{p}] \equiv W'[\hat{p}] + W''[\widetilde{p}],
        \end{equation*}
        and therefore $f_{\hat{p}(\hat{X}[\hat{p}]|\Pa{\hat{X}}{\hat{p}})} = 1$.
        The latter is true because $\hat{X}[\hat{p}]\equiv W'[\hat{p}]\equiv W''[\widetilde{p}]\equiv 0 \pmod{\kappa+1}$.
        
        \item assume that $\hat{p} = p_T$ and $\hat{X}$ is a child of T. From the definition of $\mathbf{u}$ and $\mathbf{v}_0$ we get
        \begin{equation*}
            \hat{X} \equiv T^{\M} \pmod{\kappa+1},
        \end{equation*}
        and therefore $f_{\hat{p}(\hat{X}[\hat{p}]|\Pa{\hat{X}}{\hat{p}})} = 1$. The above holds because all the variables in the above equation are zero.
        
        \item  assume that $\hat{p}\in \widetilde{\mathcal{P}}$ and exists a variable $W$ such that $\hat{p}=\widetilde{p}_W$. Let $\hat{X}$ is a child of $W$ in a path $\widetilde{p}_W$. From the definitions of $\mathbf{u}$ and $\mathbf{v}_0$, we get
        $$
        \hat{X}[\hat{p}] \equiv W[\hat{p}] \pmod{\kappa+1},
        $$ 
        and therefore $f_{\hat{p}}(\hat{X}[\hat{p}]|\Pa{\hat{X}}{\hat{p}})=1$. Again, the above holds because all the terms are zero.
        
        \item assume that $\hat{p}\in \widetilde{\mathcal{P}}$ and $\hat{X}$ is a variable on this path such that it is neither a starting node of the path $\hat{p}$ nor a child of a starting node on the path $\hat{p}$. Then, from the definitions of $\mathbf{v}_0$ and $\mathbf{u}$, we get
        $$
        \hat{X}[\hat{p}] \equiv \sum_{\hat{X}\in \Pa{\hat{X}}{\hat{p}}\setminus\{W\}}\hat{X}[\hat{p}] \pmod{\kappa+1},
        $$
        and therefore $f_{\hat{p}}(\hat{X}[\hat{p}]|\Pa{\hat{X}}{\hat{p}})=1$. Again, the above holds because all the terms are zero.
        
        \item assume $X\in \S$. Then, from the definitions of $\mathbf{v}_0$ and $\mathbf{u}$, we get
        $$
        X^{\M} \equiv M(X) \pmod{\kappa+1},
        $$
        and consequently $P^{\M}(x^{\M}|\Pa{X}{\G})=1$.
    \end{itemize}

    Now we consider the case when $U_0^\M = \gamma_{j_t}$. 
    Note that the following term depends only the realization of $\mathbf{U}^{\M}$ and $\mathbf{v}_0^{\M}$.
    \begin{equation*}
        \prod_{X \in \mathbf{S}} P^{\M}(x^{\M} \mid \Pa{X}{\G}) \prod_{U\in \mathbf{U} \setminus \{U_0\}}\!P^{\M}(u^{\M}).
    \end{equation*}
    However by the proof of Lemma 6 \cite{kivva2022revisiting} we know that there is no realization of $\mathbf{U}^{\M}$ such that:
    \begin{itemize}
        \item $\mathbb{I}(S)=0$ for all $S\in \S$, and
        \item $U_0^{\M}=\gamma_{j_t}$, and
        \item $x^{\M} \equiv M(X) (\kappa + 1)$ for all $X\in \S$. The latter is a necessary condition for $P^{\M}(x|\Pa{X}{\G})$ being non-zero. 
    \end{itemize}

    To summarize, we showed that for $U_0^{\M}=\gamma_{j_r}$, Equation \eqref{eq: main term eta subcase 3} is non-zero while it is zero for $U_0^{\M}=\gamma_{j_t}$. This implies that $\widetilde{\eta}_{j_r}(\mathbf{v}_0)\neq\widetilde{\eta}_{j_t}(\mathbf{v}_0)$ for $\epsilon=0$.
\end{proof}

\subsubsection{Proof of Lemma \ref{lemma: construct models subcase 3}}
\begin{customlem}{\ref{lemma: construct models subcase 3}}
    Let $\mathbf{S}: = \Anc{\Y', \Z'}{\G[\mathbf{V} \setminus \X']}$ and $\D$ is a set of all nodes on the paths in $\mathcal{P}$ excluding $\Z'$. Then,
    \begin{equation}
        P_{\x'}(\widetilde{\dd}|\mathbf{s}\setminus \widetilde{\dd})=\frac{
        Q[\mathbf{S}]
        }{
        \sum_{\D} Q[\mathbf{S}]
        } = Q[\D|\mathbf{S}\setminus\D]
    \end{equation}
    is not c-gID from $(\mathbb{A}, \G)$.
\end{customlem}
\begin{proof}
    We will show  that $Q[\D|\mathbf{S}\setminus\D]$ is not c-gID from $(\mathbb{A}', \G)$, where $\mathbb{A}' := \mathbb{A}\cup\{\mathbf{S}_i\}_{i=2}^{n}$. 
    To this end, we will specify two models $\M_1$ and $\M_2$ such that for each $i \in [0:m']$ and any $\mathbf{v}\in \domm{\mathbf{V}}{}$:
    \begin{align}
        \label{eq: equal known dist subcase 3}
        Q^{\M_1}[\A'_i](\mathbf{v}) &= Q^{\M_2}[\A'_i](\mathbf{v}),\\
        \label{eq: equal denom num subcase 3}
        \sum_{\D}Q^{\M_1}[\mathbf{S}](\mathbf{v}') &= \sum_{\D}Q^{\M_2}[\mathbf{S}](\mathbf{v}'),
    \end{align}
    but there exists $\mathbf{v}_0 \in \domm{\mathbf{V}}{}$ such that:
    \begin{equation}\label{eq: not equal num subcase 3}
        Q^{\M_1}[\mathbf{S}](\mathbf{v}_0) \neq Q^{\M_2}[\mathbf{S}](\mathbf{v}_0).
    \end{equation}
    Note that using Equations (\ref{eq: equal denom num subcase 3})-(\ref{eq: not equal num subcase 2}) yield
    \begin{equation*}
         Q[\D|\mathbf{S}\setminus\D]^{\M_1}(\mathbf{v}_0) \neq Q[\D|\mathbf{S}\setminus\D]^{\M_2}(\mathbf{v}_0).
    \end{equation*}
    This means that $Q[\D|\mathbf{S}\setminus\D]$ is not c-gID from $(\mathbb{A}', \G)$.

    Two this end, we consider two cases.

    \textbf{First case:} \\
    Suppose that there exists $i \in [0, m]$, such that $\widecheck{\mathbf{S}} \subset \mathbf{A}_i$.
    Further we consider models $\widetilde{\M}_1$ and $\widetilde{\M}_2$ constructed in Section \ref{sec: appendix new models subcase 3}. According to the definitions of models $\widetilde{\M}_1$ and $\widetilde{\M}_2$ for any $\mathbf{v}\in \domm{\mathbf{V}}{}$, and any $i\in [0:m']$, and any $g\in\{1, 2\}$:
    \begin{align*}
        & Q[\mathbf{A}'_i]^{\widetilde{\M}_g}(\mathbf{v}) := \sum_{U_0^\M}P^{\M_g}(u_0^{\M})\sum_{U_0[\widetilde{\mathcal{P}}]}\prod_{\hat{p}\in \widetilde{\mathcal{P}}_{U_0}}f_{\hat{p}}(U_0[\hat{p}])\sum_{\mathbf{U} \setminus \{U_0\}} \prod_{X \in \A_i'} P^{\widetilde{\M}}(x \mid \Pa{X}{\G}) \prod_{U\in \mathbf{U} \setminus \{U_0\}} P^{\widetilde{\M}}(u),\\
        \sum_{\D}&Q^{\widetilde{\M}_g}[S](\mathbf{v}) := \sum_{U_0^\M}P^{\M_g}(u_0^{\M})\sum_{U_0[\widetilde{\mathcal{P}}]}\prod_{\hat{p}\in \widetilde{\mathcal{P}}_{U_0}}f_{\hat{p}}(U_0[\hat{p}])\sum_{\D}\sum_{\mathbf{U} \setminus \{U_0\}} \prod_{X \in \mathbf{S}} P^{\widetilde{\M}}(x \mid \Pa{X}{\G}) \prod_{U\in \mathbf{U} \setminus \{U_0\}} P^{\widetilde{\M}}(u),\\
        & Q^{\widetilde{\M}_g}[\mathbf{S}](\mathbf{v}) := \sum_{U_0^\M}P^{\M_g}(u_0^{\M})\sum_{U_0[\widetilde{\mathcal{P}}]}\prod_{\hat{p}\in \widetilde{\mathcal{P}}_{U_0}}f_{\hat{p}}(U_0[\hat{p}])\sum_{\mathbf{U} \setminus \{U_0\}} \prod_{X \in \mathbf{S}} P^{\widetilde{\M}}(x \mid \Pa{X}{\G}) \prod_{U\in \mathbf{U} \setminus \{U_0\}} P^{\widetilde{\M}}(u).
    \end{align*}
    We can re-writing the above equations using the notations of $\widetilde{\theta}_{i, j}$, $\widetilde{\phi}_{j}$, and $\widetilde{\eta}_{j}$,
    \begin{align*}
        & Q[\mathbf{A}'_i]^{\widetilde{\M}_1}(\mathbf{v}) = \sum_{j=1}^d \frac{1}{d}\widetilde{\theta}_{i, j}(\mathbf{v}), \\
        & Q[\mathbf{A}'_i]^{\widetilde{\M}_2}(\mathbf{v}) = \sum_{j=1}^d p_j\widetilde{\theta}_{i, j}(\mathbf{v}), \\
        \sum_{\D} &Q[\mathbf{S}]^{\widetilde{\M}_1}(\mathbf{v}) = \sum_{j=1}^d \frac{1}{d}\widetilde{\phi}_{j}(\mathbf{v}[\D^{\dagger}]), \\
        \sum_{\D} 
        &Q[\mathbf{S}]^{\widetilde{\M}_2}(\mathbf{v}) = \sum_{j=1}^d p_j\widetilde{\phi}_{j}(\mathbf{v}[\D^{\dagger}]), \\
        & Q[\mathbf{S}]^{\widetilde{\M}_1}(\mathbf{v}) = \sum_{j=1}^d \frac{1}{d}\widetilde{\eta}_{j}(\mathbf{v}), \\
        & Q[\mathbf{S}]^{\widetilde{\M}_2}(\mathbf{v}) = \sum_{j=1}^d p_j\widetilde{\eta}_{j}(\mathbf{v}).
    \end{align*}
    The above equations imply the following equations.
    \begin{align*}
        & Q^{\widetilde{M}_2}[\A'_i](\mathbf{v}) - Q^{\widetilde{\M}_1}[\A'_i](\mathbf{v}) = \sum_{j=1}^d (p_j - \frac{1}{d}) \widetilde{\theta}_{i,j}(\mathbf{v})
        \\
        \sum_{\D} & Q[\mathbf{S}]^{\widetilde{\M}_2}(\mathbf{v}) - \sum_{\D} Q[\mathbf{S}]^{\widetilde{\M}_1}(\mathbf{v}) = \sum_{j=1}^d (p_j - \frac{1}{d}) \widetilde{\phi}_{j}(\mathbf{v}[\D^{\dagger}])
        \\
        & Q^{\widetilde{\M}_2}[\mathbf{S}](\mathbf{v}_0) -  Q^{\widetilde{\M}_1}[\mathbf{S}](\mathbf{v}_0) = \sum_{j=1}^d (p_j - \frac{1}{d}) \widetilde{\eta}_{j}(\mathbf{v}_0)
        \\
        & \sum_{j=1}^d p_j - 1 = \sum_{j=1}^d (p_j - \frac{1}{d}).
    \end{align*}

     To prove the statement of the lemma it suffices to solve a following system of linear equations over parameters $\{p_j\}_{j=1}^d$ and show that it admits a solution. 
    \begin{align*}
        & \sum_{j=1}^d (p_j - \frac{1}{d}) \widetilde{\theta}_{i,j}(\mathbf{v}) = 0, \hspace{0.2cm}\forall \mathbf{v} \in \domm{\mathbf{V}}{}, i\in [0:m'],
        \\
        & \sum_{j=1}^d (p_j - \frac{1}{d}) \widetilde{\phi}_{j}(\dd^{\dagger}) = 0, \hspace{0.2cm}\forall \dd^{\dagger} \in \domm{\D^{\dagger}}{}, i\in [0:m'],
        \\
        & \sum_{j=1}^d (p_j - \frac{1}{d}) \widetilde{\eta}_{j}(\mathbf{v}_0) \neq 0, \hspace{0.2cm} \exists \mathbf{v}_0 \in \domm{\mathbf{V}}{},
        \\
        & (p_j - \frac{1}{d}) = 0,
        \\
        & 0<p_j<1, \hspace{0.2cm} \forall j \in [1:d].
    \end{align*}
     Analogous to the proof of Lemma \ref{lemma: construct models subcase 1}, we use Lemmas \ref{lemma: theta indices subcase 3},  \ref{lemma: phi indices subcase 3}, and  \ref{lemma: eta indices subcase 3} instead of Lemmas \ref{lemma: gid equal indices}, \ref{lemma: equal indices summation subcase 1} and \ref{lemma: gid not equal indices} respectively and conclude the result.  

    \textbf{Second case:}\\
    Suppose that there is no $i \in [0, m]$, such that $\S \subset \mathbf{A}_i$. This case we solve exactly the same as the \textbf{Second case} of the Lemma \ref{lemma: construct models subcase 1}.
\end{proof}

\subsection{Proof of Lemma \ref{lemma: eliminate var in cond}}

\renewcommand{\V}{\mathbf{V}}
\renewcommand{\U}{\mathbf{U}}
\begin{customlem}{\ref{lemma: eliminate var in cond}}
     Suppose that $\X$, $\Y$ and $\Z$ are disjoint subsets of $\mathbf{V}$ in graph $\G$ and variables $Z_1 \in \Z$, $Z_2 \in \Y \cup \Z$, such that there is a directed edge from $Z_1$ to $Z_2$ in $\G$. If the causal effect $P_{\x}(\y|\z)$ is not c-gID from $(\mathbb{A}, \G)$, then the causal effect $P_{\x}(\y|\z\setminus\{z_1\})$ is also not c-gID from $(\mathbb{A}, \G)$. 
\end{customlem}
\begin{proof}
By the basic probabilistic manipulations, we get
\begin{equation*}
\begin{split}
    & P_{\x}(\y|\z) = \frac{P_x(\y, \z)}{P_{x}(\z)},\\
    & P_{\x}(\y|\z\setminus\{z_1\}) = \frac{P_x(\y, \z\setminus\{z_1\})}{P_{x}(\z\setminus\{z_1\})}.
\end{split}
\end{equation*}
Using Markov factorization property in graph $\G$, we have
\begin{equation*}
\begin{split}
    & P_{\x}(\y, \z) = 
    \sum_{\V\setminus(\X\cup \Y \cup \Z)} \sum_{\U} \prod_{W \in \V\setminus \X}P(w \mid \Pa{W}{\G}) \prod_{U \in \U} P(u),
    \\
    & P_{\x}(\z) = 
    \sum_{\V\setminus(\X \cup \Z)} \sum_{\U} \prod_{W \in \V\setminus \X}P(w \mid \Pa{W}{\G}) \prod_{U \in \U} P(u).
\end{split}
\end{equation*}
And similarly, we have
\begin{equation}
\begin{split}
    \label{eq: P_x(y, z/z1)}
    & P_{\x}(\y, \z\setminus\{Z_1\}) = 
    \sum_{Z_1}\sum_{\V\setminus(\X\cup \Y \cup \Z)} \sum_{\U} \prod_{W \in \V\setminus \X}P(w \mid \Pa{W}{\G}) \prod_{U \in \U} P(u),
    \\
    & P_{\x}(\z\setminus\{Z_1\}) = 
    \sum_{Z_1}\sum_{\V\setminus(\X \cup \Z)} \sum_{\U} \prod_{W \in \V\setminus \X}P(w \mid \Pa{W}{\G}) \prod_{U \in \U} P(u).
\end{split}
\end{equation}
Since $P_{\x}(\y|\z)$ is not c-gID from $(\mathbb{A}, \G)$, there exists $\M_1$ and $\M_2$ such that
\begin{equation*}
    Q^{\M_1}[\A_i](\mathbf{v}) = Q^{\M_2}[\A_i](\mathbf{v}),\; \forall \mathbf{v}\in \dom{\V}{},\; \forall i \in [0: m],
\end{equation*}
\begin{equation*}
    P_{\x}^{\M_1}(\y|\z)\neq P_{\x}^{\M_2}(\y|\z),\; \exists \x \in \dom{\X}{}, \; \exists \y \in \dom{\Y}{}.
\end{equation*}
Using $\M_1$ and $\M_2$, we construct two models $\M_1'$ and $\M_2'$. To do so, we first take any surjective function $F\!:\: \dom{Z_1}{}\rightarrow \{0, 1\}$ and define a function $\Psi \!:\: \{0, 1\}\times \dom{Z_1}{} \rightarrow (0,1)$ that satisfies $\Psi(0, z_1)+\Psi(1, z_1)=1$ for any $z_1 \in \dom{Z_1}{}$.

For any node $S$ that either belongs to the set of unobserved variables  or belongs to $\V \setminus (\{Z_2\}\cup \Ch{Z_2}{\G})$, we define
\begin{equation*}
    P^{\M_i'}(s|\Pa{S}{\G}) := P^{\M_i}(s|\Pa{S}{\G}), \quad i \in \{1, 2\}.
\end{equation*}
The domain of $Z_2$ in $\M_i'$ is defined as $\dom{Z_2}{}^{\M}\times \{0, 1\}$, where $\dom{Z_2}{}^{\M}$ is the domain of $Z_2$ in $\M$ (either $\M_1$ or $\M_2$). 
For $z_2 \in \dom{Z_2}{}^\M$, $i\in \{1,2\}$, and $k\in \{0, 1\}$, we define
\begin{equation*}
    P^{\M_i'}((z_2, k) 
    \mid \Pa{Z_2}{\G}\setminus\{Z_1\}, z_1) :=
    P^{\M_i}(z_2 \mid \Pa{Z_2}{\G})\times \Psi(F(z_1)\oplus k, z_1).
\end{equation*}
Due to the property of function $\Psi$,  the above definitions are valid probabilities, i.e., for any realizations $(\Pa{Z_2}{\G}, z_1)$, the following holds
\begin{equation*}
    \sum_{k\in \{0,1\}} \sum_{z_2\in \dom{Z_2}{}^{\M}} P^{\M_i'}((z_2, k)|pa(Z_2), z_1) = 1.
\end{equation*}
For each $S \in \Ch{Z_2}{\G}$, we define:
\begin{equation*}
    P^{\M_i'}(s \mid \Pa{S}{\G}\setminus \{Z_2\}, (z_2, k)) :=
    P^{\M_i}(s \mid \Pa{S}{\G}\setminus \{Z_2\}, z_2), \quad i \in \{1, 2\}, k \in \{0, 1\}.
\end{equation*}

Next, we show that $Q^{\M_1'}[\A_i](\mathbf{v}) = Q^{\M_2'}[\A_i](\mathbf{v})$ for each $\mathbf{v}\in \dom{\V}{}$ and $i \in [0:m]$.
Suppose $\mathbf{v}$ is a realization of $\V$ in $\M'_1$ with realizations $z_1$ and $(z_2, k)$ for $Z_1$ and $Z_2$, respectively. 
Consider two cases: 
\begin{itemize}
    \item  $Z_2 \notin \A_i$: In this case, we have
    \begin{align*}
        Q^{\M_1'}[\A_i](\mathbf{v}) 
        &= \sum_{\U}\prod_{A\in \A_i} P^{\M_1'}(a \mid \Pa{A}{\G})\prod_{U\in \U}P^{\M_1'}(u) \\
        & = \sum_{\U} \prod_{A\in \A_i}P^{\M_1}(a \mid \Pa{A}{\G})\prod_{U\in \U} P^{\M_1}(u) 
        = Q^{\M_1}[\A_i](\mathbf{v}) 
        = Q^{\M_2}[\A_i](\mathbf{v}) \\
        & = \sum_{\U} \prod_{A\in \A_i} P^{\M_2}(a \mid \Pa{A}{\G}) \prod_{U\in \U} P^{\M_2}(u) \\
        &= \sum_{\U} \prod_{A\in \A_i} P^{\M_2'}(a \mid \Pa{A}{\G}) \prod_{U\in \U} P^{\M_2'}(u) \\ 
        &= Q^{\M_2'}[\A_i](\mathbf{v}).
    \end{align*}
    
    \item  $Z_2 \in \mathbf{A}_i$: In this case, we have
    \begin{align*}
        Q^{\M_1'}[\A_i](\mathbf{v}) 
        &= \sum_{\U} \prod_{A\in \A_i} P^{\M_1'} (a \mid \Pa{A}{\G})\prod_{U\in \U} P^{\M_1'}(u) \\
        & = \Psi\left(F(z_1)\oplus k, z_1\right) \sum_{\U} \prod_{A\in \A_i} P^{\M_1}(a \mid  \Pa{A}{\G}) \prod_{U\in \U}P^{\M_1}(u) \\
        &= \Psi(F(z_1)\oplus k, z_1) Q^{\M_1}[\A_i](\mathbf{v}) = \Psi(F(z_1)\oplus k, z_1) Q^{\M_2}[\A_i](\mathbf{v}) \\
        &= \Psi(F(z_1)\oplus k, z_1) \sum_{\U} \prod_{A\in \A_i} P^{M_2}(a \mid \Pa{A}{\G}) \prod_{U\in \U} P^{M_2}(u) \\
        &= \sum_{\U}\prod_{A\in \A_i} P^{\M_2'}(a \mid \Pa{A}{\G})) \prod_{U\in \U}P^{\M_2'}(u) \\
        &= Q^{\M_2'}[\A_i](\mathbf{v}).
    \end{align*}
\end{itemize}

Therefore, $Q^{\M_1'}[\A_i](\mathbf{v}) = Q^{\M_2'}[\A_i](\mathbf{v})$ for each $\mathbf{v}\in \dom{\V}{}$ and $i \in [0:m]$.

On the other hand, we know that there exists $\hat{\x} \in \dom{\X}{}^{\M}$, $\hat{\y} \in \dom{\Y}{}^{\M}$ and $\hat{\z} \in \dom{\Z}{}^{\M}$ such that $P^{\M_1}_{\hat{\x}}(\hat{\y}|\hat{\z})\neq  P^{\M_2}_{\hat{\x}}(\hat{\y}|\hat{\z})$.

According to Equations \eqref{eq: P_x(y, z/z1)}, we have
\begin{align*}
    & P_{\x}^{\M_i'}(\y, \z\setminus\{Z_1\}) = 
    \sum_{z_1 \in \dom{Z_1}{}}\sum_{\V\setminus(\X\cup \Y \cup \Z)} \sum_{\U} P^{\M_i'}((z_2, k)|\Pa{Z_2}{\G}) \prod_{W \in \V\setminus (\X\cup \{Z_2\})}P^{\M_i'}(w \mid \Pa{W}{\G}) \prod_{U \in \U} P(u)\\
    & = \sum_{z_1 \in \dom{Z_1}{}}\sum_{\V\setminus(\X\cup \Y \cup \Z)} \sum_{\U} \Psi(F(z_1)\oplus k, z_1)P^{\M_i}(z_2|\Pa{Z_2}{\G}) \prod_{W \in \V\setminus (\X\cup \{Z_2\})}P^{\M_i}(w \mid \Pa{W}{\G}) \prod_{U \in \U} P(u)\\
    & =  \sum_{z_1 \in \dom{Z_1}{}}\Psi(F(z_1)\oplus k, z_1)
    \sum_{\V\setminus(\X\cup \Y \cup \Z)} \sum_{\U} P^{\M_i}(z_2|\Pa{Z_2}{\G}) \prod_{W \in \V\setminus (\X\cup \{Z_2\})}P^{\M_i}(w \mid \Pa{W}{\G}) \prod_{U \in \U} P(u)\\
    & = \sum_{z_1 \in \dom{Z_1}{}}\Psi(F(z_1)\oplus k, z_1) P^{\M_i}_\x(\y, \z).
\end{align*}
Let us denote $\dom{Z_1}{} = \{\alpha_1, \alpha_2, \dots, \alpha_n\}$. For $z_1=\alpha_j$ and $j \in [1:n]$, we also denote
\begin{align*}
    & \psi_{j} := \Psi(F(\alpha_j)\oplus 0, \alpha_j),\\
    & \beta_j^{\M_i}:=P^{\M_i}_{\hat{\x}}(\hat{\y}, \hat{\z}[\Z\setminus\{Z_1\}], \alpha_j) .
\end{align*}
This leads to
\begin{align*}
    & P_{\x}^{\M_i}(\y, \z\setminus\{Z_1\}) = \sum_{j=1}^n\psi_j\beta_j^{\M_i},
\end{align*}
for realizations $\y$ consistent with $\hat{\y}$, realization $\x$ consistent with $\hat{\x}$, $\z$ consistent with $\hat{\z}$, and $Z_2 = (z_2, k)$ consistent with $\hat{\y}\cup\hat{\z}$ and $k=0$. Recall that $\psi_j$ is a real number from the interval $(0, 1)$. 
Note that $\psi_j$ is independent from any other $\psi_l$ for $l\neq j$.

Next, we consider two cases:
\begin{itemize}
    \item Assume that $Z_2\in \Z$. In this case, we have
    \begin{align*}
        & P_{\x}^{\M_i'}(\z\setminus\{Z_1\}) = 
        \sum_{z_1 \in \dom{Z_1}{}}\sum_{\V\setminus(\X \cup \Z)} \sum_{\U} P^{\M_i'}((z_2, k)|\Pa{Z_2}{\G}) \prod_{W \in \V\setminus (\X\cup \{Z_2\})}P^{\M_i'}(w \mid \Pa{W}{\G}) \prod_{U \in \U} P(u)\\
        & = \sum_{z_1 \in \dom{Z_1}{}}\sum_{\V\setminus(\X \cup \Z)} \sum_{\U} \Psi(F(z_1)\oplus k, z_1)P^{\M_i}(z_2|\Pa{Z_2}{\G}) \prod_{W \in \V\setminus (\X\cup \{Z_2\})}P^{\M_i}(w \mid \Pa{W}{\G}) \prod_{U \in \U} P(u)\\
        & =  \sum_{z_1 \in \dom{Z_1}{}}\Psi(F(z_1)\oplus k, z_1)
        \sum_{\V\setminus(\X\cup \Z)} \sum_{\U} P^{\M_1}(z_2|\Pa{Z_2}{\G}) \prod_{W \in \V\setminus (\X\cup \{Z_2\})}P^{\M_i}(w \mid \Pa{W}{\G}) \prod_{U \in \U} P(u)\\
        & = \sum_{z_1 \in \dom{Z_1}{}}\Psi(F(z_1)\oplus k, z_1) P^{\M_i}(\z).
    \end{align*}
    For $j\in [1:n]$ and $z_1 = \alpha_j$, we denote
    \begin{align*}
         & \gamma_j^{\M_i}:=P^{\M_i}_{\hat{\x}}(\hat{\z}[\Z\setminus\{Z_1\}], \alpha_j),
    \end{align*}
    which leads to
    \begin{align*}
        & P_{\x}^{\M_i}(\z\setminus\{Z_1\}) = \sum_{j=1}^n\psi_j\gamma_j^{\M_i},
    \end{align*}
    for realizations $\y$ consistent with $\hat{\y}$, realization $\x$ consistent with $\hat{\x}$, $\z$ consistent with $\hat{\z}$, and $Z_2 = (z_2, k)$ consistent with $\hat{\y}\cup\hat{\z}$ and $k=0$.
    Thus, for such realizations, we have
    $$
    P_{\x}^{\M_i}(\hat{\y}|\hat{\z}\setminus\{Z_1\}) = \frac{\sum_{j=1}^n\psi_j\beta_j^{\M_i}}{\sum_{j=1}^n\psi_j\gamma_j^{\M_i}}.
    $$
    By the assumption of the lemma, there exists $j\in [1:n]$ such that
    \begin{equation*}
        \frac{\beta_j^{\M_1}}{\gamma_{j}^{\M_1}} \neq \frac{\beta_j^{\M_2}}{\gamma_{j}^{\M_2}},
    \end{equation*}
    or equivalently,
    \begin{equation*}
        \beta_j^{\M_1} \gamma_{j}^{\M_2} \neq \beta_j^{\M_2} \gamma_{j}^{\M_1}.
    \end{equation*}
    Without loss of generality, we assume that the aforementioned inequality holds for $j=1$.
    Next, we prove that there exists a parameters $\{\psi_{j}\}_{j=1}^{n}$ such that
    \begin{equation*}
        \frac{\sum_{j=1}^n\psi_j\beta_j^{\M_1}}{\sum_{j=1}^n\psi_j\gamma_j^{\M_1}} \neq \frac{\sum_{j=1}^n\psi_j\beta_j^{\M_2}}{\sum_{j=1}^n\psi_j\gamma_j^{\M_2}},
    \end{equation*}
    or equivalently,
    % \begin{equation*}
    %     \sum_{j=1}^n\psi_j\beta_j^{\M_1} \sum_{j=1}^n\psi_j\gamma_j^{\M_2} \neq \sum_{j=1}^n\psi_j\beta_j^{\M_2} \sum_{j=1}^n\psi_j\gamma_j^{\M_1}
    % \end{equation*}
    % or
    \begin{equation*}
        \sum_{j=1}^n\psi_j\beta_j^{\M_1} \sum_{j=1}^n\psi_j\gamma_j^{\M_2} - \sum_{j=1}^n\psi_j\beta_j^{\M_2} \sum_{j=1}^n\psi_j\gamma_j^{\M_1} \neq 0.
    \end{equation*}
    Note that the left hand side is a quadratic equation with respect to parameter $\psi_1$, e.g.,
    \begin{equation*}
        (\beta_1^{\M_1} \gamma_{1}^{\M_2} - \beta_1^{\M_2} \gamma_{1}^{\M_1})\psi_1^2
    \end{equation*}
    Since $\beta_1^{\M_1} \gamma_{1}^{\M_2} - \beta_1^{\M_2} \gamma_{1}^{\M_1}\neq0$, then we can find  $\{\psi_{j}\}_{j=1}^{n}$, such that
    \begin{equation*}
        \sum_{j=1}^n\psi_j\beta_j^{\M_1} \sum_{j=1}^n\psi_j\gamma_j^{\M_2} - \sum_{j=1}^n\psi_j\beta_j^{\M_2} \sum_{j=1}^n\psi_j\gamma_j^{\M_1} \neq 0.
    \end{equation*}
    This is possible because $\psi_i\in(0,1)$. This concludes the proof of the lemma for this case.

    \item Assume that $Z_2\in \Y$. Suppose that $P_{\x}^{\M_1}(\y, z_1|\z\setminus\{Z_1\}) = P_{\x}^{\M_2}(\y, z_1|\z\setminus\{Z_1\})$ for all $\x\in \dom{\X}{}$, $\y\in \dom{\Y}{}$ and $\z\in \dom{\Z}{}$. Then,
    \begin{equation*}
        P^{\M_1}_{\x}(\y|\z) = \frac{P^{\M_1}(\y, z_1|\z\setminus\{Z_1\})}{P^{\M_1}_{\x}(z_1|\z\setminus\{Z_1\})} = \frac{P^{\M_1}(\y, z_1|\z\setminus\{Z_1\})}{P^{\M_2}_{\x}(z_1|\z\setminus\{Z_1\})} = P^{\M_2}_{\x}(\y|\z).
    \end{equation*}
    This is impossible as $P^{\M_1}_{\hat{\x}}(\hat{\y}|\hat{\z})\neq P^{\M_2}_{\hat{\x}}(\hat{\y}|\hat{\z})$. 
    Thus, there exist $\hat{x}' \in \dom{\X}{} $, $\hat{y}' \in \dom{\Y}{}$, and $\hat{z}' \in \dom{\Z}{}$, such that
    \begin{equation*}
        P_{\hat{\x}'}^{\M_1}(\hat{\y}', \hat{z}_1'|\hat{\z}'\setminus\{Z_1\}) \neq P_{\hat{\x}'}^{\M_2}(\hat{\y}', \hat{z}'_1|\hat{\z}'\setminus\{Z_1\}).
    \end{equation*}
    On the other hand, we have 
    \begin{align*}
        & P_{\x}^{\M_i'}(\y, \z) = 
       \sum_{\V\setminus(\X\cup \Y \cup \Z)} \sum_{\U} P^{\M_i'}((z_2, k)|\Pa{Z_2}{\G}) \prod_{W \in \V\setminus (\X\cup \{Z_2\})}P^{\M_i'}(w \mid \Pa{W}{\G}) \prod_{U \in \U} P(u)\\
        & = \sum_{\V\setminus(\X\cup \Y \cup \Z)} \sum_{\U} \Psi(F(z_1)\oplus k, z_1)P^{\M_i}(z_2|\Pa{Z_2}{\G}) \prod_{W \in \V\setminus (\X\cup \{Z_2\})}P^{\M_i}(w \mid \Pa{W}{\G}) \prod_{U \in \U} P(u)\\
        & = \Psi(F(z_1)\oplus k, z_1)
        \sum_{\V\setminus(\X\cup \Y \cup \Z)} \sum_{\U} P^{\M_i}(z_2|\Pa{Z_2}{\G}) \prod_{W \in \V\setminus (\X\cup \{Z_2\})}P^{\M_i}(w \mid \Pa{W}{\G}) \prod_{U \in \U} P(u)\\
        & =\Psi(F(z_1)\oplus k, z_1) P^{\M_i}_\x(\y, \z).
    \end{align*}
    For $j \in [1:n]$, we define
    \begin{equation*}
        {\beta_j'}^{\M_i}:=P^{\M_i}_{\hat{\x}}(\hat{\y}', \hat{\z}'
        [\Z\setminus\{Z_1\}], \alpha_j).
    \end{equation*}
    Suppose $m \in [1:n]$, where $z_1=\alpha_m$ and it is consistent with $\hat{\z}'$. We assign $k=0$ and denote
    \begin{equation*}
        {\beta_m'}^{\M_i}:=P_{\hat{x}'}^{\M_i}(\hat{\y}',\hat{\z}').
    \end{equation*}
    This results in
    \begin{equation*}
        P_{\mathbf{\hat{x}}'}^{\M_i'}(\hat{\y}', \hat{\z}') = {\beta_m'}^{\M_i}\psi_m.
    \end{equation*}
    We also have
    \begin{align*}
        & P_{\x}^{\M_i'}(\z\setminus\{Z_1\}) = 
        \sum_{z_1 \in \dom{Z_1}{}}\sum_{\V\setminus(\X \cup \Z)} \sum_{\U} P^{\M_i'}((z_2, k)|\Pa{Z_2}{\G}) \prod_{W \in \V\setminus (\X\cup \{Z_2\})}P^{\M_i'}(w \mid \Pa{W}{\G}) \prod_{U \in \U} P(u)\\
        & = \sum_{z_1 \in \dom{Z_1}{}}\sum_{\V\setminus(\X \cup \Z)} \sum_{\U} \Psi(F(z_1)\oplus k, z_1)P^{\M_i}(z_2|\Pa{Z_2}{\G}) \prod_{W \in \V\setminus (\X\cup \{Z_2\})}P^{\M_i}(w \mid \Pa{W}{\G}) \prod_{U \in \U} P(u)\\
        & =  \sum_{z_1 \in \dom{Z_1}{}}
        \sum_{\V\setminus(\X\cup \Z)} \sum_{\U} \sum_{k \in \{0, 1\}}\Psi(F(z_1)\oplus k, z_1) P^{\M_i}(z_2|\Pa{Z_2}{\G}) \prod_{W \in \V\setminus (\X\cup \{Z_2\})}P^{\M_i}(w \mid \Pa{W}{\G}) \prod_{U \in \U} P(u)\\
        & = \sum_{z_1 \in \dom{Z_1}{}} P^{\M_i}(\z).
    \end{align*}
    For $j\in [1:n]$ and $z_1 = \alpha_j$, we denote
    \begin{align*}
         & {\gamma'_j}^{\M_i}:=P^{\M_i}_{\hat{\x}}(\hat{\z}'[\Z\setminus\{Z_1\}]),
    \end{align*}
    and from the above equation, we get
    \begin{align*}
        & P_{\x}^{\M_i}(\z\setminus\{Z_1\}) = \sum_{j=1}^n\gamma_j^{\M_i},
    \end{align*}
    for realizations $\y$ consistent with $\hat{\y}'$, realization $\x$ consistent with $\hat{\x}'$, $\z$ consistent with $\hat{\z}'$, and $Z_2 = (z_2, k)$ consistent with $\hat{\y}'\cup\hat{\z}'$ and $k=0$. 
    We have 
    $$
    P_{\x}^{\M_i}(\hat{\y}|\hat{\z}\setminus\{Z_1\}) = \frac{\sum_{j=1}^n\psi_j{\beta_j'}^{\M_i}}{\sum_{j=1}^n\gamma_j^{\M_i}}.
    $$

    By the assumption of the lemma, we have
    \begin{equation*}
        \frac{{\beta_m'}^{\M_1}}{\sum_{j=1}^n\gamma_j^{\M_1}} \neq \frac{{\beta_m'}^{\M_2}}{\sum_{j=1}^n\gamma_j^{\M_2}}.
    \end{equation*}
    Next, we prove that there exists a set of parameters $\{\psi_{j}\}_{j=1}^{n}$, such that
    \begin{equation*}
        \frac{\sum_{j=1}^n\psi_j{\beta_j'}^{\M_1}}{\sum_{j=1}^n\gamma_j^{\M_1}} \neq \frac{\sum_{j=1}^n\psi_j{\beta_j'}^{\M_2}}{\sum_{j=1}^n\gamma_j^{\M_2}}
    \end{equation*}
    or equivalently,
    \begin{equation*}
        \frac{\sum_{j=1}^n\psi_j{\beta_j'}^{\M_1}}{\sum_{j=1}^n\gamma_j^{\M_1}} - \frac{\sum_{j=1}^n\psi_j{\beta_j'}^{\M_2}}{\sum_{j=1}^n\gamma_j^{\M_2}} \neq 0.
    \end{equation*}
    Note that left hand side of the above equation is  linear with respect to parameter $\psi_m$ with the following coefficient,
    \begin{equation*}
        \frac{{\beta_m'}^{\M_1}}{\sum_{j=1}^n\gamma_j^{\M_1}} - \frac{{\beta_m'}^{\M_2}}{\sum_{j=1}^n\gamma_j^{\M_2}}\neq 0.
    \end{equation*}
    This ensures that we can can find a realization of $\{\psi_{j}\}_{j=1}^{n}$, such that
    \begin{equation*}
        \frac{\sum_{j=1}^n\psi_j{\beta_j'}^{\M_1}}{\sum_{j=1}^n\gamma_j^{\M_1}} - \frac{\sum_{j=1}^n\psi_j{\beta_j'}^{\M_2}}{\sum_{j=1}^n\gamma_j^{\M_2}} \neq 0.
    \end{equation*}
    This concludes the proof of the lemma for second case.
\end{itemize}

\end{proof}

%====================================================================================%

\subsection{Proof of Lemma \ref{lemma: for_the_main}}

\begin{customlem}{\ref{lemma: for_the_main}}
Suppose that $\X$, $\Y$ and $\Z$ are disjoint subsets of $\mathbf{V}$ in graph $\G$ and variables $Z_1 \in \Y$, $Z_2 \in \Y \cup \Z$, such that there is a directed edge from $Z_1$ to $Z_2$ in $\G$. If the causal effect $P_{\x}(\y|\z)$ is not c-gID from $(\mathbb{A}, \G)$, then the causal effect $P_{\x}(\y\setminus\{z_1\}|\z)$ is also not c-gID from $(\mathbb{A}, \G)$. 
\end{customlem}
\begin{proof}
By the basic probabilistic manipulations, we get
\begin{equation*}
\begin{split}
    & P_{\x}(\y|\z) = \frac{P_x(\y, \z)}{P_{x}(\z)},\\
    & P_{\x}(\y\setminus\{z_1\}|\z) = \frac{P_x(\y\setminus\{z_1\}, \z)}{P_{x}(\z)}.
\end{split}
\end{equation*}
Using Markov factorization property in graph $\G$, $P_{\mathbf{x}}(\mathbf{y})$ will be
\begin{equation*}
\begin{split}
    & P_{\x}(\y, \z) = 
    \sum_{\V\setminus(\X\cup \Y \cup \Z)} \sum_{\U} \prod_{W \in \V\setminus \X}P(w \mid \Pa{W}{\G}) \prod_{U \in \U} P(u),
    \\
    & P_{\x}(\z) = 
    \sum_{\V\setminus(\X \cup \Z)} \sum_{\U} \prod_{W \in \V\setminus \X}P(w \mid \Pa{W}{\G}) \prod_{U \in \U} P(u).
\end{split}
\end{equation*}
And similarly, we have
\begin{equation}
\begin{split}
    \label{eq: P_x(y/z1, z)}
    & P_{\x}(\y\setminus\{Z_1\}, \z) = 
    \sum_{Z_1}\sum_{\V\setminus(\X\cup \Y \cup \Z)} \sum_{\U} \prod_{W \in \V\setminus \X}P(w \mid \Pa{W}{\G}) \prod_{U \in \U} P(u),
    \\
    & P_{\x}(\z) = 
    \sum_{Z_1}\sum_{\V\setminus(\X \cup \Z)} \sum_{\U} \prod_{W \in \V\setminus \X}P(w \mid \Pa{W}{\G}) \prod_{U \in \U} P(u).
\end{split}
\end{equation}
Since $P_{\x}(\y|\z)$ is not gID from $(\mathbb{A}, \G)$, there exists $\M_1$ and $\M_2$ such that
\begin{equation*}
    Q^{\M_1}[\A_i](\mathbf{v}) = Q^{\M_2}[\A_i](\mathbf{v}),\; \forall \mathbf{v}\in \dom{\V}{},\; \forall i \in [0: m],
\end{equation*}
\begin{equation*}
    P_{\x}^{\M_1}(\y|\z)\neq P_{\x}^{\M_2}(\y|\z),\; \exists \x \in \dom{\X}{}, \; \exists \y \in \dom{\Y}{}.
\end{equation*}
Using $\M_1$ and $\M_2$, we construct two models $\M_1'$ and $\M_2'$. 
Define a surjective function $F\!:\: \dom{Z_1}{}\rightarrow \{0, 1\}$ and a function $\Psi \!:\: \{0, 1\}\times \dom{Z_1}{} \rightarrow (0,1)$ such that $\Psi(0, z_1)+\Psi(1, z_1)=1$ for each $z_1 \in \dom{Z_1}{}$.

For any node $S$ which is either unobserved or in $\V \setminus (\{Z_2\}\cup \Ch{Z_2}{\G})$, we define
\begin{equation*}
    P^{\M_i'}(s|\Pa{S}{\G}) = P^{\M_i}(s|\Pa{S}{\G}),
\end{equation*}
where $i \in \{1, 2\}$.
The domain of $Z_2$ in $\M_i'$ is defined as $\dom{Z_2}{}^{\M}\times \{0, 1\}$, where $\dom{Z_2}{}^{\M}$ is the domain of $Z_2$ in $\M$ (either $\M_1$ or $\M_2$). 
For $z_2 \in \dom{Z_2}{}^\M$, $i\in \{0,1\}$, and $k\in \{0, 1\}$, we define
\begin{equation*}
    P^{\M_i'}((z_2, k) 
    \mid \Pa{Z_2}{\G}\setminus\{Z_1\}, z_1) =
    P^{\M_i}(z_2 \mid \Pa{Z_2}{\G}) \Psi(F(z_1)\oplus k, z_1).
\end{equation*}
Moreover, for a fixed realization $(\Pa{Z_2}{\G}, z_1)$, we have
\begin{equation*}
    \sum_{k\in \{0,1\}} \sum_{z_2\in \dom{Z_2}{}^{\M}} P^{\M_i'}((z_2, k)|pa(Z_2), z_1) = 1.
\end{equation*}

For each $S \in \Ch{Z_2}{\G}$, we define:
\begin{equation*}
    P^{\M_i'}(s \mid \Pa{S}{\G}\setminus \{Z_2\}, (z_2, k)) =
    P^{\M_i}(s \mid \Pa{S}{\G}\setminus \{Z_2\}, z_2).
\end{equation*}

Next, we show that $Q^{\M_1'}[\A_i](\mathbf{v}) = Q^{\M_2'}[\A_i](\mathbf{v})$ for each $\mathbf{v}\in \dom{\V}{}$ and $i \in [0:m]$.
Suppose $\mathbf{v}$ is a realization of $\V$ in $\M'_1$ with realizations $z_1$ and $(z_2, k)$ for $Z_1$ and $Z_2$, respectively. 
Consider two cases: 
\begin{itemize}
    \item  $Z_2 \notin \A_i$: In this case, we have
    \begin{align*}
        Q^{\M_1'}[\A_i](\mathbf{v}) 
        &= \sum_{\U}\prod_{A\in \A_i} P^{\M_1'}(a \mid \Pa{A}{\G})\prod_{U\in \U}P^{\M_1'}(u) \\
        & = \sum_{\U} \prod_{A\in \A_i}P^{\M_1}(a \mid \Pa{A}{\G})\prod_{U\in \U} P^{\M_1}(u) 
        = Q^{\M_1}[\A_i](\mathbf{v}) 
        = Q^{\M_2}[\A_i](\mathbf{v}) \\
        & = \sum_{\U} \prod_{A\in \A_i} P^{\M_2}(a \mid \Pa{A}{\G}) \prod_{U\in \U} P^{\M_2}(u) \\
        &= \sum_{\U} \prod_{A\in \A_i} P^{\M_2'}(a \mid \Pa{A}{\G}) \prod_{U\in \U} P^{\M_2'}(u) \\ 
        &= Q^{\M_2'}[\A_i](\mathbf{v}).
    \end{align*}
    
    \item  $Z_2 \in \mathbf{A}_i$: In this case, we have
    \begin{align*}
        Q^{\M_1'}[\A_i](\mathbf{v}) 
        &= \sum_{\U} \prod_{A\in \A_i} P^{\M_1'} (a \mid \Pa{A}{\G})\prod_{U\in \U} P^{\M_1'}(u) \\
        & = \Psi\left(F(z_1)\oplus k, z_1\right) \sum_{\U} \prod_{A\in \A_i} P^{\M_1}(a \mid  \Pa{A}{\G}) \prod_{U\in \U}P^{\M_1}(u) \\
        &= \Psi(F(z_1)\oplus k, z_1) Q^{\M_1}[\A_i](\mathbf{v}) = \Psi(F(z_1)\oplus k, z_1) Q^{\M_2}[\A_i](\mathbf{v}) \\
        &= \Psi(F(z_1)\oplus k, z_1) \sum_{\U} \prod_{A\in \A_i} P^{M_2}(a \mid \Pa{A}{\G}) \prod_{U\in \U} P^{M_2}(u) \\
        &= \sum_{\U}\prod_{A\in \A_i} P^{\M_2'}(a \mid \Pa{A}{\G})) \prod_{U\in \U}P^{\M_2'}(u) \\
        &= Q^{\M_2'}[\A_i](\mathbf{v}).
    \end{align*}
\end{itemize}

Therefore, $Q^{\M_1'}[\A_i](\mathbf{v}) = Q^{\M_2'}[\A_i](\mathbf{v})$ for each $\mathbf{v}\in \dom{\V}{}$ and $i \in [0:m]$.

On the other hand, we know that there exists $\hat{\x} \in \dom{\X}{}^{\M}$, $\hat{\y} \in \dom{\Y}{}^{\M}$ and $\hat{\z} \in \dom{\Z}{}^{\M}$ such that $P^{\M_1}_{\hat{\x}}(\hat{\y}|\hat{\z})\neq  P^{\M_2}_{\hat{\x}}(\hat{\y}|\hat{\z})$. 

According to Equations \eqref{eq: P_x(y/z1, z)}, we have
\begin{align*}
    & P_{\x}^{\M_i'}(\y\setminus\{Z_1\}, \z) = 
    \sum_{z_1 \in \dom{Z_1}{}}\sum_{\V\setminus(\X\cup \Y \cup \Z)} \sum_{\U} P^{\M_i'}((z_2, k)|\Pa{Z_2}{\G}) \prod_{W \in \V\setminus (\X\cup \{Z_2\})}P^{\M_i'}(w \mid \Pa{W}{\G}) \prod_{U \in \U} P(u)\\
    & = \sum_{z_1 \in \dom{Z_1}{}}\sum_{\V\setminus(\X\cup \Y \cup \Z)} \sum_{\U} \Psi(F(z_1)\oplus k, z_1)P^{\M_i}(z_2|\Pa{Z_2}{\G}) \prod_{W \in \V\setminus (\X\cup \{Z_2\})}P^{\M_i}(w \mid \Pa{W}{\G}) \prod_{U \in \U} P(u)\\
    & =  \sum_{z_1 \in \dom{Z_1}{}}\Psi(F(z_1)\oplus k, z_1)
    \sum_{\V\setminus(\X\cup \Y \cup \Z)} \sum_{\U} P^{\M_i}(z_2|\Pa{Z_2}{\G}) \prod_{W \in \V\setminus (\X\cup \{Z_2\})}P^{\M_i}(w \mid \Pa{W}{\G}) \prod_{U \in \U} P(u)\\
    & = \sum_{z_1 \in \dom{Z_1}{}}\Psi(F(z_1)\oplus k, z_1) P^{\M_i}_\x(\y, \z).
\end{align*}
Let us denote $\dom{Z_1}{} = \{\alpha_1, \alpha_2, \dots, \alpha_n\}$. For $z_1=\alpha_j$ and $j \in [1:n]$, we also denote
\begin{align*}
    & \psi_{j} = \psi(F(\alpha_j)\oplus 0, \alpha_j),\\
    & P^{\M_i}_{\hat{\x}}(\hat{\y}[\Y\setminus\{Z_1\}], \hat{\z}, \alpha_j) = \beta_j^{\M_i}.
\end{align*}
For $Z_2=(\hat{z}[Z_2], 0)$ we have:
\begin{align*}
    & P_{\hat{\x}}^{\M_i}(\hat{\y}\setminus\{Z_1\}, \hat{\z}) = \sum_{j=1}^n\psi_j\beta_j^{\M_i}.
\end{align*}
Recall that $\psi_j$ is a real number from the interval $(0, 1)$. 
Note that $\psi_j$ is independent from any other $\psi_l$ for $l\neq j$.

Next, we consider two cases:
\begin{itemize}
    \item Assume that $Z_2\in \Z$. In this case, we have
    \begin{align*}
        & P_{\x}^{\M_i'}(\z) = 
        \sum_{z_1 \in \dom{Z_1}{}}\sum_{\V\setminus(\X \cup \Z)} \sum_{\U} P^{\M_i'}((z_2, k)|\Pa{Z_2}{\G}) \prod_{W \in \V\setminus (\X\cup \{Z_2\})}P^{\M_i'}(w \mid \Pa{W}{\G}) \prod_{U \in \U} P(u)\\
        & = \sum_{z_1 \in \dom{Z_1}{}}\sum_{\V\setminus(\X \cup \Z)} \sum_{\U} \Psi(F(z_1)\oplus k, z_1)P^{\M_i}(z_2|\Pa{Z_2}{\G}) \prod_{W \in \V\setminus (\X\cup \{Z_2\})}P^{\M_i}(w \mid \Pa{W}{\G}) \prod_{U \in \U} P(u)\\
        & =  \sum_{z_1 \in \dom{Z_1}{}}\Psi(F(z_1)\oplus k, z_1)
        \sum_{\V\setminus(\X\cup \Z)} \sum_{\U} P^{\M_1}(z_2|\Pa{Z_2}{\G}) \prod_{W \in \V\setminus (\X\cup \{Z_2\})}P^{\M_i}(w \mid \Pa{W}{\G}) \prod_{U \in \U} P(u)\\
        & = \sum_{z_1 \in \dom{Z_1}{}}\Psi(F(z_1)\oplus k, z_1) P^{\M_i}(\z).
    \end{align*}
    We denote
    \begin{align*}
         & P^{\M_i}_{\hat{\x}}(\hat{\z}) = \gamma^{\M_i},
    \end{align*}
    which leads to
    \begin{align*}
        & P_{\x}^{\M_i}(\z) = \sum_{j=1}^n\psi_j\gamma^{\M_i},
    \end{align*}
    for $Z_2=(\hat{\z}[Z_2], 0)$.
    Thus, 
    $$
    P_{\x}^{\M_i}(\hat{\y}\setminus\{Z_1\}|\hat{\z}) = \frac{\sum_{j=1}^n\psi_j\beta_j^{\M_i}}{\sum_{j=1}^n\psi_j\gamma^{\M_i}}.
    $$
    By the assumption of the lemma, there exists $j\in [1:n]$ such that
    \begin{equation*}
        \frac{\beta_j^{\M_1}}{\gamma^{\M_1}} \neq \frac{\beta_j^{\M_2}}{\gamma^{\M_2}},
    \end{equation*}
    or equivalently,
    \begin{equation*}
        \beta_j^{\M_1} \gamma^{\M_2} \neq \beta_j^{\M_2} \gamma^{\M_1}.
    \end{equation*}
    Without loss of generality, we assume that the aforementioned inequality holds for $j=1$.
    Next, we prove that there exists a parameters $\{\psi_{j}\}_{j=1}^{n}$ such that
    \begin{equation*}
        \frac{\sum_{j=1}^n\psi_j\beta_j^{\M_1}}{\sum_{j=1}^n\psi_j\gamma^{\M_1}} \neq \frac{\sum_{j=1}^n\psi_j\beta_j^{\M_2}}{\sum_{j=1}^n\psi_j\gamma^{\M_2}},
    \end{equation*}
    or equivalently,
    % \begin{equation*}
    %     \sum_{j=1}^n\psi_j\beta_j^{\M_1} \sum_{j=1}^n\psi_j\gamma_j^{\M_2} \neq \sum_{j=1}^n\psi_j\beta_j^{\M_2} \sum_{j=1}^n\psi_j\gamma_j^{\M_1}
    % \end{equation*}
    % or
    \begin{equation*}
        \sum_{j=1}^n\psi_j\beta_j^{\M_1} \sum_{j=1}^n\psi_j\gamma^{\M_2} - \sum_{j=1}^n\psi_j\beta_j^{\M_2} \sum_{j=1}^n\psi_j\gamma^{\M_1} \neq 0.
    \end{equation*}
    Note that the left hand side is a quadratic equation with parameter $\psi_1$ that contains the following term
    \begin{equation*}
        (\beta_1^{\M_1} \gamma^{\M_2} - \beta_1^{\M_2} \gamma^{\M_1})\psi_1^2
    \end{equation*}
    Since $\beta_1^{\M_1} \gamma^{\M_2} - \beta_1^{\M_2} \gamma^{\M_1}\neq0$, then we can can find realization of $\{\psi_{j}\}_{j=1}^{n}$, such that
    \begin{equation*}
        \sum_{j=1}^n\psi_j\beta_j^{\M_1} \sum_{j=1}^n\psi_j\gamma^{\M_2} - \sum_{j=1}^n\psi_j\beta_j^{\M_2} \sum_{j=1}^n\psi_j\gamma^{\M_1} \neq 0.
    \end{equation*}
    which concludes the proof of the lemma for this case.

    \item Assume that $Z_2\in \Y$. In this case we have:
    \begin{align*}
        & P_{\x}^{\M_i'}(\z) = 
       \sum_{\V\setminus(\X \cup \Z)} \sum_{\U} P^{\M_i'}((z_2, k)|\Pa{Z_2}{\G}) \prod_{W \in \V\setminus (\X\cup \{Z_2\})}P^{\M_i'}(w \mid \Pa{W}{\G}) \prod_{U \in \U} P(u)\\
        & = \sum_{\V\setminus(\X \cup \Z)} \sum_{\U} \Psi(F(z_1)\oplus k, z_1)P^{\M_i}(z_2|\Pa{Z_2}{\G}) \prod_{W \in \V\setminus (\X\cup \{Z_2\})}P^{\M_i}(w \mid \Pa{W}{\G}) \prod_{U \in \U} P(u)\\
        & =  
        \sum_{\V\setminus(\X\cup \Z)} \sum_{\U} \sum_{k \in \{0, 1\}}\Psi(F(z_1)\oplus k, z_1) P^{\M_i}(z_2|\Pa{Z_2}{\G}) \prod_{W \in \V\setminus (\X\cup \{Z_2\})}P^{\M_i}(w \mid \Pa{W}{\G}) \prod_{U \in \U} P(u)\\
        & = P^{\M_i}(\z).
    \end{align*}
    We denote
    \begin{align*}
         & P^{\M_i}_{\hat{\x}}(\hat{\z}) = {\gamma}^{\M_i}.
    \end{align*}
    Thus, 
    $$
    P_{\x}^{\M_i}(\hat{\y}\setminus\{Z_1\}|\hat{\z}) = \frac{\sum_{j=1}^n\psi_j{\beta_j}^{\M_i}}{\gamma^{\M_i}}.
    $$

    By the assumption of the lemma exist $m \in [1:n]$ such that
    \begin{equation*}
        \frac{{\beta_m}^{\M_1}}{\gamma_j^{\M_1}} \neq \frac{{\beta_m}^{\M_2}}{\gamma^{\M_2}}.
    \end{equation*}
    Next, we prove that there exists a set of parameters $\{\psi_{j}\}_{j=1}^{n}$, such that
    \begin{equation*}
        \frac{\sum_{j=1}^n\psi_j{\beta_j}^{\M_1}}{\gamma^{\M_1}} \neq \frac{\sum_{j=1}^n\psi_j{\beta_j}^{\M_2}}{\gamma^{\M_2}}
    \end{equation*}
    or equivalently,
    \begin{equation*}
        \frac{\sum_{j=1}^n\psi_j{\beta_j'}^{\M_1}}{\gamma^{\M_1}} - \frac{\sum_{j=1}^n\psi_j{\beta_j'}^{\M_2}}{\gamma^{\M_2}} \neq 0.
    \end{equation*}
    Note that left hand side of the above equation is  linear with respect to parameter $\psi_m$ with the following coefficient,
    \begin{equation*}
        \frac{{\beta_m'}^{\M_1}}{\gamma^{\M_1}} - \frac{{\beta_m'}^{\M_2}}{\gamma^{\M_2}}\neq 0.
    \end{equation*}
    This ensures that we can can find a realization of $\{\psi_{j}\}_{j=1}^{n}$, such that
    \begin{equation*}
        \frac{\sum_{j=1}^n\psi_j{\beta_j'}^{\M_1}}{\sum_{j=1}^n\gamma^{\M_1}} - \frac{\sum_{j=1}^n\psi_j{\beta_j'}^{\M_2}}{\sum_{j=1}^n\gamma^{\M_2}} \neq 0.
    \end{equation*}
    This concludes the proof of the lemma for the second case.
\end{itemize}
    
\end{proof}

\section{On the positivity assumption in the literature} \label{sec: app B}
 As it was pointed out in \cite{kivva2022revisiting},  positivity assumption is crucial for proving the completeness part. More precisely, the completeness of an algorithm means that if the algorithm does not compute a given conditional causal effect, then it cannot be computed uniquely by any other algorithms. To prove the completeness, two models $\mathcal{M}_1$ and $\M_2$ are constructed such that they are both positive and induce the same set of distributions as the ones given in the problem statement, i.e., $Q[\A_0]$, $Q[\A_1]$, \dots, $Q[\A_m]$
, but they result in different values for the conditional causal effect of interest, i.e., $P^{\M_1}_{\mathbf{x}}(\mathbf{y}\mid\mathbf{z})\neq P^{\M_2}_{\mathbf{x}}(\mathbf{y}\mid\mathbf{z})$. Hence, $P_{\mathbf{x}}(\mathbf{y}\mid \mathbf{z})$
 cannot be uniquely computed . 
 
 In Lee et al. [2019, 2020] and Correal et al. [2021] for the completeness part authors constructed such models $\M_1$ and $\M_2$, but the models violate the positivity assumption. That is, it is possible to have examples in which a given causal effect is identifiable under the positivity assumption while it is possible to construct two non-positive models that show the causal effect is not identifiable (Kivva et al. [2022]). Violation of positivity assumption renders some distributions ill-defined (conditioning on zero-probability events). That is why computing a causal effect in the classical setting with do-calculus implicitly contains steps in which we can cancel out a distribution (e.g., $Q$) that appears on both side of an equality, i.e., $P_1\cdot Q = P_2\cdot Q \Rightarrow P_1=P_2$. Clearly, this is only possible when $Q>0$. If positivity is violated, then such steps in computing a causal effect cannot be used.

\subsection{General Transportability} 
The work of \cite{lee2020general} proves the completeness part of the c-gID problem by constructing two models that agree on the observable distributions and disagree on the target causal effect. Those models does not satisfy the positivity assumption by the construction. A similar flaw existed in the proof of \cite{lee2019general}, which was  specified in details later by \cite{kivva2022revisiting}.
Given that \cite{lee2020general} does not discuss whether their models can be transformed into positive ones.

For further details, we refer to the technical report of \cite{lee2020general}, which contains the proofs.

\textit{Parametrizations for an s-Thicket:} According to the appendix of \cite{lee2020general}, the models in Lemma 3 which is one of the main Lemmas for proving the completeness result are based on the ones in \cite{lee2019general}. These models violate the positivity assumption according to \cite{kivva2022revisiting} and should be substituted with a fixed ones.

\textit{Parametrization for an Extended s-Thicket:} According to Eq. (5) and (6) in \cite{lee2020general}, it is easy to observe that several observed variables are deterministic functions of other observed variables. This implies that there exists a realization of observed variables such that the conditional probability of one observed variable given the rest is zero. This is against the positivity assumption.

\textit{Parametrization for an Extended s-Thicket with a Path-Witnessing Subgraph:} In Eq. (7) of \cite{lee2020general} if $v_{\mathcal{P}}$ is an observed variable with only observed parents on a backdoor path $\mathcal{P}$, again,  $v_{\mathcal{P}}$ will be a deterministic function of only observed variables. This again does not satisfy the positivity. In general, such $v_{\mathcal{P}}$ would always exist. 

Please note that the errata for \cite{lee2019general} can potentially fix the issue for s-Thicket, but not for extended s-Thicket or Extended s-Thicket with a Path-Witnessing Subgraph (the last two cases).

\begin{figure}[t]
\centering
    \begin{tikzpicture}[
        roundnode/.style={circle, draw=black!60,, fill=white, thick, inner sep=1pt, minimum size=0.65cm},
        dashednode/.style = {circle, draw=black!60, dashed, fill=white, thick, inner sep=1pt, minimum size=0.65cm},
        ]
            % Nodes
            \node[roundnode]        (X0)        at (1.2, -2)                 {$X_0$};
            \node[roundnode]        (X1)        at (-2.6, -2)                    {$X_1$};
            \node[dashednode]       (U0)        at (2, -.0)              {$U_0$};
            \node[dashednode]       (U1)        at (-2.4, -.0)              {$U_1$};
            \node[roundnode]        (T)         at (-.4, -.0)                   {$T$};
            \node[roundnode]        (Z)         at (-0.7, -2)                {$Z$};
            \node[roundnode]        (Y)         at (3.2, -2)                      {$Y$};
            
            %Edges
            \draw[latex-] (X1) -- (T);
            \draw[latex-] (X1) -- (Z);
            \draw[-latex] (Z) -- (X0);
            \draw[latex-, dashed] (X0) -- (U0);
            \draw[-latex, dashed] (U0) -- (Y);
            \draw[latex-, dashed] (T) -- (U1);
            \draw[-latex, dashed] (U1) -- (X1);
        \end{tikzpicture}
    \caption{DAG $\G$ with $\X = \{X_0, X_1\}$ and $\Y = \{Y\}.$} 
    \label{fig: counter-exmpl}
\end{figure}
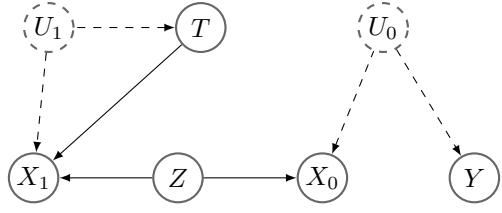 
\subsection{Counterfactual identification}
Here, we refer to the technical report of \cite{correa2021nested} and construct a simple example that demonstrates our main concerns about the proof of the completness part of the c-gID problem.

Recall that a causal effect $P_{\mathbf{T}}(\mathbf{Y}|\mathbf{X})$ can be written as a counterfactual $P(\mathbf{Y}_*, \mathbf{X}_*)$, where $\mathbf{Y}_* \cup \mathbf{X}_* = \{W_{[\mathbf{T}]}|W \in \mathbf{V}(\mathbf{Y}_* \cup \mathbf{X}_*)\}$ and $[\mathbf{T}]$ denotes an intervention under which the counterfactual value is observed. 
Now, consider the graph $\G$ in Fig. \ref{fig: counter-exmpl}.
Suppose that the known distribution is $P(\V)$ and the target conditional causal effect is 
$$
P_{T}(Y|X_0, X_1) = P(\X_*, \Y_*),
$$ 
where $\X_* = \{X_{0[T]}, X_{1[T]}\}$, $\Y_* = \{Y_{[T]}\}$, $\X = \{X_0, X_1\}$ and $\Y = \{Y\}$. 
Note that for both $X_0$ and $X_1$, there exists an active backdoor path to $Y$, thus, we cannot use the second rule of do-calculus to simplify $P_{T}(Y|X_0, X_1)$.
Please note that in this graph, $X_0, X_1, Z, Y$ belong to the same ancestral component (Def. 7 in [2]) induced by $\mathbf{X}_* \cup \mathbf{Y}_*$ given $\mathbf{X}_*$.  
This is because $Z \in An(X_{0[T]})_{\G_{\underline{X_0}}} \cap  An(Z_{[T]})_{\G}\cap An(X_{1[T]})_{\G_{\underline{X_1}}}$ and there is a bidirected arrow between $X_0$ and $Y$. 
This ancestral component contains $Y$ and based on the definition of $\mathbf{D}_*$ (after Eq. (69) in [2]), we have $\mathbf{D}_* = \{X_{0[T]}, X_{1[T]}, Z_{[T]}, Y_{[T]}\}$.
Furthermore, according to Equation (70) in \cite{correa2021nested} is
$$
\rho(\x, \y) := \sum_{\mathbf{d}_*\setminus (\y_*\cup\x_*)}P(\bigwedge_{D_t\in \mathbf{D}_*}D_{\textbf{pa}_d} = d)
$$
 and in our example, it is equivalent to 
 $$
 P_{T}(Y = \y[Y], X_0 = \x[X_0], X_1 = \x[X_1]).
 $$
In part of the proof, they encounter a setting in which $\rho(\x, \y)$ and $\rho(\x)$ are not g-ID and they need to show that 
$$\rho(\y | \x)=\rho(\x, \y)/\rho(\x)$$
is not c-gID.
To do so, they consider two models $\M^{(1)}$ and $\M^{(2)}$ that shows $\rho(\x, \y)$ is not g-ID and transform them into two new models to prove the non-c-gID of $\rho(\y | \x)$.
According to \cite{correa2021nested},  realizations $\x', \y'$ are such that for models $\M^{(1)}$ and $\M^{(2)}$: 
$$
\rho^{(1)}(\y',\! \x')\!\neq\! \rho^{(2)}(\y', \x').
$$
% Since $X_0, X_1$ and $Y$ do not have any children then distributions $P_{T}(Y, X_0)$, $P_{T}(Y, X_1)$ and $P_{T}(X_0, X_1)$ do not depend on the choice of distributions $P(X_1|\Pa{X_1}{\G})$, $P(X_0|\Pa{X_0}{\G})$ and $P(Y|\Pa{Y}{\G})$, respectively. 
Models $\M^{(1)'}$ and $\M^{(2)'}$ obtained from models $\M^{(1)}$ and $\M^{(2)}$ as follows:
\begin{enumerate}
    \item Append an extra bit $U_p$ to the node $U_0$. 
    
    \item $X_p$ and $Y_p$ binary unobserved variables defined for variables $X_0$ and $Y$, respectively.
    \item Rename $X_0$ and $Y$ as $\widetilde{X}_0$ and $\widetilde{Y}$ and make them unobserved, then $X_0$ and $Y$ are defined in models $\M^{(1)'}$ and $\M^{(2)'}$ as $X_0:=\x'[X_0]$ if $X_p=1$ and $\widetilde{X}_0$, otherwise. Similarly, they defined $Y$ using $Y_p$ and $\widetilde{Y}$.
\end{enumerate}
    % \begin{align*}
    %     & X_0 = 
    %     \begin{cases}
    %         & \x'[X_0] \quad \text{if} \; X_p=1,\\
    %         & \widetilde{X}_0 \quad \text{otherwise},
    %     \end{cases}
    %     \quad
    %      Y = 
    %     \begin{cases}
    %         & \y'[Y] \quad \text{if} \; Y_p=1,\\
    %         & \widetilde{Y} \quad \text{otherwise},
    %     \end{cases}
    % \end{align*}
%4. Distributions $P(X_p|U_p)$ and $P(Y_p|U_p)$ are defined such that: $ P(X_p=Y_p)= \gamma/2$ and  $P(X_p\neq Y_p)=(1-\gamma)/2$.
According to the definitions of $\rho(\x,\y)$, $\widetilde{Y}$, and $\widetilde{X}$, and using the law of total probability, we have
\begin{equation*}
    \rho'(\x', \y')\! =\! \sum_{X_p, Y_p, \widetilde{X}, \widetilde{Y}}\rho'(\x', \y', \widetilde{X}_0, \widetilde{Y}| X_p, Y_p)P(X_p, Y_p)
\end{equation*}
and therefore 
\begin{equation}
\begin{split}
    \label{eq:1}
     \rho'(\x', \y') = &P(X_p=0, Y_p=0)\rho(\x', \y') + \\
    & P(X_p=1, Y_p = 0)\rho(\x'[X_1], \y') + \\
    & P(X_p=0, Y_p=1)\rho(\x') + \\
    & P(X_p=1, Y_p=1)\rho(\x'[X_1]).
\end{split}
\end{equation}
Clearly, $\rho(\x'[X_1])\neq1$ otherwise, the positivity assumption does not hold. On the other hand, based on Eq. (78)-(81) \cite{correa2021nested},  $\rho'(\x', \y')$ is computed by 
\begin{equation}\label{eq:2}
\begin{split}
    & P(X_p=0, Y_p=0)\rho(\x', \y') +\\
    & P(X_p=1, Y_p=0)\rho( \y') +\\
    & P(X_p=0, Y_p=1)\rho(\x') +\\
    &P(X_p=1, Y_p=1).
\end{split}
\end{equation}
In general, \eqref{eq:1} and \eqref{eq:2} are not equal unless for example, $\X=\{X\}$ and $\Y=\{Y\}$. 

Moreover the rest of the proof in \cite{correa2021nested}, i.e., Eq. (83)-(92) heavily relies on Eq. \eqref{eq:2}, therefore without corresponding fix the whole proof for the completeness part in c-gID problem falls apart.

\end{document}